%% file: arxiv.tex
\documentclass[11pt]{article} %

\usepackage{times}
\usepackage{url}            %
\usepackage{booktabs} 
\usepackage{wrapfig}%
\usepackage{amsfonts}       %
\usepackage{nicefrac}       %
\usepackage{microtype}      %
\usepackage{subcaption} 
\usepackage{caption}
\usepackage{xspace}
\usepackage{multirow}
\usepackage{amsmath}
\usepackage{algorithm}
\usepackage{algorithmic}
\usepackage{color}
\usepackage{color, colortbl}

\usepackage{thmtools,thm-restate} % if loaded after amsthm, we get a bunch of weird errors...
\usepackage{amsthm}
\usepackage{mkolar_definitions}

\usepackage{enumitem}
\usepackage{comment}
\usepackage{bm}
\usepackage[left=2cm,top=2cm,right=2cm]{geometry}
\usepackage[scaled=.9]{helvet}
\usepackage{url}
\usepackage{authblk}
\usepackage{csquotes}
\usepackage[dvipsnames]{xcolor}
\usepackage[colorlinks=true, linkcolor=blue, citecolor=blue]{hyperref}
\usepackage{natbib}
\bibpunct{(}{)}{;}{a}{,}{,}
\usepackage[para,symbol]{footmisc}

\input{notation.tex}

\definecolor{Gray}{gray}{0.9}

\usepackage[utf8]{inputenc} % allow utf-8 input
\usepackage[T1]{fontenc}    % use 8-bit T1 fonts
\usepackage[colorlinks=true, linkcolor=blue, citecolor=blue]{hyperref}
\usepackage{url}            % simple URL typesetting
\usepackage{booktabs}       % professional-quality tables
\usepackage{amsfonts}       % blackboard math symbols
\usepackage{nicefrac}       % compact symbols for 1/2, etc.
\usepackage{microtype}      % microtypography
\usepackage{xcolor}         % colors
\usepackage{enumitem}
\usepackage{xspace}
\usepackage{makecell}
\usepackage{caption}
\usepackage{eqparbox}
\allowdisplaybreaks

\newcommand{\targ}{\text{target}}
\newcommand{\ie}{\textit{i.e.}}
\newcommand{\eg}{\textit{e.g.}}

\title{Provable Benefits of Representational Transfer in\\ Reinforcement Learning\footnote{alphabetical order.\protect\\}}
\usepackage{textcomp}

\author[1]{Alekh Agarwal\footnote{alekhagarwal@google.com}}
\author[2]{Yuda Song\footnote{yudas@andrew.cmu.edu}}
\author[3]{Wen Sun$^\text{\textsection}$}
\author[3]{Kaiwen Wang\footnote{\{ws455,kw437\}@cornell.edu}}
\author[4]{Mengdi Wang$^\text{\textparagraph}$}
\author[4]{Xuezhou Zhang\footnote{\{mengdiw,xz7392\}@princeton.edu}}
\affil[1]{Google Research}
\affil[2]{Carnegie Mellon University}
\affil[3]{Cornell University}
\affil[4]{Princeton University}
\date{}

\begin{document}

\maketitle

\begin{abstract}%
  We study the problem of representational transfer in RL, where an agent first pretrains in a number of \emph{source tasks} to discover a shared representation, which is subsequently used to learn a good policy in a \emph{target task}. We propose a new notion of task relatedness between source and target tasks, and develop a novel approach for representational transfer under this assumption. Concretely, we show that given a generative access to source tasks, we can discover a representation, using which subsequent linear RL techniques quickly converge to a near-optimal policy in the target task.
  The sample complexity is close to knowing the ground truth features in the target task, and comparable to prior representation learning results in the source tasks. We complement our positive results with lower bounds without generative access, and validate our findings with empirical evaluation on rich observation MDPs that require deep exploration. In our experiments, we observe speed up in learning in the target by pre-training, and also validate the need for generative access in source tasks.
\end{abstract}

\section{Introduction}

Leveraging historical experiences acquired in learning past skills to accelerate the learning of a new skill is a hallmark of intelligent behavior. In this paper, we study this question in the context of reinforcement learning (RL). Specifically, we consider a setting where the learner is exposed to multiple tasks and ask the following question:
\vspace{-0.25cm}
\begin{center}\emph{Can we accelerate RL by sharing representations across multiple related tasks?}\end{center}
\vspace{-0.25cm}
There is rich empirical literature which studies multiple approaches to this question and various paradigms for instantiating it. For instance, in a multi-task learning scenario, the learner has simultaneous access to different tasks and tries to improve the sample complexity by sharing data across them~\citep{caruana1997multitask}. Other works study a transfer learning setting, where the learner has access to multiple source tasks during a \emph{pre-training} phase, followed by a target task~\citep{pan2009survey}. The goal is to learn features and/or a policy which can be quickly adapted to succeed in the target task. More generally, the paradigms of meta-learning~\citep{finn2017model}, lifelong learning~\citep{parisi2019continual} and curriculum learning~\citep{bengio2009curriculum} also consider related questions. \looseness=-1

On the theoretical side, questions of representation learning have received an increased recent emphasis owing to their practical significance, both in supervised learning and RL settings. In RL, a limited form of transfer learning across multiple downstream reward functions is enabled by several recent reward-free representation learning approaches~\citep{jin2020reward, zhang2020task, wang2020reward, du2019provably, misra2020kinematic, Agarwal2020_flambe, modi2021model}. Inspired by recent treatments of representation transfer in supervised~\citep{maurer2016benefit,du2020few} and imitation learning~\citep{arora2020provable}, some works also study more general task collections in bandits~\citep{hu2021near,yang2020impact,yang2022nearly} and RL~\citep{hu2021near,lu2021power}. Almost all these works study settings where the representation is \emph{frozen} after pre-training in the source tasks, and a linear policy or optimal value function approximation is trained in the target task using these learned features. This setting, which we call \emph{representational transfer}, is the main focus of our paper.

A crucial question in formalizing representational transfer settings is the notion of similarity between source and target tasks. Prior works in supervised learning make the stringent assumption that the covariates $x$ follow the same underlying distribution in all the tasks, and only the conditional $P(y|x)$ can vary across tasks~\citep{du2020few}.
This assumption does not nicely generalize to RL settings, where state distributions are typically policy dependent, and prior extensions to RL~\citep{lu2021power,cheng2022provable} resulted in strong assumptions during the learning setup.

\noindent With this context, we summarize our \textbf{main contributions} below.%\noindent\textbf{Our contributions.}
\begin{itemize}[leftmargin=*,noitemsep,nosep]%[leftmargin=*,noitemsep,nosep]
\item \textbf{Task relatedness:} We propose a new \emph{state-dependent} linear span assumption of task relatedness, which generalizes all prior settings for representational transfer in RL, \eg, \citet{cheng2022provable}. We give examples captured by this assumption, and it generalizes all prior settings for representational transfer in RL. We do not make any linearity assumptions on our feature maps.
\item \textbf{Generative access to source tasks:} We propose \ouralg and prove that it successfully pre-trains a representation for downstream \emph{online learning in any target task} satisfying the linear span assumption. Our algorithm uses a novel cross-sampling procedure made possible by generative access in the source tasks. Our key result is that the regret in the target task, up to a task-relatedness constant, matches that of learning in a linear MDP equipped with \emph{known} ground truth features, which is the strongest possible benchmark to compete with. A key technical contribution is a novel analysis of \LSVIUCB{} that attains regret under an average-case misspecified linear MDP.
\item \textbf{Lower bound without generative access:} We further show a counter-example where representational transfer fails without generative access under our assumptions.
As a partial remedy, we posit that every \emph{observed state} is reachable in each source task, and show a modification of \ouralg is still sufficient for transfer learning with only online access. While strong, this observational reachability assumption still generalizes prior results in transfer RL, \eg, \citet{cheng2022provable}.  \looseness=-1
\item \textbf{Empirical validation:} We empirically validate \ouralg on a challenging benchmark \citep{misra2020kinematic}, and show that \ouralg saves an order of magnitude of target samples compared to training from scratch using the SOTA Block MDP algorithm BRIEE.
\end{itemize}

\noindent
%Along with our novel task relatedness  introduced new setting (i.e., the new task relatedness assumption) and a set of new results for provably efficient representation transfer,
Our intermediate results may also be of independent interest: (1) to pre-train a representation, we developed an oracle-efficient reward-free exploration algorithm for low-rank MDPs, (2) to transfer the pre-trained representation to the target task, we develop a new analysis for linear MDP under an average case model misspecification extending prior work which relies on a much stronger $\ell_\infty$ style model misspecification \citep{jin2020provably}.

\subsection{Related Work}\label{sec:related_work}
% \vspace{-5pt}
% \begin{table*}[th!]
% \centering
% \begin{tabular}{|p{2.5cm}|p{4cm}|p{4cm}|p{4cm}|}
% \hline
% Assumptions & \cite{lu2021power} & \cite{cheng2022provable} & Ours\\\hhline{|=|=|=|=|}
% Task-Relatedness & Source Task Coverage: $\overline{W_h^\star} = [\omega^\star_{1;h},...,\omega^\star_{K-1;h}]$ satisfies $\sigma^2_d(\overline{W_h^\star})\geq \Omega(K/d)$.
% & Global Linear Span:
% $\mu^\star_{K;h}(s') = \sum_{k=1}^{K-1} \alpha_{k;h}\mu^\star_{k;h}(s')$.
% & Point-wise Linear Span:
% $\mu^\star_{K;h}(s') = \sum_{k=1}^{K-1} \alpha_{k;h}(s')\mu^\star_{k;h}(s')$. \\\hline
% Data Distribution & Generative Access to an ideal distribution $q$ in source AND target task.
% & Observation Reachability: $\exists \pi$, s.t. $P^\pi_{k;h}(s)\geq \psi$.
% & Feature Reachability: $\exists \pi$, s.t. $\lambda_{\min}(\Sigma^\pi_{k;h})\geq \psi$.
% \newline
% OR
% \newline
% Observation Reachability: $\exists \pi$, s.t. $P^\pi_{k;h}(s)\geq \psi$.\\\hline
% Function Class
% & For any $\phi,\phi'\in\Phi$, if $\text{Tr}[D_q(\phi,\phi')]\leq \epsilon$, then $\|P_\phi^\perp \phi'(x)\|\leq \epsilon,\forall \|x\|\leq 1$.
% &
% Average TV Error upper-bounds Point-wise TV Error
% $\|P^1_h(\cdot|s,a)-P^2_h(\cdot|s,a)\|_{TV}\leq C_R \EE_{(s_h,a_h)\sim \Ucal(\Scal,\Acal)}\|P^1_h(\cdot|s,a)-P^2_h(\cdot|s,a)\|_{TV}$ for all $(s_h,a_h)\in \Scal\times \Acal$.
% &
% None. \\\hline
% Transferability & Generative to Generative
% & Online to Online
% & Generative to Online\newline
% OR
% \newline
% Online to Online\\\hline
\begin{table}[]
    \centering
    \begin{tabular}{>{\centering}m{2.25cm}|>{\centering}m{3.75cm}|>{\centering}m{2cm}|>{\centering}m{3.5cm}|>{\centering\arraybackslash}m{2.5cm}}
        \hline ~& Task-relatedness & Access type (source/target) & Reachability & Generalization (assumption)  \\ \hline
         \citet{lu2021power}& Full rank LSVI weight matrix from source & Gen/Gen & Distribution $q$ (\textbf{given}) covers observations & $\text{err}(s,a) \leq C \EE_q[\text{err}] \forall s,a$\\ \hline
         \citet{cheng2022provable}& $P^\star_\targ(s'|s,a) = \sum_{i=1}^{K} \alpha_i P^\star_i(s'|s,a)$ & On/On & Observational reachability & $\text{err}(s,a) \leq C \EE_{\text{Unif}}[\text{err}] \forall s,a$\\ \hline
         Theorem~\ref{thm:generative-end-to-end} (this paper) & $P^\star_\targ(s'|s,a) = \sum_{i=1}^{K} \textcolor{red}{\alpha_i(s')} P^\star_i(s'|s,a)$ & Gen/On & Feature reachability ($\phi^\star$) & None\\ \hline
         Theorem~\ref{thm:online} (this paper) & $P^\star_\targ(s'|s,a) = \sum_{i=1}^{K} \textcolor{red}{\alpha_i(s')} P^\star_i(s'|s,a)$ & On/On & Observational reachability & None\\ \hline
         Theorem~\ref{thm:imp} (lower bound) & $P^\star_\targ(s'|s,a) = \sum_{i=1}^{K} \textcolor{red}{\alpha_i(s')} P^\star_i(s'|s,a)$ & On/On & Feature reachability ($\phi^\star$) & None\\ \hline
    \end{tabular}
    \caption{Assumptions for representational transfer in low-rank MDPs. ``Gen'' and ``On'' refer to generative or online access to source and target tasks.
    Feature reachability means that each source task a policy with a full rank covariance under the features $\phi^\star$ (\pref{ass:reachability}). Observational reachability requires each high-dimensional raw observation to be reachability with some lower bounded probability (\pref{ass:reachability_raw}).
    The last row is a lower bound which precludes learning under the assumptions of \pref{thm:generative-end-to-end} without generative access in the source tasks. }
    \label{tab:comparisons-to-transfer-work}
\end{table}

\textbf{Transfer Learning in Low-rank MDPs.}
The closest work to ours is \citet{cheng2022provable}, which also performs reward-free exploration in the source tasks for representation learning, and use the learned representation in the target task to perform online learning.
\citet{cheng2022provable} proposed a linear span assumption with globally fixed coefficients, which is generalized by our \emph{state-dependent} linear span assumption.
However, despite the more stringent relatedness condition, their work still makes stronger assumptions to enable transfer.
First, their Assumption 5.3 sidesteps the need to handle generalization by assuming point-wise error is bounded by average-case error (this allows them to directly use the result from \citet{jin2020provably}), whereas our analysis only relies on standard in-distribution generalization and indeed one of our key technical contributions is showing that \LSVIUCB{} succeeds even with average-case misspecification. Our \pref{thm:online} generalizes the result of \citet{cheng2022provable}.
Second, their Assumption 5.1 assumes reachability in the high-dimensional observation space, whereas we show that with our novel cross-sampling, it is possible to require the much more realistic spectral coverage in the ground truth (unknown) feature space (\pref{ass:reachability}).

Another work on transfer learning in low-rank MDPs is \citet{lu2021power}, which also makes much stronger assumptions than our work. First, they require generative access in \emph{both} source and target tasks. Second, they require the covariance of \emph{any pair} of features (in feature class $\Phi$) to be full rank, while we only require this reachability condition for the true feature $\phi^\star$. Third, they assume a given distribution $q$ on which the learned representation can extrapolate, whereas we explicitly construct such a data distribution using the novel cross-sampling procedure.
%we show that data collected from our cross-sampling procedure suffices for transfer learning.
In sum, compared to the prior works \citet{cheng2022provable,lu2021power}, we leverage a novel cross-sampling procedure to enable transfer under significantly weaker assumptions. Furthermore, we prove a lower bound showing that generative access is necessary unless stronger assumptions, \eg, those in \citet{cheng2022provable}, are made. We summarize the comparision to these prior works in \pref{tab:comparisons-to-transfer-work}.

\textbf{Transfer Learning in Bandit and small-size MDPs.}
\cite{lazaric2013sequential} study spectral techniques for online sequential transfer learning in multi-arm bandits.
\cite{brunskill2014pac} study transfer in semi-MDPs by learning options. \cite{lecarpentier2021lipschitz} consider lifelong learning in Lipschitz MDP. All these works consider tabular models while we focus on large-state MDPs.

\textbf{Multi-task learning. }
While the multi-task setting also deals with multiple tasks, it is different from the transfer learning setting in its objective.
The goal of multi-task learning is to perform well over all tasks (typically the average performance of tasks), while transfer learning cares exclusively about performance in the target task. Thus, the results from multi-task learning are not directly comparable to the transfer learning results that we focus on in this paper. We survey some multi-task literature for completeness.
For multi-task learning in low-rank MDPs, \citet{huang2023provably} only assumed $\phi^\star$ to be shared ($\mu^\star$ can be arbitrarily different between tasks), and showed that the sample complexity of a multi-task variant of BiLin-UCB \citep{du2021bilinear} does not scale as $K|\Phi|$ but only as $|\Phi|$. However, like BiLin-UCB, the algorithm is not computationally efficient. Several recent works study multi-task linear bandits with linear representations ($\phi(s) = A\,s$ with unknown $A$) \citep{hu2021near,yang2020impact,yang2022nearly}. The techniques developed in these works crucially rely on the linear structure and can not be applied to nonlinear function classes.

For a discussion of the empirical transfer literature, as well as more detailed comparisons to related works, please see \pref{sec:compare}.

\section{Preliminaries}\label{sec:prelim}
In this paper, we study transfer learning in finite-horizon, episodic Markov Decision Processes (MDPs), $\Mcal=\langle H, \Scal, \Acal, \{P_h^\star\}_{0:H-1},\{r_h\}_{0:H-1}, d_0 \rangle $,
specified by the episode length $H$, state space $\Scal$, discrete action space $\Acal$ of size $A$,
\emph{unknown} transition dynamics $P_h^\star:\Scal \times \Acal\to \Delta(\Scal)$, \emph{known} reward functions $r_h:\Scal \times \Acal \to [0,1]$, and a \emph{known} initial distribution $d_0\in \Delta(\Scal)$.
We now define the value, $Q$ functions and visitation distribution, where we make the dependence on the transition dynamics $P^\star=\{P_h^\star\}_{0:H-1}$ and reward functions $r = \{r_h\}_{0:H-1}$ explicit.
For any Markov policy $\pi: \Scal\to\Delta(\Acal)$, let $\Eb[\pi,P^\star]{\cdot}$ denote the expectation under the trajectory distribution of
executing $\pi$ in an MDP with transitions $P^\star$, \ie, start at an initial state $s_0\sim d_0$, then for all $h\in[0:H-1]$\footnote{For $1\leq a\leq b$, we denote $[a:b]=a,a+1,\dots,b-1,b$, and $[b]=[1:b]$.}, $a_h\sim\pi_h(s_h),s_{h+1}\sim P_h^\star(s_h,a_h)$.
If $P^\star$ is clear from context, we use $\Eb[\pi]{\cdot}$ instead.
The value function is the expected reward-to-go of $\pi$ starting at state $s$ in step $h$, \ie,$V^{\pi}_{P^\star,r;h}(s) = \Eb[\pi,P^\star]{ \sum_{\tau=h}^{H-1} r_\tau(s_\tau,a_\tau) \mid s_h = s }$.
The $Q$ function is $Q^{\pi}_{P^\star,r;h}(s,a) := r_h(s,a) + \EE_{s'\sim P_h^\star(\cdot | s,a)} V^{\pi}_{P^\star,r;h+1}(s')$.
We denote the expected total reward of a policy $\pi$ as $V^{\pi}_{P,r} := \EE_{s_0\sim d_0} V^{\pi}_{P,r;0}(s_0)$.
We define the state-action occupancy distribution $d^{\pi}_{P^\star;h}(s,a)$ as the probability of $\pi$ visiting $(s,a)$ at time step $h$.

The \emph{transfer learning} problem consists of two phases:
(1) the \emph{pre-training phase} where the agent interacts with $K$ source tasks with transition dynamics $\braces*{P_k^\star}_{k\in[K]}$, and
(2) the \emph{deployment phase} where the agent is deployed into the target task with transition dynamics $P^\star_\targ$ and can no longer access the source tasks.
% We posit that the transition dynamics of the target task is related to the transition dynamics of the source tasks, which we will soon make precise.
The performance of a transfer learning algorithm is measured by (1) the sample complexity in the source tasks during pre-training, and (2) the regret in the target task during deployment, which is defined as
$
    \op{Regret}(T)=\sum_{t=1}^T V_{\text{target}}^\star - V^{\pi^t}_{P^\star_\targ, r},
$
where $V_{\text{target}}^\star$ is the optimal value that can be obtained in the target task, and $\pi^t$ is the policy played in the $t$-th episode of deployment. For notation, we let $d^\pi_{k;h}$ be short-hand for $d^\pi_{P^\star_k;h}$.

We begin formalizing our problem with the low-rank MDP model.
\begin{definition}[Low-rank MDP~\citep{jiang2017contextual,Agarwal2020_flambe}]
A transition model $P_h^{\star}:\Scal\times \Acal \to \Delta(\Scal)$ is low rank with rank $d\in \NN$
if there exist two unknown embedding functions $\phi_h^{\star}: \Scal\times\Acal \mapsto \RR^{d}$,
$\mu_h^{\star}:\Scal \mapsto \RR^{d}$ such that
$\forall s,s'\in \Scal, a\in \Acal: P_h^{\star}(s'\mid s,a)=\mu_h^{\star}(s')^{\top}{\phi_h^{\star}(s,a)}$,
where $\|\phi_h^\star(s,a)\|_2\leq 1$ for all $(s,a)$ and for any function $g:\Scal \to [0,1]$, $\nm{\int g(s)\diff\mu_h^{\star}(s)}_2\leq \sqrt{d}$.
An MDP is a low rank MDP if $P_h^{\star}$ admits such a low rank decomposition for all $h\in[0:H-1]$.
\end{definition}
Low-rank MDPs capture the latent variable model \citep{Agarwal2020_flambe} where $\phi^\star(s,a)$ is a distribution over a discrete latent state space $\mathcal{Z}$,
and the block-MDP model \citep{du2019provably} where $\phi^\star(s,a)$ is a one-hot encoding vector.
Note that $\Phi$ can be a non-linear, flexible function class, so the low-rank framework generalizes prior works with linear representations \citep{hu2021near,yang2020impact,yang2022nearly}.
Next, we define what it means for a policy to be exploratory in a low-rank MDP.
\begin{definition}[Feature coverage]\label{def:exploratory-policies}
For $\alpha\in\RR_+$, a policy $\pi$ is $\alpha$-\emph{exploratory} in an MDP with transition dynamics $P^\star$ if for all $h\in[0:H-1]$, we have $\lambda_{\min}\prns*{\Eb[\pi,P^\star]{ \phi_h^\star(s_h,a_h)\phi_h^{\star\top}(s_h,a_h)^\top }}\geq\alpha$.
\end{definition}
An exploratory $\pi$ intuitively ensures that the whole $\RR^d$ feature space is well-explored in a spectral sense.
Note this generalizes the notion of ``Policy Cover'' in Block MDPs from \citet{misra2019kinematic}.

We now make two mild structural assumptions on the tasks to enable representational transfer.
\begin{assum}[Common $\phi^\star$]\label{ass:shared_phi}
All tasks are low-rank MDPs with a shared representation $\phi_h^\star(s,a)$.
\end{assum}
\begin{assum}[Point-wise Linear span]\label{ass:span}
For any $h\in[0:H-1]$ and $s'\in\Scal_{h+1}$, there is a vector $\alpha_h(s') \in \RR^{K}$
such that $\mu^\star_{\targ;h}(s') = \sum_{k=1}^{K} \alpha_{k;h}(s')\mu^\star_{k;h}(s')$.
\label{ass:linear-span}
\end{assum}
The motivation for \pref{ass:span} is: if $s'$ is reachable from an $(s,a)$ pair in the target task,
then it must be reachable from the same $(s,a)$ pair in at least one of the source tasks.
Intuitively, this is necessary for transfer learning to succeed, as $s'$ could be a high rewarding state in the target.
Based on \pref{ass:span}, we define,
$\alpha_{\max}=\max_{h;k,s'\in\Scal} |\alpha_{k;h}(s')|$ and
$\bar\alpha=\max_h \sum_{k=1}^{K} \max_{s'\in\Scal} |\alpha_{k;h}(s')|$.
Note that $\alpha_{\max}\leq\bar\alpha$, which we assume is bounded.
We conclude the section with a couple of examples where these assumptions are satisfied.

\begin{example}[Mixture of source tasks]\label{ex:mixture}
The mixture model posits that the target task's transition dynamics is a mixture of the source tasks, \ie,  $P^\star_\targ(s'|s,a) = \sum_{k=1}^{K} \alpha_k P^\star_k(s'|s,a)$.
This maps to \pref{ass:span} with $\alpha_{k;h}(s') = p_k$ where $\{p_k\}_{k\in[K]}$ is a probability distribution, so $\bar\alpha = 1$.
These mixture models have been studied in the context of known source models \citep{modi2020sample,ayoub2020model},
and, corresponding to our \pref{ass:shared_phi}, unknown low-rank source models with the same $\phi^\star$ \citep{cheng2022provable}.
Our linear span \pref{ass:span} strictly generalizes the mixture model by allowing linear span coefficients to flexibly depend on $s'$,
which is more realistic in practice.
\end{example}
In \pref{ex:mixture}, $\bar\alpha$ (and hence $\alpha_{\max}$) was nicely bounded by $1$.
However, if the target task largely focuses on observations quite rare under the source tasks,
then $\bar\alpha$ can grow large. \looseness=-1
% We remark that \citet{cheng2022provable} used strictly stronger assumptions: they posited \pref{ass:shared_phi} and their
% Assumption 5.4 requires the linear span's coefficients to be globally fixed, while our

\begin{example}[Block MDPs with shared latent dynamics]
Here, each MDP $P_k^\star$ is a Block MDP \citep{du2019provably} with a shared latent space $\Zcal$ and a shared decoder $\psi^\star~:~\Scal\to\Zcal$. In a block MDP, given state action pair $(s,a)$, the decoder $\psi^\star$ maps $s$ to a latent state $z$, the next latent state is sampled from the latent transition $z'\sim P(\cdot | z, a)$, and the next state is generated from an emission distribution $s'\sim o(\cdot | z')$. Recall that $o(s'|z') > 0$ at only one $z'\in\Zcal$ for any $s'\in\Scal$ for a block MDP. We assume that the latent transition model $P(z'|z,a)$ is shared across all the tasks, but the emission process differs across the MDPs.
For instance, in a typical navigation example used to motivate Block MDPs, the latent dynamics might correspond to navigating in a shared 2-D map, while emission distributions capture different wall colors or lighting conditions across multiple rooms.  %of rooms with different wall colors or lighting conditions, leading to different observations from the same latent state.
Then \pref{def:exploratory-policies} posits that the agent can visit the entire 2-D map, while \pref{ass:linear-span} requires that the color/lighting conditions of the target task resemble that of at least one source task. The coefficients $\alpha$ for any $s'$ are non-zero on the source tasks which can generate that observation. %Note that since the $\alpha$ coefficients in Assumption~\ref{ass:linear-span} can depend on the state $s'$, we can even handle scenarios where different parts of the target task have a light/coloring pattern resembling different source tasks.
\label{ex:block-mdp}
\end{example}

\section{Representational Transfer with Generative Access in Source Tasks}
In this section, we study transfer learning assuming \emph{generative access} to the source tasks.
\begin{assum}[Generative access in the source tasks]\label{ass:generative_model}
For any $k\in[K],h\in[0:H-1]$ and $s,a\in\Scal\times\Acal$,
we can query independent samples from $P^\star_{k;h}(s,a)$.
\end{assum}
The generative model access is not unrealistic, especially in applications where a
high-quality simulation environment is available.
Perhaps surprisingly, we will also show (in \pref{sec:online}) that generative access in source tasks
is \emph{necessary} assuming only feature coverage, as in ~\pref{def:exploratory-policies}.

\subsection{The Algorithm}

\begin{algorithm}[t]
	\caption{Exploratory Policy Search (\expAlg)}
	\label{alg:reward_free}
	\begin{algorithmic}[1]
	\STATE \textbf{Input: } MDP $\Mcal$ with online access, num. \LSVIUCB{} episodes $N_\LSVIUCB$, num. model-learning episodes $N_{\textsc{RewardFree}}$, failure probability $\delta$.
	\STATE Learn model $\{\wh P_h = (\wh\phi_h, \wh \mu_h) \}_{h=0}^{H-1}$ by running \textsc{RewardFree Rep-UCB} (\pref{alg:reward_free_repucb}) in $\Mcal$ for $N_{\textsc{RewardFree}}$ episodes.
	\STATE Set $\beta = dH\sqrt{\log(dHN_{\LSVIUCB}/\delta)}$.
	\STATE Return $\rho = \LSVIUCB\prns{ \{\wh\phi_h\}_{h=0}^{H-1}, r=0, N_{\LSVIUCB}, \beta, \textsc{UniformActions}=\textsc{True} }$ by simulating in the learned model $\wh P$ (\pref{alg:LSVI}). Note this step requires no samples from $P^\star$.
	\end{algorithmic}
\end{algorithm}

\begin{algorithm}[!t]
	\caption{Transfer learning with generative access (\ouralg)}\label{alg:main}
	\begin{flushleft}
		\textsc{\textbf{Pre-Training Phase}}\\
% 		\textbf{Input:} num. \LSVIUCB{} episodes $N_\LSVIUCB$, num. model-learning episodes $N_{\textsc{RewardFree}}$, size of cross-sampled datasets $n$, failure probability $\delta$.
		\textbf{Input:} exploratory policies in the source tasks $\{\pi_k\}_{1:K}$, function classes $\Phi,\braces{\Upsilon_k}_{1:K}$, size of cross-sampled datasets $n$, failure probability $\delta$.
	\end{flushleft}
	\begin{algorithmic}[1]
% 		\FOR{source task $k=1,...,K-1$}
% 		    \STATE Find exploratory policy in $P_k^\star$: $\pi_k=$\textsc{RewardFree}$(P_k^\star, N_\LSVIUCB, N_{\textsc{RewardFree}}, \delta)$.
% 		\ENDFOR
		\FOR[\textcolor{blue}{cross sampling procedure}]{\emph{task pairs} $i,j$, \ie, for all $i,j\in[K]$}
		    \STATE  For each $h\in[H-1]$, sample dataset $\Dcal_{ij;h}$ containing $n$ \textit{i.i.d.} $(s,a,s')$ tuples sampled as:
            \begin{align}
                (\tilde s,\tilde a)\sim d^{\pi_i}_{i;h-1}, s\sim P^\star_{j;h-1}(\cdot|\tilde s,\tilde a), a \sim\op{unif}(\Acal), s'\sim P^\star_{i;h}(\cdot|s, a).\label{eq:data_sample}
            \end{align}
            For $h=0$, use $s\sim P^\star_{j;h-1}(\cdot|\tilde s,\tilde a), a \sim\op{unif}(\Acal), s'\sim P^\star_{i;h}(\cdot|s, a)$.
		\ENDFOR
		\STATE For each $h\in[0:H-1]$, learn features with MLE, \ie, ``Multi-task \textsc{RepLearn}'',
            \begin{align}
                \wh\phi_h,\wh\mu_{1:K} = \argmax_{\phi\in\Phi,\mu_k\in\Upsilon_k}
                \sum_{i,j\in[K]}\mathbb{E}_{\Dcal_{ij;h}}\left[\log \phi(s,a)^\top\mu_k(s')\right]. \label{eq:multi-task-replearn}
            \end{align}
	\end{algorithmic}
	\vspace{-0.5cm}
	\begin{flushleft}
		\textsc{\textbf{Deployment Phase}}\\
        \textbf{Additional Input:} number of deployment episodes $T$.
	\end{flushleft}
	\begin{algorithmic}[1]
	\STATE Set $\beta = H\sqrt{d} + \bar\alpha dH\sqrt{\log(dHT/\delta)}$.
	\STATE Run $\LSVIUCB\prns{\{ \wh\phi_h\}_{h=0}^{H-1}, r=r_\targ, T, \beta}$ in the target task $P_\targ^\star$ (\pref{alg:LSVI}).
	\end{algorithmic}
\end{algorithm}

We first describe the helper algorithm \textbf{E}xploratory \textbf{P}olicy \textbf{S}earch (\expAlg) (\pref{alg:reward_free}) to discover exploratory policies in low-rank MDPs. \expAlg has two steps. First, it runs a reward-free variant of \textsc{Rep-UCB}\citep{uehara2021representation} in each source task $k$, to learn a linear MDP which approximates the true low-rank MDP $P_k^\star$. Then, an exploratory policy is learned via reward-free exploration in the learned linear MDP (e.g., using \LSVIUCB{} with zero reward), which involves no further interactions with the true environment. Intuitively, the policy $\pi_k$ is trained to fit \pref{def:exploratory-policies} in the source task $k$.

We now present our main algorithm \ouralg (\pref{alg:main}), which takes as input exploratory policies in each source task that can be obtained from \expAlg.
During the pre-training phase, \ouralg collects a dataset via a novel \emph{cross-sampling} procedure across all pairs of source tasks.
Note this step is only possible due to generative access in the source tasks.
Specifically, given exploratory policies $\pi_i,\pi_j$ in source tasks $i,j$,
for any $h\in[H-1]$ we first sample from the visitation distribution of $\pi_i$ in task $i$, \ie, $s_{h-1}, a_{h-1} \sim d^{\pi_i}_{i; h-1}$.
Then, in the simulator of task $j$, we \emph{reset} to $(s_{h-1},a_{h-1})$ and
perform a transition step to $s_h$, \ie, $s_h\sim P_{j;h-1}(s_{h-1},a_{h-1})$.
Next, we uniformly sample an action $a_h$, reset the simulator of task $k$ to state $s_h,a_h$,
and transition to $s_{h+1}$, \ie, $s_{h+1}\sim P^\star_k(s_h,a_h)$.
We then perform Maximum Likelihood Estimation (MLE) representation learning in Eq.~\pref{eq:multi-task-replearn} using the union of the cross-sampled datasets across all pairs of source tasks.
% Under realizability of the function class (\pref{ass:fcn_class}),
% the MLE representation $\wh\phi_h$ can predict the transition probability well across all source tasks, via a generalization argument from \citet{Agarwal2020_flambe}.
In sum, \ouralg learns a \emph{single} representation $\wh\phi$ in the pre-training phase using MLE on the cross-sampled datasets from exploratory policies across tasks.
In the deployment phase, \ouralg runs optimistic least squares value iteration (\LSVIUCB{}) in the target task with the learned representation.
First proposed by \citet{jin2020provably}, \LSVIUCB{} is displayed in \pref{alg:LSVI}, which at a high level is as follows.
Given any dataset, $\{s,a,r,s'\}$  feature $\phi$, and reward $r$, LSVI learns a $Q$ function backward,
\ie, at step $h$ via $\hat w_h = \arg\min_{w} \sum_{s,a,s'} (w^\top \phi(s,a) - \hat V_{h+1}(s'))^2+\lambda\|w\|^2$
and sets $\hat V_h(s) = \max_a ( r(s,a) + \hat w_h^\top \phi(s,a)),\forall s$.
UCB, short for Upper Confidence Bound, refers to an exploration bonus added to basic LSVI.

\subsection{Main Result}
In this section, we prove our main transfer learning result, which shows that \ouralg achieves near optimal regret in the target task with nice pre-training sample complexity in the source tasks.
Our main result requires two assumptions. 
First, we need to ensure that \expAlg can successfully discover an exploratory policy in the source tasks, \ie, there should exist a policy that non-trivially reaches the whole $\RR^d$ in the feature space. Without exploratory policies in the source tasks, it may be possible that the optimal target policy visits subspaces unexplorable in any source task, in which case, pre-training will not have any benefits.
% For \expAlg to successfully discover an exploratory policy in the source tasks, we need to assume that no subspace in $\RR^d$ is unreachable in the source tasks, as otherwise exploratory policies do not even exist. \wen{this isn't a convincing argument -- you can do reward free without reachability -- see the comment beow below well}
\begin{assum}[Reachability in source tasks]\label{ass:reachability}
There exists a $\psi\in\RR_+$ such that for all $k\in[K],h\in[0:H-1]$, there exists a policy $\pi$ such that
$\lambda_{\min}\prns{\EE_{\pi,P_k^\star}[\phi^\star_h(s_h,a_h)\phi^\star_h(s_h,a_h)^\top]}\geq\psi$.
\end{assum}
% ensures that there is no redundant dimensions in the ground-truth representation, and our low-rank version 
Note that this low-rank reachability assumption generalizes the reachability assumption in latent variable and block MDPs, e.g. \citep{modi2021model,misra2020kinematic}.

Second, For the MLE in Eq.~\pref{eq:multi-task-replearn} to succeed, we need to assume the standard realizability assumption, which is made in almost all prior works in low-rank MDPs.
\begin{assum}
[Realizability]\label{ass:fcn_class}
For any source task $k\in[K]$ and any $h\in[H]$, $\phi^\star_h\in\Phi$ and $\mu^\star_{k;h}\in \Upsilon_{k}$.
For normalization, we assume that for all $\phi\in\Phi,\mu\in\Upsilon_{k},g:\Scal \to [0,1]$,
we have $\|\phi(s,a)\|_2\leq 1$ and $\nm{\int g(s)\diff\mu_h^{\star}(s)}_2\leq \sqrt{d}$.
\end{assum}
This leads to our main theorem.
\begin{restatable}[Regret under generative source access]{theorem}{generativeCor}\label{thm:generative-end-to-end}
Suppose Assumptions~\ref{ass:shared_phi},\ref{ass:span},\ref{ass:generative_model},\ref{ass:reachability},\ref{ass:fcn_class} hold, and fix any $\delta\in(0,1)$.
Then, running \ouralg with policies from \expAlg (parameters set as in \pref{lem:coverage}) has regret in the target task of
$\wt\Ocal\prns{\bar\alpha H^2d^{1.5}\sqrt{T\log(1/\delta)}}$,
with at most \\
$\wt\Ocal\prns{ A^4\alpha_{\max}^3d^5 H^7 K^2 T \psi^{-2} \prns{ \log(|\Phi|/\delta)+K\log|\Upsilon| } }$ generative accesses per source task.
% $\wt\Ocal\prns{ A^9d^{16}H^{22}\psi^{-4} \log\prns{|\Phi||\Upsilon|/\delta} + A^4\alpha_{\max}^3d^5 H^7 K^2 T \psi^{-2} \prns{ \log(|\Phi|/\delta)+K\log|\Upsilon| } }$ transitions per source task.
\end{restatable}
\noindent Remarkably, \pref{thm:main} shows that with the pre-trained features, we achieve the same regret bound on the target task to the setting of linear MDP with known $\phi^\star$ \citep{jin2020provably}, up to the additional $\bar\alpha$ factor that depends on the linear span coefficients and captures the intrinsic hardness of transfer learning.
For special cases such as convex combination, i.e., $\alpha$ is state-independent and $\alpha_h \in \Delta(K)$, then $\bar\alpha = 1$.
In the worst-case, some dependence on the scale of $\alpha$ seems unavoidable as we can have a state $s'$ such that $\mu_{\targ}(s') = 1$ and $\mu_k(s') \ll 1$ with $\alpha_k(s') \gg 1$.
This corresponds to a rarely observed state for the source task encountered often in the target, and our estimates of transitions involving this state can be highly unreliable if it is not seen in any other source,
roughly scaling the error between target and source tasks as $|\alpha_k(s')|$.
Obtaining formal lower bounds that capture a matching dependence on structural properties of $\alpha$ is an interesting question for future research.

\subsection{Proof Sketch}
The proof can be broken down into three parts.
First, under reachability, we show in \pref{lem:coverage} that \expAlg can indeed identify an exploratory policy.
Second, we show in \pref{lem:mle_target} that our novel cross-sampling procedure with MLE can learn a representation that linearly approximates $P^\star_\targ$, in an average-case sense.
Third, we prove that even under average-case misspecification, LSVI-UCB succeeds with low regret.
We start by showing that \expAlg can identify an exploratory policy.
\begin{restatable}[Source task exploration]{lemma}{rewardFreeExploreCoverage}\label{lem:coverage}
Suppose Assumptions~\ref{ass:reachability},\ref{ass:fcn_class} hold.
Then, for any $\delta\in(0,1)$, w.p. $1-\delta$, running \expAlg in any source task with $N_{\LSVIUCB} = \wt\Theta\prns{A^3d^6H^8\psi^{-2}}$ and $N_{\textsc{RewardFree}} = \wt\Ocal\prns{A^3d^4H^6\log\prns{\tfrac{|\Phi||\Upsilon|}{\delta}} N_{\LSVIUCB}^{2}}$ returns a $\lambda_{\min}$-exploratory policy where
$
\lambda_{\min} = \wt\Omega\prns{A^{-3}d^{-5}H^{-7}\psi^2}\,.
$
The sample complexity in the source task is $N_{\textsc{RewardFree}}$ episodes.
\end{restatable}
\noindent To the best of our knowledge, \pref{lem:coverage} is the first result that finds an exploratory policy in low-rank MDPs, and might be of independent interest.
\citet{wagenmaker2022reward} recently obtained a related guarantee in the linear MDP setting with known features $\phi^\star$.
\citet[REFUEL]{cheng2022provable} is also a reward-free modification of Rep-UCB, but the algorithm proceeds jointly over all tasks while we run \textsc{RewardFree Rep-UCB} in each task independently.
We note that REFUEL involves optimizing the Pseudo-Cumulative Value Function (PCV), which may be computationally hard in the planning step.
Our \textsc{RewardFree Rep-UCB}'s planning step is the same as Rep-UCB (i.e., planning in a known linear MDP model), and is computationally efficient.
We also remark that this step of identifying exploratory policies is modular and one could also directly use the reward-free algorithm FLAMBE \citep{Agarwal2020_flambe}, despite having a worse sample complexity in source.

We now analyze our novel cross-sampling procedure using the MLE generalization analysis of \citet{Agarwal2020_flambe}. 
For any $h$, let $\Dcal_{k;h}$ be a dataset of size $N$ sampled from some distribution $\nu_{k;h}$.
Under realizability, running multi-task MLE in \pref{eq:multi-task-replearn} with these datasets satisfies w.p. at least $1-\delta$,
\begin{equation}\label{eq:mle}
    \sum_{i,j\in[K]}\EE_{\nu_{ij;h}}\nm{\wh\phi_h(s,a)^\top \wh\mu_{k;h}(\cdot) - \phi_h^\star(s,a)^\top\mu_{k;h}^\star(\cdot)}_{TV}^2 \leq \zeta_N\coloneqq \Ocal\prns{\nicefrac{ \prns{\log\left(\nicefrac{|\Phi|}{\delta}\right)+K\log(|\Upsilon|)}}{N}}.
\end{equation}
where $\|\cdot\|_{TV}$ denotes the total variation (TV) norm, and $\nu_{ij;h}$ is the distribution from which we sampled $\Dcal_{ij;h}$. That is, $s,a\sim\nu_{ij;h}$ is equivalent to $(\wt s,\wt a)\sim d^{\pi_i}_{i;h-1}, s\sim P_{j;h-1}^\star(\wt s,\wt a), a\sim \op{unif}(\Acal)$.
Then, by applying the one-step back lemma (due to the low-rank structure of the target), followed by the linear span assumption, we can prove the following lemma. 
\begin{restatable}[Target model error]{lemma}{mleTargetGenerative}\label{lem:mle_target}
Suppose \pref{ass:linear-span} holds and $\pi_k$ is $\lambda_{\min}$-exploratory for each source task $k$.
For any $\delta\in(0,1)$, w.p. $1-\delta$, $\forall~h\in[0:H-1]$, $\exists~ \wt\mu_h:\Scal\rightarrow \RR^d$ such that
\begin{align}
  \hspace{-0.5cm} \sup_\pi\EE_{\pi, P_\targ^\star}\nm{\wh\phi_h(s_h,a_h)^\top \wt\mu_h(\cdot) - \phi_h^\star(s_h,a_h)^\top\mu_{\targ;h}^\star(\cdot) }_{TV} \leq \eps_{TV} \coloneqq \sqrt{ |\Acal|\alpha^{3}_{\max}K\zeta_n/\lambda_{\min} }, \label{eq:average-approximately-linear}
\end{align}
and, for any function $g:\Scal \to [0,1]$, $\|\int g(s)\diff\wt\mu_h(s)\|_2\leq \bar\alpha\sqrt{d}$.
\end{restatable}
\noindent \pref{lem:mle_target} implies that the learned $\wh\phi$ is a feature such that $P^\star_\targ$ is \textit{approximately linear} in $\wh\phi$, under the occupancy distribution induced by \emph{any} policy.
Remarkably, this guarantee holds before the agent has ever interacted with the target task!
Intuitively, this is because cross-sampling ensures that our training data contains all possible states that
can be encountered in the target task. Failure modes without this can be found in the discussion
following \pref{thm:online-lb}.

% We remark that generative access does not trivialize the challenge of efficient exploration in the source tasks,
% since there are a potentially infinite number of states and the ground truth representation $\phi^\star$ is unknown.
% Prior works using generative access typically require either a known representation $\phi^\star$ (so that one can perform D-optimal design to construct an exploratory state-action distribution \citep{rltheorybookAJKS}),
% or directly assume access to a diverse state-action distribution which provides coverage and from which one can sample.
% Neither $\phi^\star$ nor such a diverse sampling distribution is given in our case.

The final step is to show that the deployment phase,
which runs \LSVIUCB{} in an approximately linear MDP of the target task, achieves low regret. Note that we face an approximately linear MDP, as \pref{lem:mle_target} shows, due to the use of learned features $\wh\phi$, even though $P^\star_{\targ}$ is linear in the unknown features $\phi^\star$.
Online learning in approximately linear MDPs has been studied in \citet{jin2020provably},
but under a much stronger $\ell_\infty$ error bound.
Instead, we work under the weaker, and more realistic, average-case misspecification in Eq.~\pref{eq:average-approximately-linear}.
Indeed, it is possible that some states are unlikely to be visited by any policy,
so we should not impose strong misspecification restrictions on these parts of the state space.
We now state our novel \LSVIUCB{} regret bound.
\begin{theorem}[\LSVIUCB{} under average-misspecification]\label{thm:lsvi-average-misspec}
Under Eq.~\pref{eq:average-approximately-linear},
for any $\delta\in(0,1)$, w.p. $1-\delta$, \LSVIUCB{} in the deployment phase
has regret
$
    \wt\Ocal\prns{ dH^2 T\eps_{TV} + \bar\alpha d^{1.5} H^2 \sqrt{T} \log(1/\delta) }.
$
\end{theorem}
The key step in proving \pref{thm:lsvi-average-misspec} is showing \emph{almost-optimism under the occupancy distribution of the optimal policy $\pi^\star$}.
As a technical remark, we also employ a novel trajectory-wise indicator to deal with the clipping of the value estimates.
Note the $\bar\alpha$ in the leading term comes from the scaling of the $\wt\mu$ in \pref{lem:mle_target}.
The full proof and general \LSVIUCB{} result is in \pref{thm:deployment_regret}.
To the best of our knowledge, this is the strongest result for learning in an approximately linear MDP, which may be of independent interest.

We arrive at the final regret bound by collecting enough samples in the source tasks to make $\eps_{TV}=1/\sqrt{T}$, which makes the first linear-in-$T$ term lower order.
This gives the \ouralg guarantee. Note that the guarantee holds independent of the mechanism used for obtain exploratory policies in the source tasks.
\begin{restatable}[\ouralg]{theorem}{generativeThm}\label{thm:main}
Suppose Assumptions~\ref{ass:shared_phi},\ref{ass:span},\ref{ass:generative_model},\ref{ass:fcn_class}, and $\pi_k$ is $\lambda_{\min}$-exploratory for each source task $k$.
Then, for any $\delta\in(0,1)$, w.p. $1-\delta$, \ouralg when deployed in the target task has regret at most
$\wt\Ocal\prns{\bar\alpha H^2d^{1.5}\sqrt{T\log(1/\delta)}}$, with at most $Kn$ generative accesses per source task, with \\$n = \Ocal\prns{ \lambda_{\min}^{-1}A\alpha_{\max}^3KT\prns{ \log\frac{|\Phi|}{\delta}+K\log|\Upsilon| } }$. \end{restatable}
\noindent Combining with the $\lambda_{\min}$ specified in \pref{lem:coverage}, we conclude the proof sketch.

\iffalse
In worst case, $\bar\alpha$ could scale in $K$ which means our bound scales linearly with respect to $K$. This is likely unavoidable since in worst case, the linear span implies that the target's $\mu^\star_\targ$ can have norm that scales in $K$, i.e., $\sup_{v:\|v\|_{\infty} \leq 1} \|v^\top \mu^\star_\targ\|_2 = \Theta(K \sqrt{d})$ which is $K$ times larger than the usual norm bound on the $\mu$ part for linear MDPs.
\fi

% Recall that our transfer learning algorithm, for source task $P_k^\star$, relies on an exploratory policy $\pi_k$ such that the empirical covariance with respect to $\phi_h^\star$ is lower bounded. This ensures good exploration in the underlying \emph{ground truth feature space} in $P_k^\star$. \pref{alg:reward_free} achieves this goal by first invoking the \textsc{RewardFree Rep-UCB} algorithm (\pref{alg:reward_free_repucb}) to learn an estimated linear MDP model $\wh P_k$ for $P^\star_k$.

\section{Failure of transfer learning without generative access to source tasks}\label{sec:online}

In the previous section, we show that efficient transfer learning is possible under very weak structural assumptions, but requires generative access to the source tasks. One natural question is whether transfer learning is possible with only online access to the source tasks. Somewhat surprisingly, we show that this is impossible without significantly stronger assumption.
\begin{theorem}[Lower bound for online access to source tasks]\label{thm:imp}
Let $\Mcal=\{(P_1^\star,...,P_K^\star,P_\targ^\star)\}$ be a set of $K+1$ tasks that satisfies (1) all tasks are Block MDPs; (2) all tasks satisfy \pref{ass:reachability} and \pref{ass:linear-span}; (3) the latent dynamics are \textit{exactly the same} for all source and target tasks.
For any pre-training algorithm $\Acal$ which outputs a feature $\hat\phi$ by interacting with the source tasks $k\in[K]$, there exists $(P_1^\star,...,P_K^\star,P_\targ^\star)\in\Mcal$, such that with probability at least $1/2$, $\Acal$ will output a feature $\hat\phi$, such that for any policy taking the functional form of $\pi(s) = f\left(\{\hat\phi(s,a)\}_{a\in\Acal},\{r(s,a)\}_{a\in\Acal}\right)$, we have
% \begin{align*}
$
V_{target}^\star - V_{target}^\pi \geq 1/2.
$
% \end{align*}
\label{thm:online-lb}
\end{theorem}
Here, the particular functional form $f$ is defined so that the policy $\pi$ cannot distinguish between two state-action pair with the same feature embedding.
\pref{thm:imp} implies that a representation learned only from online access to source tasks does not enable learning in downstream tasks if the downstream task algorithm is restricted to use the representation as the only information of the state-action pairs (e.g., running \LSVIUCB{} with $\hat\phi$).

We briefly explain the intuition behind the above lower bound. In a Block MDP, for any $(s,a)$, we can model the ground-truth $\phi^\star$ as a one-hot encoding $e_{(z,a)}$ corresponding to the latent state-action pair $(z,a)$ with $z = \psi^\star(s)$ being the encoded latent state. The key observation here is that any permutation of $\phi^\star$ will also be a perfect feature in terms of characterizing the Block MDPs, since it corresponds to simply permuting the indices of the latent states. Therefore, without cross referencing, the agent could potentially learn different permutations in different source tasks, which would collapse in the target task. A precise constructive proof of \pref{thm:imp} can be found in \pref{sec:comparison}.
Part of the reason that the above example fails is that each source task has its own observed subset of raw states,
which permits such a permutation to happen.

\subsection{Representational Transfer under Observational Reachability}
To complement the impossibility result, we next show that under an additional assumption on the reachability of raw states, a slight variant of the same algorithm (\pref{alg:online} in \pref{sec:transfer_proof}) can achieve the same regret with only online access to the source tasks. The main difference in \pref{alg:online} is that it performs
sampling directly from the occupancy distribution of $\pi_k$ in source task $k$ (in an online, episodic manner without needing generative access) instead of the cross-sampling procedure used in \pref{alg:main}.
\begin{assum}
[Reachability in the raw states]\label{ass:reachability_raw}
For all source tasks $k\in[K]$, any policy $\pi$ and $h\in[0:H-1]$, we have
$\inf_{s\in\Scal,a\in\Acal} d^\pi_{k;h}(s,a)\geq \psi_{raw} \lambda_{\min}\prns{\Eb[\pi,P_k^\star]{\phi^\star_h(s_h,a_h)\phi^\star_h(s_h,a_h)^\top}}$.
\end{assum}
\pref{ass:reachability_raw} implies that for each source task, any policy that achieves a full-rank covariance matrix also achieves global coverage across the raw state-action space. In addition, in order to apply importance sampling (IS) to transfer the TV error from source task to target task, we need to assume that the target task distribution has bounded density. This is true, for example, when $\Scal$ is discrete.% or when $\Scal$ is compact and has bounded measure. %Without loss of generality, we assume the upper bound is 1, formally stated as follows.
\looseness=-1
\begin{assum}[Bounded density]\label{ass:density}
For all $(\pi,h,s,a)$, we have $d^{\pi}_{\targ;h}(s,a)\leq 1$.
\end{assum}
\begin{restatable}[Regret with online access]{theorem}{onlineThm}\label{thm:online}
% Suppose \pref{ass:shared_phi},\pref{ass:fcn_class}, \pref{ass:reachability}, \pref{ass:span}, \pref{ass:reachability_raw} and \pref{ass:density}.
Suppose Assumptions~\ref{ass:shared_phi}-\ref{ass:span},\ref{ass:reachability_raw},\ref{ass:density} hold. W.p. $1-\delta$, \pref{alg:online} with appropriate parameters achieves a regret in the target
$\wt\Ocal\prns{ \bar\alpha d^{1.5}H^2\sqrt{T\log(1/\delta)} }$,
with \\$\op{poly}\prns{A,\alpha_{\max},d,H,K,T,\psi^{-1},\psi_{raw}^{-1},\log(|\Phi||\Upsilon|/\delta)}$ online queries in the source tasks. \looseness=-1
\end{restatable}
\pref{ass:reachability_raw} is satisfied in a Block MDP, when, for example, the emission function $o(s|z)$ satisfies that $\forall s, \exists z, \text{ s.t., } o( s| z) \geq c$. That is, for any source task, any state in the  state space can be generated by at least one latent state. However, we believe such a covering condition is generally too strong to hold in practice. Furthermore, the parameter $\psi_{raw}$ will typically scale with the number of observed states, which we expect to be prohibitively large in most interesting problems, and view the this result has mainly to quantify the degree of applicability of \pref{thm:imp}.

\section{Experiments}

\begin{figure}[!h]
\begin{tabular}{cc}
\begin{minipage}{0.3\textwidth}
\centering
\includegraphics[width=\textwidth]{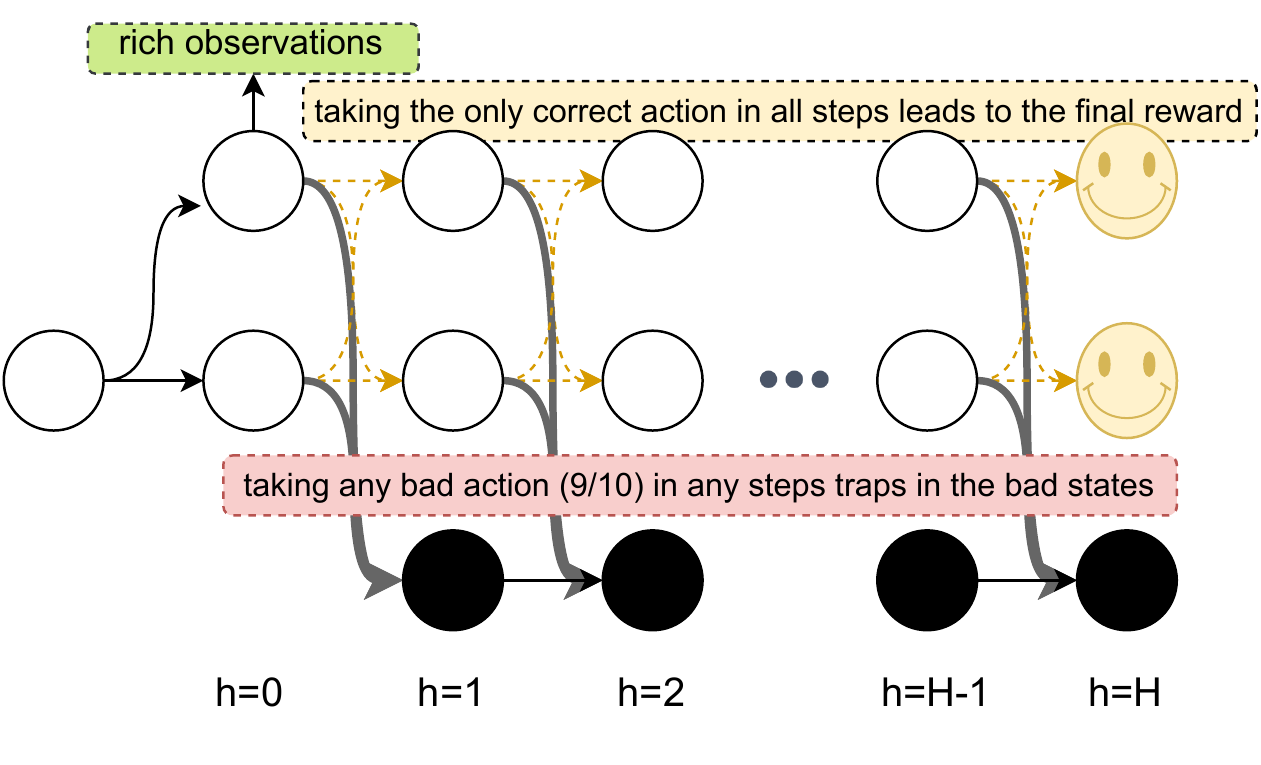}
(a)
\end{minipage}
&\;
\begin{minipage}{0.64\textwidth}
\renewcommand{\arraystretch}{2}
\begin{center}
\resizebox{\columnwidth}{!}{
\begin{tabular}{lccccc}
\toprule
                            &\; Source      &\; \ouralgo                    &\; \ouralgg                  &\; Oracle                            &\; Target                         \\
\hline
\makecell{Observational \\and Feature\\Coverage}                    &\; $\infty$    &\; \makecell{8006\\(294.7)}    &\; \makecell{7790\\(267.6)} &\;  \makecell{7048\\(164.8)}         &\; \makecell{181450\\(147600.2)}  \\
\hline
\makecell{Feature\\Coverage\\Only}     &\; $\infty$    &\; $\infty$                    &\; \makecell{7764\\(145.7)} &\;  \makecell{7336\\(270.7)}         &\; \makecell{85000 \\(14469.8)}    \\
\toprule
\end{tabular}
}
\\~\\
(b)
\end{center}
\end{minipage}
\end{tabular}
\caption {\textbf{(a):} A visualization of the rich observation comblock environment (see \pref{app:comblock} for details).
%Latent states (white and black) emit continuous high-dimensional observations. The reward is sparse (only in white states at $H$). Each white state has 10 actions where one of them is a good action that leads to the next two white states while the other 9 lead to black states (good action for each white state is different). Once the agent is in black states, it stays stuck in black until the end of the episode. Thus, a random exploration strategy has an exponentially small probability of hitting the goal (i.e., roughly $(1/10)^{H}$).
\textbf{(b) Top:} Number of episodes required to solve the target environment under the observational coverage setting. \textbf{(b) Bottom:} Number of episodes required to solve the target environment under the feature coverage setting. An algorithm solves the target task if it can achieve the optimal return (i.e., 1) for 5 consecutive iterations with 50 evaluation runs each. We include the mean and standard deviation (in the brackets) for 5 random seeds. $\infty$ denotes that an algorithm can not solve the target task within a fixed sample budget. The sample efficiency of \ouralg under feature and observational coverage verifies the benefit of representational transfer, and the failure of \ouralgo without observational coverage suggests the necessity of generative assumption during representational transfer. In \pref{fig:vis} in the appendix, we further provide the visualization of the representations that is learned in both settings. \looseness=-1}
\label{exp:comblock_and_table}
\end{figure}

In this section we empirically study the following questions: i) the effectiveness of pretraining with \ouralg, under the linear span and feature/observation coverage assumptions.  ii) the hardness of representational transfer under the linear span assumption without the generative model access. \looseness=-1
%Is representational transfer necessary/ what prevents directly applying learned source feature for the target task? Under what circumstances would it permit/prevent the online representational transfer? What is the practical benefit of representational transfer against directly training in the target?
%\begin{enumerate}[leftmargin=*,itemsep=0pt]
%\item Is representational transfer necessary/ what prevents directly applying learned source feature for the target task?
%\item  Under what circumstances would it permit/prevent the online representational transfer?
%\item What is the practical benefit of representational transfer against directly training in the target?
%\end{enumerate}
Our experiments are under the Block MDP setting, with the challenging Rich Observation Combination Lock (comblock) benchmark (Fig.~\ref{exp:comblock_and_table}(a)). We design two sets of experiments to investigate the above questions respectively.
We defer the details of the experiments in Appendix~\ref{app:sec:exp}.
%We defer the details of the experiments (the design of vanilla comblock, hyperparameters, etc.) in Appendix~\ref{app:sec:exp} and include the essential design information in the following two sections.

\paragraph{Baselines.} We denote \textbf{Source} as the smallest sample complexity of \LSVIUCB{} using learned features from any of the source tasks; \textbf{\ouralgo} as \ouralg with only online access to the source tasks;  \textbf{\ouralgg} as \ouralg with generative access to the source tasks; \textbf{Oracle} as learning in the target task with ground truth features; and \textbf{Target} as running \algname \citep{zhang2022efficient} --- the SOTA Block MDP algorithm, in the target task with no pretraining.
% \footnote{Importantly, BRIEE differs in how it learns features, but uses the same \LSVIUCB{} method at each step for policy learning with the learned features.}

\paragraph{Effectiveness of \ouralg.}
We first analyze the how representational transfer benefits where both feature coverage (\pref{thm:generative-end-to-end}) and observational coverage (\pref{thm:online}) assumptions are met. We use 5 source tasks (horizon $H=25$), with different latent dynamics. To ensure linear span assumption (\pref{ass:linear-span}), for each timestep $h$, we make the target latent transition dynamics from one of the sources uniformly at random. For the coverage assumptions, note that for comblock, the feature coverage assumption (\pref{ass:reachability}) is always satisfied. We also guarantee the observational coverage (\pref{ass:reachability_raw}) by equipping all environments with the same emission distribution on a compact observational space.
% In this section we start off with an easier setting: we use 5 source tasks, each with horizon $H=25$, 3 latent states per level and 10 actions. The latent transition dynamics are different for each source task, but notice that for comblock, the reachability assumption (c.r. \pref{ass:reachability}) is always satisfied. The emission distribution of all the source and target tasks is identical, so \pref{ass:reachability_raw} holds here. \wen{why this is true? isn't gonna depend on the emission distribution's property?}   %any source task is designed such that it has full support on the observation space (i.e., any observation has non-zero probability being visited in any source environment).
%We further allow each source environment to have slightly different noise during the emission process (i.e., Assumption \ref{ass:reachability_raw} is slightly violated), but we will show that in practice such minor misspecification does not prevent successful representational transfer if only online access to the sources is available.
%For the construction of the target task, for each timestep $h$, we choose the latent transition dynamics from one of the sources uniformly at random (thus \pref{ass:linear-span} is satisfied). \looseness=-1
We record the number of episodes in the target environment that each method takes to solve the target environment in Table.~\ref{exp:comblock_and_table}(b). We first observe that \ouralg with either online or generative access can solve the target task (since. \pref{ass:reachability_raw} holds). % This is because in this setting, for any source task, any observation can be generated by at least one latent state.
%even in face of minor support mismatch between the source and target observations. We hypothesize that this issue is mitigated by the function approximation from a rich enough function class (see Appendix.~\ref{app:sec:exp} for details).
Second, we observe that directly applying the learned feature from any \emph{single source task} does not suffice to solve the target environment. This is because the representation learned from a single source task may collapse two latent states into a single one during encoding (e.g., if two latent states at the same time step have exactly identical latent transitions). %See the visualization of comparisons between the learned decoders from a single source task and \ouralg in Fig.~\ref{exp:fig:vis}(a). %\yuda{The current plot is really hard to see unless make much bigger. Should we defer them in the appendix?}
Third, the result shows that \ouralg saves order of magnitude of target samples compared with training in the target environment from scratch using the SOTA Block MDPs algorithm BRIEE. \textbf{This set of results verifies the empirical benefits of representation learning from multiple tasks, i.e., resolves ambiguity and speeds up downstream task learning}.

\paragraph{Hardness without the generative access.}
In this section, following the intuition of our lower bound  (\pref{thm:imp}), we construct a setting where the supports of the emission distributions from each task are completely disjoint, while the emission distribution in the target task is a mixture of all source emissions and the latent dynamics are identical across tasks. Hence the latent coverage (\pref{ass:linear-span}) holds while observational coverage (\pref{ass:reachability_raw}) fails. So we expect that an algorithm without generative access to source tasks will fail based on \pref{thm:online-lb}. %Specifically, under disjoint emission supports, a representation can decode the latent state correctly for each source, but permute latent state labels across sources, causing decoding errors on the target.
%This is an exact practical version of the hardness construction in Thm.~\ref{thm:imp} without identical latent transition (note the identical latent transition is not necessary, and we used randomized latent transition for a more general result).
We record the number of target episodes for each method to solve the target task in Table.~\ref{exp:comblock_and_table}(b). We observe that indeed the online version fails while the generative version still succeeds.
%In practice, we note that two source environments suffice to validate our theoretical hypothesis that online representational transfer requires stronger assumptions (such as assumption.\ref{ass:reachability_raw}), and when they are completely violated, online representational transfer is impossible.
%We show the visualizations of \ouralgo and \ouralgg in Fig.~\ref{exp:fig:vis}(b), and we note that the empirical results exactly match our theoretical results.
\textbf{This ablation verifies that source generative model access is needed without the observational coverage}.% unless one has additional stronger distributional assumptions.
%The same result holds for using source features because and we still gain a significant boost in the sample efficiency from \ouralg than cold-started representation learning in the target environment. We defer all details (such as environment constructions and the visualization of the decoders from all baselines in Appendix.~\ref{app:sec:exp}).

\section{Conclusion}

We study representational transfer among low rank MDPs which share the same unknown representation. Under a reasonably flexible linear span task relatedness assumption, we propose an algorithm that provably transfers the representation learned from source tasks to the target task. The regret in target task matches the bound obtained with oracle access to the true representation, using only polynomial number of samples from source tasks. Our approach relies on the generative model access in source tasks, which we prove is not avoidable in the worst case under the linear span assumption. To complement the lower bound, we propose a stronger assumption on the conditions of the reachability in raw states, under which online access to source tasks suffices for provably efficient representation transfer. Finding modalities other than generative access which avoid the lower bound, and a more extensive empirical evaluation beyond the proof-of-concept experiments here are important directions for future research

\bibliography{ref}

\newpage

\appendix

\setlength{\parindent}{0pt}
\setlength{\parskip}{\baselineskip}

\begin{center}\LARGE
\textbf{Appendices}
\end{center}

\section{Notations}

{\renewcommand{\arraystretch}{1.2}% for the vertical padding
\begin{table}[h!]
    \centering
      \caption{List of Notations} \vspace{0.3cm}
    \begin{tabular}{l|l}
    $\Scal, \Acal, A$ & State and action spaces, and $A = |\Acal|$. \\
    $\Delta(S)$ & The set of distributions supported by $S$. \\
    $\lambda_{\min}(A)$ & Smallest eigenvalue of matrix $A$. \\
    $e_j$ & One-hot encoding of $j$, i.e. $0$ at each index except the one corresponding to $j$. \\
        & Length of vector implied from context. \\
    $(x)_{\leq y}$ & $\min\{x,y\}$. \\
    $H$ & Episode length of MDPs, a.k.a. time horizon. We index steps as $h=0,1,...,H-1$. \\
    $K$ & The number of source tasks. \\
    $d$ & dimension of the low-rank MDP, i.e. dimension of $\phi^\star$. \\
    $P^\star_{k;h}$ & Ground truth transition at time $h$ for source task $k$. \\
    $P^\star_{\targ;h}$ & Ground truth transition at time $h$ for target task. \\
    $r_{\targ;h}$ & Reward function of the target task. \\
    $\Eb[\pi,P]{\cdot}$ & Expectation under the distribution of trajectories when $\pi$ is executed in $P$. \\
                        & We sometimes omit $P$ when the MDP is clear from context. \\
    $d^\pi_{P;h}$ & Occupancy distribution of $\pi$ under transitions $P$ at time $h$. \\
    $d^\pi_{k;h}$ & Occupancy distribution for the $k$-th task, i.e. $d^\pi_{P^\star_{k};h}$. \\
    $\phi_h^\star(s,a)$  & Embedding function for $(s,a)$ at time $h$. \\
    $\Phi_h$  & Realizable function class for $\phi_h^\star$. \\
    $|\Phi|$  & Defined as $\max_h |\Phi_h|$. \\
    $\mu_{k;h}^\star(s')$ & Emission embedding function for $s'$ at time $h$ for environment $k$. \\
    $\Upsilon_{k;h}$ & Realizable function class for $\mu_{k;h}^\star$. \\
    $|\Upsilon|$ & Defined as $\max_{k;h} |\Upsilon_{k,h}|$. \\
    $\alpha_{\max}$ & $\max_{h;k,s'\in\Scal} |\alpha_{k;h}(s')|$ (based on \pref{ass:span}). \\
    $\bar\alpha$ & $\max_h \sum_{k=1}^{K} \max_{s'\in\Scal} |\alpha_{k;h}(s')|$ (based on \pref{ass:span}). \\
    $\psi$ & Feature reachability in the source task (\pref{ass:reachability}). \\
    $\psi_{raw}$ & Raw states reachability parameter (\pref{ass:reachability_raw}) \\
    \end{tabular}
    \label{tab:notation}
\end{table}
}

\newpage
\section{Omitted Algorithms}

\begin{algorithm}
	\caption{\textsc{RewardFree Rep-UCB}}
	\label{alg:reward_free_repucb}
	\begin{algorithmic}[1]
	\STATE \textbf{Input: } Regularizer $\lambda_n$, bonus scaling $\alpha_n$, model class $\Mcal = \Phi\times\Upsilon$, number of episodes $N$.
	\STATE Initialize $\wh\pi_0$ as random and $\Dcal_{h,0},\Dcal_{h,0}' = \emptyset$.
	\FOR{episode $n = 1, 2, ..., N$}
	    \STATE Data collection from $\wh\pi_{n-1}$: for $h=1,2,...,H-1$,
	    \begin{align*}
	        &s\sim d^{\wh\pi_{n-1}}_h, a\sim\op{Unif}(\Acal), s'\sim P^\star_h(s,a);
	        \\&\wt s\sim d^{\wh\pi_{n-1}}_{h-1}, \wt a\sim\op{Unif}(\Acal), \wt s'\sim P^\star_{h-1}(\wt s,\wt a), \wt a'\sim\op{Unif}(\Acal), \wt s''\sim P^\star_h(\wt s',\wt a');
	        \\&\Dcal_{h,n} = \Dcal_{h,n-1}\cup\braces{(s,a,s')}, \Dcal_{h,n}' = \Dcal_{h,n-1}'\cup\braces{(\wt s',\wt a',\wt s'')}.
	    \end{align*}
	    For $h=0$, only collect $\Dcal_{0,n}$.
	    \STATE Learn model via MLE: for all $h=0,1,...,H-1$,
	    \begin{align*}
	        \wh P_{h,n} = (\wh\phi_{h,n},\wh\mu_{h,n}) = \argmax_{\phi_h,\mu_h\in\Mcal_h}\Eb[\Dcal_{h,n}\cup\Dcal_{h,n}']{\log \phi_h(s,a)^T\mu_h(s')}.
	    \end{align*}
	    \STATE Update exploration bonus: for all $h=0,1,...,H-1$,
	    \begin{align*}
	        &\wh b_{h,n}(s,a) = \alpha_n\nm{\wh\phi_{h,n}(s,a)}_{\wh\Sigma_{h,n}^{-1}}
	        \\&\wh\Sigma_{h,n} = \sum_{(s,a,\_)\in\Dcal_{h,n}}\wh\phi_{h,n}(s,a)\wh\phi_{h,n}(s,a)^T + \lambda_n I.
	    \end{align*}
	    \STATE Learn policy $\wh\pi_n = \argmax_\pi V^\pi_{\wh P_n,\wh b_n}$ and let $\wh V_n$ be its value. \label{line:rep_ucb_planning_step}
	\ENDFOR
	\STATE Let $\wh n = \argmin_{n\geq N/2} \wh V_n$.
	\STATE \textbf{Output: } $\wh n, \wh P_{\wh n}$.
	\end{algorithmic}
\end{algorithm}

\newpage
\begin{algorithm}
	\caption{Transfer learning with online access}\label{alg:online}
	\begin{flushleft}
	    \textsc{\textbf{Pre-Training Phase}}\\
	    \textbf{Input:} num. \LSVIUCB{} episodes $N_\LSVIUCB$, num. model-learning episodes $N_{\textsc{RewardFree}}$, size of cross-sampled datasets $n$, failure probability $\delta$.
	\end{flushleft}
	\begin{algorithmic}[1]
		\FOR{source task $k=1,...,K-1$}
		    \STATE Find policy cover $\pi_k=$\textsc{RewardFree}$(P_k^\star, N_\LSVIUCB, N_{\textsc{RewardFree}}, \delta)$. (\pref{alg:reward_free})
		\ENDFOR
		\FOR{source task $k=1,...,K-1$}
		    \STATE For each $h\in[0:H-1]$, sample $\Dcal_{k}$ as $n$ i.i.d. $(s_h,a_h,s_{h+1})$ tuples from $\pi_{k}$.
		\ENDFOR
        \STATE For each $h\in[0:H-1]$, learn features with MLE,
        \begin{align*}
            \wh\phi_h,\wh\mu_{1:K} = \argmax_{\phi\in\Phi,\mu_k\in\Upsilon_k}\sum_{k\in[K]}\EE_{\Dcal_{k;h}}\bracks{ \log\phi(s,a)^\top\mu_k(s') }.
        \end{align*}
	\end{algorithmic}
	\begin{flushleft}
		\textsc{\textbf{Deployment Phase}} \\
		\textbf{Additional Input:} number of deployment episodes $T$.
	\end{flushleft}
	\begin{algorithmic}[1]
	\STATE Set $\beta = H\sqrt{d} + \bar\alpha dH\sqrt{\log(dHT/\delta)}$.
	\STATE Run $\LSVIUCB\prns{\{ \wh\phi_h\}_{h=0}^{H-1}, r=r_{K}, T, \beta}$ in the target task $P_\targ^\star$ (\pref{alg:LSVI}).
	\end{algorithmic}
\end{algorithm}

\newpage
Let $(x)_{\leq y}$ refer to the clamping operator, i.e. $(x)_{\leq y} = \min\{x,y\}$.
Let $M_V$ be the maximum possible value in the MDP with the given reward function.
\begin{algorithm}[!h]
	\caption{\LSVIUCB}
	\label{alg:LSVI}
	\begin{algorithmic}[1]
	    \STATE \textbf{Input: } Features $\{\wh\phi_h\}_{h=0,1,...,H-1}$, reward $\{r_h\}_{h=0,1,...,H-1}$, number of episodes $N$, bonus scaling parameter $\beta$, \textsc{UniformActions} = \textsc{False}.
	    \FOR{\text{episode } $e=1,2,...,N$}
	    	\STATE Initialize $\wh V_{H,e}(s) = 0, \forall s$
	        \FOR{\text{step } $h=H-1,H-2...,0$}
	            \STATE Learn best predictor for $\wh V_{h+1}^e$,
	            \begin{align*}
	                &\Lambda_{h,e} = \sum_{k=1}^{e-1} \wh\phi_h(s_h^k,a_h^k)\wh\phi_h(s_h^k,a_h^k)^\top + I,
	                \\&\wh w_{h,e} = \Lambda_{h,e}^{-1} \sum_{k=1}^{e-1} \wh\phi_h(s_h^k,a_h^k)\wh V_{h+1,e}(s^k_{h+1}).
	            \end{align*}
	            \STATE Set bonus and value functions,
	            \begin{align*}
	                &b_{h,e}(s,a) = \nm{\wh\phi_h(s,a)}_{\Lambda_{h,e}^{-1}},
	                \\&\wh Q_{h,e}(s,a) = \wh w_{h,e}^\top \wh\phi_h(s,a) + r_h(s,a) + \beta b_{h,e}(s,a),
	                \\&\wh V_{h,e}(s) = \prns{ \max_a \braces{\wh Q_{h,e}(s,a)} }_{\leq M_V}.
	            \end{align*}
	            \STATE Set $\pi_h^e(s) = \argmax_{a} \wh Q_{h,e}(s,a)$.
	        \ENDFOR
	        \STATE Execute $\pi^e$ to collect a trajectory $(s^e_h, a_h^e)_{h=0}^{H-1}$.
	        \STATE If \textsc{UniformActions} = \textsc{True}, discard $a_h^e$ and draw freshly sampled uniform actions independently for all $h$, i.e. $a_h^e \sim \op{Unif}(\Acal)$. \label{line:unif-actions-lsvi-collect-data}
	    \ENDFOR
	    \STATE\textbf{Return: } uniform mixture $\rho = \text{Uniform}(\{\pi^e\}_{e=1}^{N})$.
	\end{algorithmic}
\end{algorithm}

% {\renewcommand{\arraystretch}{1.2}% for the vertical padding
% \begin{table}[h!]
%     \centering
%       \caption{List of Algorithms} \vspace{0.3cm}
%     \begin{tabular}{l|l}
%     \ouralg (\pref{alg:main}) & Transfer learning with generative access. \\
%     \textsc{RewardFree} (\pref{alg:reward_free}) & Reward Free exploration in low-rank MDP. \\
%     Multi-task \textsc{RepLearn} (\pref{alg:replearn}) & Multi-task Maximum Likelihood (MLE) Representation Learning. \\
%     \textsc{RewardFree \textsc{Rep-UCB}} (\pref{alg:reward_free_repucb}) & Modified \textsc{Rep-UCB} for reward-free model learning. \\
%     \ouralg for online access (\pref{alg:online}) & Simplified version of \ouralg if we have  \\
%         & reachability under raw states (see \pref{thm:online}). \\
%     \LSVIUCB{} (\pref{alg:LSVI}) & Optimistmic Least Squares Value Iteration (with bonus). \\
%     \end{tabular}
%     \label{tab:list-of-algs}
% \end{table}
% }

\newpage
\section{ More discussion on related works }\label{sec:compare}
\subsection{Empirical works in transfer learning}\label{sec:compare-empirics}
The idea of learning transferable representation has been extensively explored in the empirical literature. Here we don't intend to provide a comprehensive survey of all existing works on this topic. Instead, we discuss a few representative approach that may be of interest.
\\
Towards transfer learning across different environments, progressive neural network \citep{rusu2016progressive} is among the first neural-based attempt to learning a transferable representation for a sequence of downstream tasks that tries to overcome the challenge of catastrophic forgetting. It maintains the learned neural models for all previous tasks and introduce additional connections between the network of the current tasks to those of prior tasks to allow information reuse. However, a drawback common to such an approach is that the network size grows linearly with the number of tasks. Other approaches include directly learning a multi-task policy that can perform well on a set of source tasks, with the hope that it will generalize to future tasks \citep{parisotto2015actor}. Such an approach requires the tasks to be similar in their optimal policy, which is a much stronger assumption than ours.
\\
Slightly off-topic are the works about ``transfer learning'' inside the same environment but across different reward functions, which is more restricted than the setting considered in this paper.
% Such a learning scenario is referred to as reward-free RL \citep{jin2020reward} or task-agnostic RL \citep{zhang2020task} in the theoretical RL community and has been shown to not require significantly different algorithmic treatment than task-specific learning.
Several prior works design representation learning algorithms that aim to learn a representation that generalize across multiple reward function/goals \citep{dayan1993improving, barreto2017successor,touati2021learning,blier2021learning}. These are related to the \textsc{RewardFree Rep-UCB} we developped in \pref{sec:rf_repucb}. The key difference is that we concern representation learning along with efficient exploration to derive an end-to-end polynomial sample complexity bound. These prior works do not consider exploration and do not come with provable sample complexity bounds.
We refer interested readers to a recent survey \citep{zhu2020transfer} for a comprehensive discussion of other empirical approaches.

\subsection{Comparison to \citet{lu2021power}}
In the prior work of \cite{lu2021power}, which also studies transfer learning in low-rank MDPs with nonlinear function approximations, they need to make the following assumptions:
\begin{enumerate}[leftmargin=*,itemsep=0cm,topsep=0.01cm]
    \item shared representation (identical to our \pref{ass:shared_phi}).
    \item task diversity (similar to our \pref{ass:linear-span}).
    \item generative model access to both the source and the target tasks. In contrast, we only require generative model access to the source tasks and allow online learning in the target task.
    \item a somewhat strong coverage assumption saying that the data covariance matrix (under the generative data distribution) between arbitrary pairs of features $\phi,\phi'\in\Phi$ must be full rank. In contrast, our analysis only requires coverage in the true feature $\phi^\star$ in the source tasks.
    \item the existence of an ideal distribution q on which the learned representation can extrapolate. We do not require an assumption of a similar nature. Instead, we show that the data collected from our strategic reward-free exploration phase suffices for successful transfer.
    \item the uniqueness for each $\phi$ in the sense of linear-transform equivalence. Two representation functions $\phi$ and $\phi'$ can yield similar estimation result if and only if they differ by just an invertible linear transformation. In contrast, we do not make any additional structural assumptions on the function class $\Phi$ beyond realizability.
\end{enumerate}
In summary, our work present a theoretical framework that permits successful representation transfer based on significantly weaker assumptions. We believe that this is a solid step towards understanding transfer learning in RL.

\subsection{Comparison to \cite{cheng2022provable}}
% Recall that in addition to the commonly made shared representation assumption (\pref{ass:shared_phi}), we made two additional structural assumptions: reachability (\pref{ass:reachability}) and linear span on $\mu^\star$(\pref{ass:linear-span}). The linear span assumption is closely related to the diversity assumptions made in prior works of transfer learning in both supervised learning \citep{tripuraneni2020theory} and reinforcement learning \citep{lu2021power}.
\cite{cheng2022provable} studies representational transfer in low-rank MDPs, with not only a weaker notion of task relatedness (with global coefficients in the linear span) but also stronger assumptions.
Particularly, we restate the following strong assumption from \citet[Assumption 5.3]{cheng2022provable}.
\begin{assum}
For any two different models in the model class $\Phi\times\Psi$, say $P^1(s'|s,a)=\langle \phi^1(s,a),\mu^1(s')\rangle$ and $P^2(s'|s,a)=\langle \phi^2(s,a),\mu^2(s')\rangle$, there exists a constant $C_R$ such that for all $(s,a)\in \Scal\times\Acal$ and $h\in[H]$,
\begin{align*}
    \|P^1(\cdot|s,a)-P^2(\cdot|s,a)\|_{TV}\leq C_R \EE_{(s,a)\sim \Ucal(\Scal,\Acal)}\|P^1(\cdot|s,a)-P^2(\cdot|s,a)\|_{TV}
\end{align*}
\end{assum}
This assumption ensures that the point-wise TV error is bounded, as long as the population-level TV error is bounded. \citet{cheng2022provable} used this to transfer the MLE error from the source tasks to the target task. This type of assumption is strong in the sense that we typically expect $C_R$ to scale with $|\Scal|$.
In contrast, our analysis (\pref{lem:mle_target_online}) shows that this assumption is in fact not necessary, even assuming online access only to source tasks. The generative access to source task studied here, which enables transfer under weaker reachability assumptions is not studied in their work.

It is worth noting that \cite{cheng2022provable} also study offline RL in the target task which we do not cover, while we mainly focus on the setting of generative models in the source tasks and demonstrating a more complete picture by proving generative model access in source tasks is needed without additional assumptions. Comparing to \citep{cheng2022provable}, we also further implement and perform experimental evaluations of our algorithm.

\subsection{Works in multi-task learning}
\textbf{Multi-task and Transfer Learning in Supervised Learning.} The theoretical benefit of representation learning are well studied under conditions such as the i.i.d. task assumption \citep{maurer2016benefit} and the diversity assumption \citep{du2020few, tripuraneni2020theory}. Many works below successfully adopt the frameworks and assumptions to sequential decision making problems.

\textbf{Multi-task and Transfer Learning in Bandit and small-size MDPs.}
Several recent works study multi-task linear bandits with linear representations ($\phi(s) = A\,s$ with unknown $A$) \citep{hu2021near,yang2020impact,yang2022nearly}. The techniques developed in these works crucially rely on the linear structure and can not be applied to nonlinear function classes. \cite{lazaric2013sequential} study spectral techniques for online sequential transfer learning.
\cite{brunskill2013sample} study multi-task RL under a fixed distribution over finitely many MDPs, while
\cite{brunskill2014pac} consider transfer in semi-MDPs by learning options. \cite{lecarpentier2021lipschitz} consider lifelong learning in Lipschitz MDP. All these works consider small size tabular models while we focus on large-scale MDPs.  \looseness=-1

\textbf{Multi-task and Transfer Learning in RL via representation learning.} Beyond tabular MDPs, \cite{arora2020provable} and \cite{d2019sharing} show benefits of  representation learning in imitation learning and planning, but do not address exploration.
\cite{lu2021power} study transfer learning in low-rank MDPs with general nonlinear representations, but make a generative model assumption on both the source tasks and the target task, along with other distributional and structural assumptions. We do not require generative access to the target task and make much weaker structural assumptions on the source-target relatedness.
Recently and independently, \cite{cheng2022provable} also studied transfer learning in low-rank MDPs in the online learning setting, identical to the setting we study in \pref{sec:online}. However, their analysis relies on an additional assumption that bounds the point-wise TV error with the population TV error, which we show is in fact not necessary.

\textbf{Efficient Representation Learning in RL.} Even in the single task setting, efficient representation learning is an active area witnessing recent advances with exploration~\citep{Agarwal2020_flambe,modi2021model,uehara2021representation,zhang2022efficient} or without \citep{ren2021free}. Other papers study feature selection~\citep[e.g.][]{farahmand2011model,jiang2015abstraction, pacchiano2020regret, cutkosky2021dynamic,lee2021online,zhang2021provably} or sparse models \citep{hao2021sparse,hao2021online}.

\newpage
\section{Impossibility Results}\label{sec:comparison}

Here, we present an interesting result showing that the above assumptions we make are so weak that they do not even permit efficient transfer in supervised learning:

\begin{theorem}[Counter-example in supervised learning]\label{thm:counter_example_SL}
Assume that we want to perform conditional density estimation, where $P^\star_k(y|x) = \phi^\star(x)^\top \mu_k^\star(y)$. Under \pref{ass:shared_phi} (shared representation) and \pref{ass:linear-span} (linear span), and assume that in each source task, one have access to a data generating distribution $\rho_k(x)$ such that $\lambda_{\min} \EE_{\rho_k}[\phi^\star(x)\phi^\star(x)^\top]\geq\psi$ (reachability). No algorithm can consistently achieve $\EE_{\rho_K}[\|\hat P^\star_\targ(y|x)- P^\star_\targ(y|x)\|_{TV} ]\leq 1/2$ on the target task using the feature learned from the source tasks with probability more than $1/2$.
\end{theorem}
\begin{proof}[Proof of \pref{thm:counter_example_SL}]
Consider the following example. $X = \RR^2$ and we have the following 3 sets.
\begin{align*}
    \Scal_1 =&\; B_{1/2}((-1,-1))\\
    \Scal_2 =&\; B_{1/2}((-2,-2))\\
    \Scal_3 =&\; B_{1/2}((0,1))
\end{align*}
where $B_{a}((x,y))$ stands for the ball with radius $a$ centered at $(x,y)$. These will be the support of 3 tasks: task 1 and 2 are two source tasks, task 3 is the target task. Let's assume that $P^\star_k(x)$ are uniform distribution on $\Scal_k$.

Suppose that the feature class $\Phi$ only contains two functions:
\begin{align*}
        \phi_1: \{x_1\leq 0 \And x_2\geq x_1\}\rightarrow (1,0) \text{, and $(0,1)$ otherwise}\\
    \phi_2: \{x_2\leq 0 \And x_1\geq x_2\}\rightarrow (0,1) \text{, and $(1,0)$ otherwise}
\end{align*}
That is, the feature maps from $\RR^2$ to the set of binary encoding of dimension 2, i.e. $\{(1,0),(0,1)\}$. We further assume that $\mu^\star_k = (p_1(y),p_2(y))$ for some distributions $p_1,p_2$, which is identical for all task $k$, where $\|p_1,p_2\|_{TV}=1$. We also assume that $\mu^\star_k$ is known to the learner a prior, i.e. $\Upsilon_k=\{\mu^\star_k\}$ for all $k\in[K]$, so all the learner needs to do is to pick the correct $\phi$ out of two candidates.

Given the above setup, it's easy to verify that both \pref{ass:shared_phi} and \pref{ass:linear-span} are satisfied, because the decision boundary of both $\phi_1$ and $\phi_2$ passes through the support of the source tasks, and all $\mu^\star_k$'s are identical. However, $\phi_1$ and $\phi_2$ are equivalent in $\Scal_1$ and $\Scal_2$ in terms of their representation power, therefore no algorithm can always pick the correct feature function with probability more than $1/2$, regardless of the number of samples. Suppose $\phi_1$ is the true feature and the algorithm incorrectly chooses $\phi_2$. Then, for $x\in S_3\bigcap \{x_1\geq 0\}$ which has probability mass $1/2$, $\hat P_3(y|x)=p_2$ whereas $P^\star_3(y|x)=p_1$. Thus, the expected total variation distance between $\hat P_\targ$ and $P^\star_\targ$ is $1/2$.
\end{proof}

The above construction shows that our assumption are not sufficient to permit reliable representation transfer, even in the supervised learning setting.
Yet, surprisingly, these assumptions are sufficient in the RL setting, implying somehow that transfer learning in RL in easier than transfer learning in SL. To understand this phenomenon, observe that in RL, the marginal distribution on $(s,a)$ is not independent from the conditional density $P(s'|s,a)$ we desire to estimate. In particular, if one collects data in the source tasks in an online fashion via running a policy, $\rho(s,a)$ is structurally restricted to be an occupancy distribution generated by the ground-truth transition $P^\star(s'|s,a)$. Such a connection can only exist in Markov chains, and our analysis elegantly utilizes this additional structure to establish the soundness of the learned representation. Also note, crucially, that we never learn a representation to capture $d_0$, which would suffer from similar issues as the supervised learning setting, but is not necessary for sample-efficient RL.

Next, we prove the impossibility result in \pref{thm:imp}, restated below as \pref{thm:imp2}. This result shows that one can not achieve online learning in the source tasks without significantly stronger assumptions such as \pref{ass:reachability_raw}. Before that, we provide a preliminary version, showing that the learned $\hat\phi$ is not sufficient to fit the transition model in the target task, which motivates the construction in \pref{thm:imp2}.
\begin{theorem}[Impossibility Result: Model Learning]\label{thm:imp1}
Let $\Mcal=\{(P_1,...,P_K,P_\targ)\}$ be a set of $K+1$ tasks that satisfies
\begin{enumerate}[leftmargin=*,itemsep=0pt]
    \item all tasks are Block MDPs;
    \item all tasks satisfy \pref{ass:reachability} and \pref{ass:linear-span};
    \item the latent dynamics are \textit{exactly the same} for all source and target tasks.
\end{enumerate}
For any pre-training algorithm $\Acal$, there exists $(P_1,...,P_K,P_\targ)\in\Mcal$ and an occupancy distribution $\rho_\targ$ on the target task, such that with probability at least $1/2$, $\Acal$ will output a feature $\hat\phi$ and for any $\mu$
\begin{align*}
    \EE_{\rho_\targ}\|\hat\phi(s,a)^\top\mu(\cdot)-P^\star_\targ(\cdot|s,a)\|_{TV}\geq 1/2.
\end{align*}
\end{theorem}

\begin{proof}[Proof of \pref{thm:imp1}]
Consider a tabular MDP with 2 latent states $z_1,z_2$ and an observation state space $\Scal = R_1\bigcup R_2\bigcup B_1\bigcup B_2$, where in task 1 one can only observe $R_1\bigcup R_2$ and in task 2 one can only observe $B_1\bigcup B_2$. Correspondingly, $o_1(s|z)$ is only supported on $R_1\bigcup R_2$ (i.e., $o_1(R_i|z_i) = 1$) and similar for task 2. Let the latent state transition be such that $P(z_1|z_1,a)=1$ and $P(z_2|z_2,a)=1$, i.e. only self-transition regardless of the actions.

Now, consider a 2-element feature class $\Psi = \{\psi_1,\psi_2\}$ such that
\begin{align*}
    \psi_1 =&\; \{R_1\rightarrow 1, R_2\rightarrow 2, B_1\rightarrow 1, B_2\rightarrow 2\}\\
    \psi_2 =&\; \{R_1\rightarrow 1, R_2\rightarrow 2, B_1\rightarrow 2, B_2\rightarrow 1\}
\end{align*}
Denote $\phi_i(s,a)=e_{(\psi_i(s),a)}$ for $i\in[1,2]$.
Consider for each task $k$, a 2-element $\Upsilon_k$ class in the form of
$\Upsilon_k = \{(o_k(s|z_1),o_k(s|z_2)),(o_k(s|z_2),o_k(s|z_1))\}
$.

Notice that $\phi_1$ and $\phi_2$ are merely permutations of one another and so given any single task data, the two hypothesis will not be distinguishable by any means. Therefore, for any algorithm, there is at least probability $1/2$ that it will choose the wrong hypothesis if the ground truth $\phi^\star$ is sampled between $\phi_1$ and $\phi_2$ uniformly at random. Suppose $\phi_1$ is the correct hypothesis and $\phi_2$ is the one that the algorithm picks (i.e., $\hat\phi= \phi_2$). Let task 3 be such that any state emits to $R_1\bigcup R_2$ and $B_1\bigcup B_2$ each with probability $1/2$ (i.e., $o_3(R_i | z_i) = o_3(B_i|z_i) = 0.5$). This construction satisfies \pref{ass:reachability} and \pref{ass:linear-span}.

Then, within task 3, one would encounter observations from both $R_1$ and $B_2$ which should be mapped to latent state $z_1$ and $z_2$ respectively by the true decoder $\phi_1$, but are instead both mapped to latent state $z_1$ by the learned decoder $\phi_2$, and thus $z_1$ and $z_2$  become indistinguishable. Suppose $\rho_\targ(z_1) = \rho_\targ(z_2)=1/2$, then
\begin{align*}
    &\;\EE_{\rho_\targ}[\|\hat\phi(s,a)^\top\mu(\cdot)-P^\star(\cdot|s,a)\|_{TV}]\\
     &\; =  \frac{1}{4}\|\hat\phi(R_1)^\top\mu(\cdot)-\phi(R_1)^{\star\top}\mu^\star(\cdot)\|_{TV} + \frac{1}{4}\|\hat\phi(B_1)^\top\mu(\cdot)-\phi(B_1)^{\star\top}\mu^\star(\cdot)\|_{TV} \\
    &\; \quad \quad +  \frac{1}{4}\|\hat\phi(R_2)^\top\mu(\cdot)-\phi(R_2)^{\star\top}\mu^\star(\cdot)\|_{TV} + \frac{1}{4}\|\hat\phi(B_2)^\top\mu(\cdot)-\phi(B_2)^{\star\top}\mu^\star(\cdot)\|_{TV} \\
    &\; = \| o_1 - o_1^\star  \|_{TV} / 4  + \| o_2 - o_1^\star   \|_{TV} / 4 + \| o_2 - o_2^\star \|_{TV} / 4  + \| o_1 - o_2^\star  \|_{TV} / 4 \\
    &\; \geq \frac{1}{4} \| o_1^\star - o_2^\star  \|_{TV}  + \frac{1}{4} \| o_1^\star - o_2^\star  \|_{TV} \\
    &\; = \frac{1}{2},
\end{align*}
where the last second inequality uses triangle inequality, and the last equality comes from the fact that $o_3(\cdot | z_1)$ and $o_3(\cdot | z_2)$ have disjoint support which implies that $\|p_1^\star - p_2^\star \|_{TV} = 1$.  \end{proof}

Now, we are ready to restate and prove \pref{thm:imp}.
\begin{theorem}[Impossibility Result: Optimal Policy Identification]\label{thm:imp2}
Let $\Mcal=\{(P_1,...,P_K,P_\targ)\}$ be a set of $K+1$ tasks that satisfies
\begin{enumerate}[leftmargin=*,itemsep=0pt]
    \item all tasks are Block MDPs;
    \item all tasks satisfy \pref{ass:reachability} and \pref{ass:linear-span};
    \item the latent dynamics are \textit{exactly the same} for all source and target tasks.
\end{enumerate}
For any pre-training algorithm $\Acal$, there exists $(P_1,...,P_K,P_\targ)\in\Mcal$, such that with probability at least $1/2$, $\Acal$ will output a feature $\hat\phi$, such that for any policy taking the functional form of $\pi(s) = f\left(\{\hat\phi(s,a)\}_{a\in\Acal},\{r(s,a)\}_{a\in\Acal}\right)$, we have
\begin{align*}
    V^\star - V^\pi \geq 1/2.
\end{align*}
\end{theorem}
\begin{proof}[Proof of \pref{thm:imp2}]
Consider a tabular MDP with $H=2$, two latent states $z_1,z_2$ for $h=1$ and two latent states $z_3,z_4$ for $h=2$.

\begin{itemize}
\item
For $h=1$,
let there be two actions $a_1,a_2$. Let the observation state space be $\Scal = R_1\bigcup R_2\bigcup B_1\bigcup B_2$, where in task 1 one can only observe $R_1\bigcup R_2$ and in task 2 one can only observe $B_1\bigcup B_2$. Correspondingly, $o_1(s|z)$ is only supported on $R_1\bigcup R_2$ (i.e., $o_1(R_i|z_i) = 1$) and similar for task 2.
Let the latent state transition be such that $P(z_3|z_1,a_1)=P(z_3|z_2,a_2)=1$, and $P(z_4|z_1,a_2)=P(z_4|z_2,a_1)=1$. All rewards are $0$ for $h=1$.
\item
For $h=2$, in state $z_3$, all actions have reward $1$, and in state $z_4$ all actions have reward $0$.
\item
The initial state distribution is $d_0(z_1)=d_0(z_2)=1/2$.
\end{itemize}

Now, consider a 2-element feature class $\Psi = \{\psi_1,\psi_2\}$ for $h=1$, such that
\begin{align*}
    \psi_1 =&\; \{R_1\rightarrow 1, R_2\rightarrow 2, B_1\rightarrow 1, B_2\rightarrow 2\}\\
    \psi_2 =&\; \{R_1\rightarrow 1, R_2\rightarrow 2, B_1\rightarrow 2, B_2\rightarrow 1\}
\end{align*}
Denote $\phi_i(s,a)=e_{(\psi_i(s),a)}$ for $i\in[1,2]$. In addition, define $\Upsilon = \{\mu_1,\mu_2\}$ where
\begin{align*}
    \mu_1 =&\; \{z_3\rightarrow (1,0), z_4\rightarrow (0,1)\}\\
    \mu_2 =&\; \{z_4\rightarrow (1,0), z_3\rightarrow (0,1)\}
\end{align*}
Notice that $\phi_1$ and $\phi_2$ are merely permutations of one another and so given any single task data, the two hypothesis will not be distinguishable by any means. Therefore, for any algorithm, there is at least probability $1/2$ that it will choose the wrong hypothesis. Suppose $\phi_1$ is the correct hypothesis and $\phi_2$ is the one that the algorithm picks (i.e., $\hat\phi= \phi_2$). Let task 3 be such that any state emits to $R_1\bigcup R_2$ and $B_1\bigcup B_2$ each with probability $1/2$ (i.e., $o_3(R_i | z_i) = o_3(B_i|z_i) = 0.5$). This construction satisfies \pref{ass:reachability} and \pref{ass:linear-span}.

Then, for any policy that only make decision based on $\hat\phi(s,a)$ and $r(s,a)$, $\pi$ would output the same action for observations in $R_1$ and $B_2$, or for $B_1$ and $R_2$. However, notice that the optimal policy, which would try to go to $z_3$ from either $z_1$ or $z_2$, will pick $a_1$ at $R_1$ and $B_1$ while picking $a_2$ at $R_2$ and $B_2$, which means that the optimal policy will not agree on $R_1$ and $B_2$, and it also will not agree on $R_2$ and $B_1$. Thus clearly, no such policy as defined above is capable of capturing the optimal policy. From the reward perspective, notice that
 $d^\pi(z_1)=d^\pi(z_2)=1/2$ and $d^\pi(R_1)=d^\pi(R_2)=d^\pi(B_1)=d^\pi(B_1)=1/4$. Since $\pi(R_1) = \pi(B_2)$, the agent will only be able to collect reward at one of the $R_1$ and $B_2$ (but not at both). Similarly, since $\pi(R_2) = \pi(B_1)$, the agent will only be able to collect reward at one of the $R_2$ and $B_1$ by reaching $z_3$ (but not at both). This means that $\pi$ will have average reward $1/2$. Since the optimal policy will be able to collect reward at all $R_1, R_2, B_1, B_2$, it will have average reward $1$. This concludes the proof.
 % Then, the agent will not get reward when encountering $B_2$ and $R_2$, since at $B_2$ the agent will pick the same action it will pick at $R_1$ while at $R_2$ the agent will pick the same action it will pick at $B_1$,
 % and thus achieving an averaged reward of only $1/2$ while the optimal policy achieves an optimal reward of $1$.
\end{proof}

\pref{thm:imp1} and \pref{thm:imp2} show that it's impossible to allow online learning in the source tasks without much stronger assumptions. In our paper, we show that our \pref{ass:reachability_raw}, which ensures reachability in the raw states, is sufficient to establish an end-to-ending online transfer learning result. However, it is unclear if \pref{ass:reachability_raw} is necessary for online learning. We leave this as an important direction of future work.

\section{Reward-free Rep-UCB}\label{sec:rf_repucb}

% In the previous version of the manuscript, we used FLAMBE to explore in the source tasks. In this new version, we switch from FLAMBE to reward-free Rep-UCB -- a state-of-art algorithm for learning in low-rank MDPs. The reward-free Rep-UCB analysis is motivated by the analysis in a concurrent work \citep{cheng2022provable}. The sample complexity improvement from Rep-UCB immediately results an improvement of sample complexity in source tasks.

In this section, we adapt the Rep-UCB algorithm \citep{uehara2021representation} for reward-free exploration in a single task. We drop all task subscripts as this section is for a single task only, i.e. think about the task as being each source task.
The original Rep-UCB algorithm was for infinite-horizon discounted MDPs, so we modify it to work for our undiscounted and finite-horizon setting.
Our goal is to prove that Rep-UCB can learn a model that satisfies strong TV guarantees, i.e. \pref{thm:reward-free-model-tv} and \pref{eq:rep-ucb-tv-visitation-bound}. Note that FLAMBE \citep[Theorem 2]{Agarwal2020_flambe} can be used for this directly, but at a worse (polynomial) sample complexity. Thus, we do a bit more work to derive a new model-learning algorithm for low-rank MDPs, based on Rep-UCB, that is more sample efficient in the source tasks.

A finite-horizon analysis of Rep-UCB was done in BRIEE \citep{zhang2022efficient}, so here we just need to replace BRIEE's RepLearn $\zeta_n$ with that of the MLE, which is how we learn $\wh\phi$ and $\wh\mu$, as in Rep-UCB.
Recall the notation of \citep{zhang2022efficient},
\begin{align*}
    &\rho_{h,n}(s,a) = \frac{1}{n}\sum_{i=0}^{n-1} d^{\wh\pi_i}_h(s)\op{Unif}(a)
    \\&\beta_{h,n}(s,a) = \frac{1}{n}\sum_{i=0}^{n-1} \Eb[\wt s\sim d^{\wh\pi_i}_{h-1}, \wt a\sim \op{Unif}(\Acal)]{P_h^\star(s\mid \wt s,\wt a)\op{Unif}(a)}
    \\&\gamma_{h,n}(s,a) = \frac{1}{n}\sum_{i=0}^{n-1}d^{\wh\pi_i}_h(s,a)
    \\&\Sigma_{\rho,\phi,n} = n\Eb[\rho]{\phi(s,a)\phi(s,a)^T} + \lambda_n I.
\end{align*}
By using MLE \citep[Lemma 18]{uehara2021representation} to learn models, with probability at least $1-\delta$, for any $n=1,2,...,N$ and $h=0,1,...,H-1$, we have
\begin{align}
    &\max\braces{
    \EE_{\rho_{h,n}}\nm{ \wh P_{h,n}(s,a) - P^\star_h(s,a) }_{TV}^2,
    \EE_{\beta_{h,n}}\nm{ \wh P_{h,n}(s,a) - P^\star_h(s,a) }_{TV}^2 } \leq \zeta_n, \label{eq:rep-ucb-mle-tv}
\end{align}
where
\begin{align*}
    \zeta_n = \Ocal\prns{\frac{\log\prns{|\Mcal|nH/\delta}}{n}},
\end{align*}
and $|\Mcal| = \max_{h\in[H]} |\Phi_h||\Upsilon_h|$.
We also adopt the same choice of $\alpha_n,\lambda_n$ parameters as BRIEE, which we assume from now on.
\begin{align*}
    &\lambda_n = \Theta\prns{ d \log(|\Mcal|nH/\delta) }
    \\&\alpha_n = \Theta\prns{ \sqrt{n|\Acal|^2\zeta_n + \lambda_n d} }.
\end{align*}

As in Rep-UCB, we posit standard assumptions about realizability and normalization, on the (source) task of interest.
\begin{assum}
For any $h=0,1,...,H-1$, we have $\phi_h^\star\in\Phi_h$ and $\mu_h^\star\in\Upsilon_h$.
For any $\phi\in\Phi_h$, $\|\phi(s,a)\|_2 \leq 1$. For all $\mu\in\Upsilon_h$ and any function $g:\Scal\to\RR$, we have $\|\int_s g(s)\diff\mu(s)\|_2 \leq \|g\|_\infty \sqrt{d}$.
\end{assum}

\begin{lemma}
Let $r$ be any reward function.
Suppose we ran \pref{alg:reward_free_repucb} with \pref{line:rep_ucb_planning_step} having reward $r+\wh b_n$ instead of just $\wh b_n$.
Then, for any $\delta\in(0,1)$, w.p. at least $1-\delta$, we have
\begin{align*}
    \sum_{n=0}^{N-1} V^{\wh\pi_n}_{\wh P_n, r+\wh b_n} - V^{\wh\pi_n}_{P^\star, r} \leq \Ocal\prns{ H^2d^2|\Acal|^{1.5} \sqrt{N \log(|\Mcal|NH/\delta)} }
\end{align*}
\end{lemma}
\begin{proof}
Start from the third equation of \citet[Theorem A.4]{zhang2022efficient}. Following their proof until the last page of their proof, we arrive at the following: for any $n=1,2,...,N$,
\begin{align*}
    &V^{\wh\pi_n}_{\wh P_n, r+\wh b_n} - V^{\wh\pi_n}_{P^\star, r}
    \\&\lesssim \sum_{h=0}^{H-2}\EE_{\wt s,\wt a\sim d^{\wh\pi_n}_{P^\star,h}}\nm{\phi^\star_h(\wt s,\wt a)}_{\Sigma_{\gamma_{h,n},\phi^\star_h}^{-1}} \sqrt{ |\Acal|\alpha_n^2 d + \lambda_n d } + \sqrt{ |\Acal|\alpha_1^2 d/n }
    \\&+ (2H+1)\sum_{h=0}^{H-2} \EE_{\wt s,\wt a\sim d^{\wh\pi_n}_{P^\star,h}} \nm{\phi^\star_h(\wt s,\wt a)}_{\Sigma_{\gamma_{h,n},\phi^\star_h}^{-1}}\sqrt{n|\Acal|\zeta_n + \lambda_n d} + (2H+1)\sqrt{|\Acal|\zeta_n}.
\end{align*}
By elliptical potential arguments, we have
\begin{align*}
    \sum_{n=0}^{N-1}\EE_{\wt s,\wt a\sim d^{\wh\pi_n}_{P^\star,h}}\nm{\phi^\star_h(\wt s,\wt a)}_{\Sigma_{\gamma_{h,n},\phi^\star_h}^{-1}} \leq \sqrt{dN\log\prns{1+\frac{N}{d\lambda_1}}}.
\end{align*}
Thus, summing over $n$, noting that $n\zeta_n,\alpha_n,\lambda_n$ are increasing in $n$, we can combine the above to get,
\begin{align*}
    &\sum_{n=0}^{N-1}V^{\wh\pi_n}_{\wh P_n, r+\wh b_n} - V^{\wh\pi_n}_{P^\star, r}
    \\&\lesssim \sqrt{dN\log\prns{1+\frac{N}{d\lambda_1}}} \prns{ H \sqrt{|\Acal|\alpha_N^2 d + \lambda_N d} + H^2 \sqrt{ N|\Acal|\zeta_N + \lambda_N d } }
    \\&\lesssim \sqrt{dN\log\prns{1+\frac{N}{d\lambda_1}}} \prns{ H \sqrt{N|\Acal|^3\zeta_N d + \lambda_N d^2}  + H^2\sqrt{ N|\Acal|\zeta_N + \lambda_N d } }
    \\&\lesssim \sqrt{dN\log\prns{1+\frac{N}{d\lambda_1}}} \prns{ H^2 \sqrt{d |\Acal|^3\log(|\Mcal|NH/\delta) + d^3 \log(|\Mcal|NH/\delta) } }
    \\&\in \Ocal\prns{ H^2d^2|\Acal|^{1.5} \sqrt{N \log(|\Mcal|NH/\delta)} }.
\end{align*}
\end{proof}

This gives the following useful corollary for reward free exploration.
\begin{lemma}\label{lem:reward-free-value-upper-bound}
For any $\delta\in(0,1)$ w.p. at least $1-\delta$ we have
\begin{align*}
    \wh V_{\wh n} \leq \Ocal\prns{ H^2d^2|\Acal|^{1.5} \sqrt{\frac{\log(|\Mcal|NH/\delta)}{N}} }.
\end{align*}
\end{lemma}
\begin{proof}
    By definition of $\wh V_{\wh n}$, we have
    \begin{align*}
        \frac{N}{2}\wh V_{\wh n} \leq \sum_{n=N/2}^{N-1} V^{\wh\pi_{\wh n}}_{\wh P_n, \wh b_n} \leq \sum_{n=0}^{N-1} V^{\wh\pi_{\wh n}}_{\wh P_n, \wh b_n},
    \end{align*}
    which is bounded by the previous lemma and the fact that $V^{\wh\pi_n}_{P^\star, r=0} = 0$, since in \pref{alg:reward_free_repucb}, the reward function is zero.
\end{proof}

Conditioning on this, we now show that the environment $\wh P_{\wh n}$ has low TV error for any policy-induced distribution.
\begin{theorem}\label{thm:reward-free-model-tv}
For any policy $\pi$, we have
\begin{align*}
    \sum_{h=0}^{H-1} \EE_{d^\pi_{P^\star,h}}\nm{ P^\star_h(s,a)-\wh P_{h, \wh n}(s,a) }_{TV} \leq \Ocal\prns{ H^3d^2|\Acal|^{1.5} \sqrt{\frac{\log(|\Mcal|NH/\delta)}{N}} } := \eps_{TV}.
\end{align*}
\end{theorem}
\begin{proof}
% The proof is from \citep[Theorem C.4]{zhang2022efficient} and only differs in minor ways, i.e. different setting of $\zeta_n$ and that we set a floor for $\wh n$. However, we replicate it here for completeness.
In this proof, let $\wh P = \wh P_{\wh n}$, which is the returned environment from the algorithm.
Let $r(s,a) = \nm{ P^\star_h(s,a)-\wh P_h(s,a) }_{TV}\in [0,2]$. Then,
\begin{align*}
    &\sum_{h=0}^{H-1} \prns{ \EE_{d^\pi_{P^\star,h}} - \EE_{d^\pi_{\wh P, h}} } [r(s,a)]
    \\&= V_{P^\star, r}^\pi - V_{\wh P, r}^\pi
    \\&= \sum_{h=0}^{H-1} \Eb[d^\pi_{\wh P, h}]{ \prns{\EE_{P_h^\star(s,a)} - \EE_{\wh P_{h}(s,a)} }V^\pi_{P^\star, r, h+1}(s') } \tag{Simulation lemma}
    \\&\leq 2H \sum_{h=0}^{H-1} \EE_{d^\pi_{\wh P, h}} \nm{P^\star_h(s,a) - \wh P_{h}(s,a)}_{TV}.
\end{align*}
Thus,
\begin{align*}
    &\sum_{h=0}^{H-1} \EE_{d^\pi_{P^\star,h}}\nm{ P^\star_h(s,a)-\wh P_{h}(s,a) }_{TV}
    \\&\leq (2H+1) \sum_{h=0}^{H-1} \EE_{d^\pi_{\wh P, h}} \nm{P^\star_h(s,a) - \wh P_{h}(s,a)}_{TV}
    \tag{by \citep[Lemma A.1]{zhang2022efficient}}\\
    &\lesssim H \prns{ \sum_{h=0}^{H-2} \Eb[d^\pi_{\wh P,h}]{\wh b_{h,\wh n}(s,a)} + \sqrt{|\Acal|\zeta_{N/2}}}
    \\&\leq H \prns{ V^{\pi}_{\wh P, \wh b_{\wh n}} + \sqrt{2|\Acal|\frac{\log(|\Mcal|NH/\delta)}{N}}}
    \\&\leq H \prns{ V^{\wh\pi_{\wh n}}_{\wh P, \wh b_{\wh n}} + \sqrt{2|\Acal|\frac{\log(|\Mcal|NH/\delta)}{N}}}
    \\&\lesssim H \prns{ H^2d^2|\Acal|^{1.5} \sqrt{\frac{\log(|\Mcal|NH/\delta)}{N}} + \sqrt{|\Acal|\frac{\log(|\Mcal|NH/\delta)}{N}}} \tag{by \pref{lem:reward-free-value-upper-bound}}
    \\&\in \Ocal\prns{ H^3d^2|\Acal|^{1.5} \sqrt{\frac{\log(|\Mcal|NH/\delta)}{N}} }.
\end{align*}
\end{proof}

This also gives us a guarantee on the TV distance between the visitation distributions induced by $P^\star$ vs. by $\wh P$.
\begin{lemma}
Suppose $\wh P$ satisfies the following for all $h=0,1,...,H-1$,
\begin{align}
    \forall\pi: \EE_{d^\pi_{P^\star,h}}\nm{ \wh P_h(s,a) - P_h^\star(s,a) }_{TV} \leq \eps_h. \label{eq:tv-guarantee-assum}
\end{align}
Then, for any $h=0,1,...,H-1$, we have
\begin{align*}
    \forall\pi: \nm{ d^\pi_{\wh P,h} - d^\pi_{P^\star,h} }_{TV} \leq \sum_{t=0}^{h-1} \eps_t.
\end{align*}
Note, for $h=0$, the sum is empty so the right hand side is $0$.
\end{lemma}
\begin{proof}
We proceed by induction for $h=0,1,...,H-1$.
For the base case of $h=0$, no transition has been taken, so that $d^\pi_{\wh P,0} = d^\pi_{P^\star, 0}$.
Now let $h\in\{0,1,...,H-2\}$ be arbitrary, and suppose that the claim is true for $h$ (IH).
We want to show the claim holds for $h+1$.
One key fact we'll use is that, for any measure $\mu$, we have $\|\mu\|_{TV} = \sup_{\|f\|_\infty\leq 1}|\EE_\mu[f]|$.
Below we use the notation that $f(s,\pi) = \EE_{a\sim \pi(s)} f(s,a)$.
\begin{align*}
    &\;\|d^\pi_{\wh P,h+1} - d^\pi_{h+1}\|_{TV}
    \\
    =&\; \sup_{\|f\|_\infty \leq 1} \abs{ \Eb[d^\pi_{\wh P,h+1}]{f(s,a)} - \Eb[d^\pi_{h+1}]{f(s,a)} }
    \\=&\; \sup_{\|f\|_\infty \leq 1} \abs{ \Eb[(\wt s,\wt a)\sim d^\pi_{\wh P,h},(s,a)\sim\wh P_h(\wt s,\wt a)]{f(s,\pi_{h+1})} - \Eb[(\wt s,\wt a)\sim d^\pi_{h},(s,a)\sim P_h^\star(\wt s,\wt a)]{f(s,\pi_{h+1})} }
    \\
    \leq&\; \sup_{\|f\|_\infty \leq 1} \abs{ \prns{ \EE_{(\wt s,\wt a)\sim d^\pi_{\wh P,h}} - \EE_{(\wt s,\wt a)\sim d^\pi_{h}} } \EE_{\wh P_h(\wt s,\wt a)} f(s,\pi_{h+1}) }
    \\&\;+ \sup_{\|f\|_\infty \leq 1} \abs{ \Eb[(\wt s,\wt a)\sim d^\pi_{h}]{ \EE_{\wh P_h(\wt s,\wt a)} f(s,a) - \EE_{P_h^{\star}(\wt s,\wt a)} f(s,\pi_{h+1})} }
    \\
    \leq &\; \sum_{t=0}^{h-1} \eps_t + \Eb[(\wt s,\wt a)\sim d^\pi_{h}]{  \sup_{\|f\|_\infty \leq 1} \abs{  \prns{ \EE_{\wh P_h(\wt s,\wt a)} - \EE_{P_h^{\star}(\wt s,\wt a)} } f(s,\pi_{h+1})} } \tag{by IH and Jensen}
    % \\
    % \leq &\; \sum_{t=0}^{h-1} \eps_t +  \Eb[(\wt s,\wt a)\sim d^\pi_{h}]{ \sup_{\|f\|_\infty \leq 1} \left|\wh P^\pi_h f(\wt s,\wt a) - P_h^{\star\pi} f(\wt s,\wt a)\right|}\tag{by Jensen's inequality}
    \\
    \leq &\; \sum_{t=0}^{h-1} \eps_t + \eps_h, \tag{by \pref{eq:tv-guarantee-assum} and $\|f(s,\pi_{h+1})\|\leq 1$}
\end{align*}
as desired.
\end{proof}
Thus, when combined with \pref{thm:reward-free-model-tv}, we have for $h=0,1,...,H-1$ and any policy $\pi$,
\begin{align}
    \|d^\pi_{\wh P,h} - d^\pi_{P^\star, h}\|_{TV} \leq \Ocal\prns{ H^3d^2|\Acal|^{1.5} \sqrt{\frac{\log(|\Mcal|NH/\delta)}{N}} } = \eps_{TV}.  \label{eq:rep-ucb-tv-visitation-bound}
\end{align}
In other words, the sample complexity needed for a model-error of $\eps_{TV}$ is
\begin{align*}
    \Ocal\prns{ \frac{H^6d^4|\Acal|^3\log(|\Mcal|NH/\delta)}{\eps_{TV}^2} }.
\end{align*}
Note this is much better than FLAMBE's guarantee \citep[Theorem 2]{Agarwal2020_flambe} which requires,
\begin{align*}
    \Ocal\prns{ \frac{H^{22}d^7|\Acal|^9\log(|\Mcal|NH/\delta)}{\eps_{TV}^{10}} }.
\end{align*}

\section{Reward-free Exploration}
In this section, we show that the mixture policy returned by \pref{alg:reward_free} is exploratory.
% Similar to the previous section, this section is for a fixed low-rank MDP, and we will apply this result to each of the source tasks.
% Typically, reward-free RL refers to a setting where the agent doesn't have access to a reward function during exploration, but using the samples collected, the agent needs to generate a near-optimal policy for any arbitrary reward function revealed after exploration \citep{wagenmaker2022reward,wang2020reward}. We use the algorithm from a slightly different angle: we want to generate a (mixture) policy that has coverage, in the sense of lower bounded eigenvalues. It turns out that the same reward-free algorithm can generate a policy with coverage. \citet[Theorem 4]{wagenmaker2022reward} showed this for linear MDPs with known $\phi^\star$. In this section, we extend the result to the more general setting of unknown $\phi^\star$ in low-rank MDPS.
Recall that \pref{alg:reward_free} contains two main steps:
\begin{enumerate}
    \item[Step 1] Learn a model $\wh P$. This was the focus of the previous section, where our modified \repucb method obtained a strong TV guarantee (\pref{eq:rep-ucb-tv-visitation-bound}) by requiring number of episodes at most,
    \begin{align*}
        N_{\textsc{RewardFree}} = \Ocal\prns{ \frac{H^6d^4|\Acal|^3\log(|\Mcal|NH/\delta)}{\eps_{TV}^2} }.
    \end{align*}
    \item[Step 2] Run \LSVIUCB{} (\pref{alg:LSVI}) in the \emph{learned} model $\wh P$ with reward at the $e$-th episode being $b_{h,e}$ and \textsc{UniformActions} = \textsc{True}.
    The optimistic bonus pushes the algorithm to explore directions that are not well-covered yet by the mixture policy up to this point. With elliptical potential, we can establish that this process will terminate in polynomial number of steps.
\end{enumerate}

We now focus on Step 2. Let $\pi_h^{+1}$ denote rolling-in $\pi$ for $h$ steps and taking uniform actions on the $h+1$ step, thus inducing a distribution over $s_{h+1},a_{h+1}$. We abuse the notation a little and use $\pi_{-1}^{+1}$ for a policy that just takes one uniform action from the initial distribution $d_0$.

% Intuitively, reward-free exploration explores the directions not yet covered by the existing mixture policy. It does so by optimizing the bonus term, which is small in directions with large eigenvalues of the covariance (i.e. directions already explored), and large in directions with small eigenvalues of the covariance (i.e. directions yet to be explored). We now recall a key result of reward-free RL: the bonuses will eventually be small, i.e. we'll eventually have coverage over the whole space.
\begin{lemma}\label{lem:reward-free-exploration-pi-star}
Let $\delta\in(0,1)$ and run \textsc{RewardFree} (\pref{alg:reward_free}).
Let $\Lambda_{h,N}$ be the empirical covariance at the $N$-th iteration of \LSVIUCB{} (\pref{alg:LSVI}).
Then, w.p. at least $1-\delta$ we have,
\begin{align*}
\sup_\pi \sum_{h=0}^{H-1} \EE_{\pi_{h-1}^{+1},\wh P}\nm{ \wh\phi_h(s_h,a_h) }_{\Lambda_{h,N}^{-1}} \lesssim A d^{1.5}H^3\sqrt{\log(dNH/\delta)/N}.
\end{align*}
\end{lemma}
\begin{proof}
In this proof, we'll treat the empirical MDP as $P^\star$, as that is the environment we're running in. Thus, we abuse notation and $\wh P_{h,e}$ is the model-based perpsective of the linear MDP, i.e. $\wh\phi_h\wh\mu_{h,e}$ where $\wh\mu_{h,e}$ is $\Lambda_{h,e}^{-1}\sum_{k=1}^{e-1}\wh\phi_h(s_h^k,a_h^k)\delta(s_{h+1}^k)$.
Also, in \pref{alg:reward_free}, we set reward to be zero, but for the purpose of this analysis, suppose the reward function is precisely the (unscaled) bonus in \LSVIUCB{}, i.e. $r_{h,e}(s_h,a_h) = b_{h,e}(s_h,a_h)$. This does not change the algorithm at all since the $\beta$-scaling of the bonus dominates this reward in the definition of $\wh Q_{h,e}$, but thinking about the reward in this way will make our analysis simpler.

Recall the high-level proof structure of reward free guarantee of linear MDP (with known features $\wh\phi$) \citep[Lemma 3.2]{wang2020reward}.
% , which is written down in detail in Chapter 8 of \citep{rltheorybookAJKS}, or \citep{jin2020provably}.
\begin{enumerate}
    \item[Step 1] Show that $\wh V_{h,e}\in\Vcal_h$ and w.p. $1-\delta$, for all $h,e$,
    \begin{align*}
        \forall s_h,a_h: \sup_{f\in\Vcal_h} \abs{ \wh P_{h,e}(s_h,a_h) - P^\star_h(s_h,a_h) } f \leq \beta b_{h,e}(s_h,a_h).
    \end{align*}
    This step only uses self-normalized martingale bounds. So, \pref{line:unif-actions-lsvi-collect-data} can use any martingale sequence of states and actions, and this claim still holds, with bonus $b_{h,e}$ using the appropriate covariance under the data.
    \item[Step 2] Show optimism conditioned on Step 1. Specifically, for all $e=1,2,...,N$, we have $\Eb[d_0]{V^\star_0(s_0, r_{e}) - \wh V_{0,e}(s_0)} \leq 0$. To show this, we need that $\wh V_{h,e}(s_h) = \wh Q_{h,e}(s_h,\pi^e_h(s_h)) \geq \wh Q_{h,e}(s_h,\pi^\star_h(s_h))$ (this is for the unclipped case of $V$-optimism), which we have satisfied in the algorithm, i.e. $\pi^e_h$ is greedy w.r.t. $\wh Q_{h,e}$.
    % \item[Step 3] Use a simulation lemma, along with the results of the prior two steps to argue that regret is bounded by the sum of expected bonuses. Then, apply Azuma-Hoeffding and elliptical potential to conclude the bound. Here, the application of elliptical potential does require data to be sampled from $\pi^e$, which is where we ought to change to uniform actions via importance sampling.
    \item[Step 3] Bound the sum $\sum_e \wh V_{h,e}$, where we decompose it as a sum of expected bonuses with the expectation is under $\pi^e$.
\end{enumerate}
Step 3 is the only place where we use the fact that $s_h^k,a_h^k$ are data sampled from rolling out $\wh\pi^e$. For Step 1 and 2, please refer to existing proofs in \citep{rltheorybookAJKS,jin2020provably,wang2020reward}.

Now we show Step 3 for our modified algorithm with uniform actions.
First, let us show a simulation lemma. For any episode $e=1,2,...,N$, for any $s$, recalling definition of reward being $b_{h,e}$, we have
\begin{align*}
    \wh V_{0,e}(s_0)
    &\leq (1+\beta) b_{0,e}(s_0,\pi^e_0(s_0)) + \wh P_{0,e}(s_0,\pi^e(s_0)) \wh V_{1,e}
    \\&\leq (1+2\beta) b_{0,e}(s_0,\pi^e_0(s_0)) + P_{0,e}^\star(s_0,\pi^e(s_0)) \wh V_{1,e},
\end{align*}
where the first inequality is due to the thresholding on $\wh V_{h,e}$'s and the second inequality is due to Step 1.
Continuing in this fashion, we have
\begin{align*}
    \Eb[d_0]{\wh V_{0,e}(s_0)} \leq (1+2\beta) \sum_{h=0}^{H-1}\Eb[\pi^e]{ b_{h,e}(s_h,a_h) }.
\end{align*}
Summing over $e=1,2,...,N$, we have
\begin{align*}
    \sum_{e=1}^N \Eb[d_0]{ \wh V_{0,e}(s_0) }
    &\lesssim \beta \sum_{h=0}^{H-1} \sum_{e=1}^N \Eb[\pi^e]{ b_{h,e}(s_h,a_h) }
    \\&\leq A\beta \sum_{h=0}^{H-1} \sum_{e=1}^N \Eb[\prns{\pi^e_{h-1}}^{+1}]{ b_{h,e}(s_h,a_h) }
    \intertext{For each $h=0,1,...,H-1$, apply Azuma's inequality to the martingale difference sequence $\Delta_e = \Eb[\prns{\pi^e_{h-1}}^{+1}]{ b_{h,e}(s_h,a_h) } - b_{h,e}(s_h^e,a_h^e)$.
    The envelope is at most $2$. So, w.p. $1-\delta$, }
    &\leq A\beta \sum_{h=0}^{H-1} \sum_{e=1}^N b_{h,e}(s_h^e,a_h^e) + A\beta \sqrt{N\log(H/\delta)}.
\end{align*}
Now apply a self-normalized elliptical potential bound to the first term, giving that
\begin{align*}
    \sum_{h=0}^{H-1} \sum_{e=1}^N b_{h,e}(s_h^e,a_h^e) \leq \sum_{h=0}^{H-1} \sqrt{N} \sqrt{\sum_{e=1}^N b_{h,e}(s_h^e,a_h^e)^2} \lesssim H\sqrt{ dN\log(N) }.
\end{align*}
Thus, we finally have
\begin{align*}
    \sum_{e=1}^N \Eb[d_0]{ \wh V_{0,e}(s_0) }  \lesssim A\beta H \sqrt{ dN\log(NH/\delta) }.
\end{align*}

Consider any episode $e=1,2,...,N$.
By definition, $\Lambda_{h,N} \succeq \Lambda_{h,e}$, so for all $s,a$ we have pointwise that $b_{h,N}(s,a) \leq b_{h,e}(s,a)$. Hence, for all $s$, we have $V_0^\star(s;r^N) \leq V_0^\star(s;r^e)$, and further using optimism, we have
\begin{align*}
    N\Eb[d_0]{V_0^\star(s_0;r_N)} \leq \sum_{e=1}^N \Eb[d_0]{V_0^\star(s_0;r_e)} \leq \sum_{e=1}^N \Eb[d_0]{ \wh V_{0,e}(s_0) } \lesssim A\beta H \sqrt{ dN\log(NH/\delta) }.
\end{align*}
Now consider any $h$ and policy $\pi$, and consider rolling it out for $h-1$ steps and taking a random action. Then we have
\begin{align*}
    \EE_{\pi_{h-1}^{+1}, \wh P}\nm{\wh\phi_h(s_h,a_h)}_{\Lambda_{h,N}^{-1}} \leq \Eb[d_0]{V_0^{\pi_{h-1}^{+1}}(s_0; r_N)} \leq A\beta H\sqrt{d\log(NH/\delta)/N}.
\end{align*}
Summing over $h$ incurs an extra $H$ factor on the right.
This concludes the proof.
\end{proof}

\begin{lemma}[One-step back for Linear MDP]\label{lem:one-step-back-linear-mdp}
Suppose $P_h = (\phi_h, \mu_h)$ is a linear MDP. Suppose $\rho$ is any mixture of $n$ policies, and let $\Sigma_h := n \Eb[\rho]{\phi_h(s_h,a_h)\phi_h(s_h,a_h)^\top} + \lambda I$ denote the unnormalized covariance. For any $g: S \times A \to \RR$, policy $\pi$, and $h=0,1,...,H-2$, we have
\begin{align*}
    \Eb[\pi]{g(s_{h+1},a_{h+1})} \leq&\; \Eb[\pi]{ \|\phi_h(s_h,a_h)\|_{\Sigma_h^{-1}} } \sqrt{ n A \Eb[\rho_h^{+1}]{ g(s_{h+1},a_{h+1})^2 } + \lambda d \|g\|_\infty^2 }
\end{align*}
\end{lemma}
\begin{proof}
%  i.e. it satisfies $\rho_{h+1}(s,a) = \Eb[\wt s,\wt a\sim\rho_h]{ P_h(s\mid\wt s,\wt a)\pi(a\mid s) }$
\begin{align*}
    \Eb[\pi]{g(s_{h+1},a_{h+1})}
    =&\; \left\langle \Eb[\pi]{ \phi_h(s_h,a_h) }, \int_{s_{h+1}} g(s_{h+1},\pi_{h+1}) \diff\mu_h(s_{h+1}) \right\rangle
    \\\leq&\; \EE_{\pi} \nm{\phi_h(s_h,a_h)}_{\Sigma_h^{-1}}  \nm{ \int_{s_{h+1}} g(s_{h+1},\pi_{h+1}) \diff\mu_h(s_{h+1})  }_{\Sigma_h},
\end{align*}
where
\begin{align*}
    &\; \nm{ \int_{s_{h+1}} g(s_{h+1},\pi_{h+1}) \diff\mu_h(s_{h+1})  }_{\Sigma_h}
    \\
    =&\; n \Eb[\rho]{ \prns{ \Eb[s_{h+1} \sim P_h(s_h,a_h)]{g(s_{h+1} ,\pi_{h+1})} }^2 } + \lambda\nm{\int_{s_{h+1}} g(s_{h+1},\pi_{h+1})\diff\mu_h(s_{h+1})}^2
    \\
    \leq&\; n |\Acal| \Eb[\rho_h^{+1}]{ g(s_{h+1},a_{h+1})^2 } + \lambda d \|g\|_\infty^2.
\end{align*}
\end{proof}

Under reachability, we can show that small (squared) bonuses and spectral coverage, in the sense of having lower bounded eigenvalues, are somewhat equivalent.
%We note that we actually only need a weaker, directional version of reachability: for any direction $v$ with $\|v\|=1$, there exists a policy $\pi$ such that $\Eb[d^\pi_h]{(\phi_h(s,a)^T v)^2} \geq \psi$. This is implied by \pref{ass:reachability} since $\Eb[d^\pi_h]{(\phi_h(s,a)^T v)^2} = v^T \Eb[d^\pi_h]{\phi_h(s,a)\phi_h(s,a)^T} v \geq \lambda_{min}(\Eb[d^\pi_h]{\phi_h(s,a)\phi_h(s,a)^T}) \geq \psi$.
\begin{lemma}\label{lem:bonus-coverage-equivalence}
Let $\Sigma$ be a symmetric positive definite matrix and define the bonus $b_h(s,a) = \|\phi^\star_h(s,a)\|_{\Sigma^{-1}}$. Then we have
\begin{enumerate}
    \item For any policy $\pi$, $\Eb[d^{\pi}_{h}]{ b_h^2(s,a) } \leq \frac{1}{\lambda_{\min}(\Sigma)}$. That is, coverage implies small squared bonus.
    \item Suppose reachability under $\phi^\star$ (\pref{ass:reachability}), then we have the converse: there exists $\wh\pi$, for any policy $\pi$, $\Eb[d^{\wt\pi}_{h}]{ b_{h}^2(s,a) } \geq \frac{\psi}{\lambda_{\min}(\Sigma)}$. That is, small squared bonus implies coverage.
\end{enumerate}
% For this lemma, define the covariance $\Sigma_{\pi,h} := \Eb[d^\pi_h]{\phi_h\phi_h^T}$ and the bonus $b_{h,\pi}(s,a) = \phi_h(s,a)^T(\Sigma_{\pi,h}+\lambda I)^{-1}\phi_h(s,a)$ for $\lambda > 0$.
% First, we have that $\Eb[d^{\pi}_{h}]{ b_h(s,a) } \leq \frac{1}{\lambda_{min}(\Sigma_{\pi,h})+\lambda}$.
% Now suppose that we have reachability , i.e. there exists a policy $\wt\pi$ s.t. $\lambda_{min}(\Sigma_{\wt\pi,h}) \geq \psi$.
% Then, we have the converse: for any policy $\pi$, we have $\Eb[d^{\wt\pi}_{h}]{ b_{h,\pi}(s,a) } \geq \frac{\psi}{\lambda_{min}(\Sigma_{\pi,h})+\lambda}$.
\end{lemma}
\begin{proof}
The first claim follows directly from Cauchy-Schwartz. Indeed, for any policy $\pi$, we have
\begin{align*}
    \Eb[d^{\pi}_{h}]{ b_h^2(s,a) } \leq \Eb[d^{\pi}_{h}]{ \|\phi^\star_h(s,a)\|_2^2 \|\Sigma^{-1}\|_2 } \leq \frac{1}{\lambda_{\min}(\Sigma)}.
\end{align*}

% The first claim follows directly from $\|(\Sigma_{\pi,h}+\lambda I)^{-1}\| = (\lambda_{min}(\Sigma_{\pi,h})+\lambda)^{-1}$ and the fact that $\|\phi_h(s,a)\| \leq 1$,
% \begin{align*}
%     \Eb[d^{\pi}_{h}]{ b_h(s,a) } \leq \Eb[d^{\pi}_{h}]{ \|\phi_h(s,a)\|^2 \|\Sigma^{-1}\| } \leq \frac{1}{\lambda_{min}(\Sigma_{\pi,h})+\lambda}.
% \end{align*}

For the second claim, \pref{ass:reachability} implies that there exist a policy $\wt\pi$ such that for all vectors $v\in\RR^d$ with $\|v\|_2=1$, we have $\Eb[d^{\wt\pi}_h]{(\phi^\star_h(s,a)^\top v)^2} \geq \psi$.
Now decompose $\Sigma = \sum_{i=1}^d \lambda_i v_iv_i^\top$, where $\lambda_i, v_i$ are eigenvalue/vector pairs with $\|v_i\|_2=1$ and $\lambda_1 \geq \lambda_2 \geq ... \geq \lambda_d$. Then substituting this into the definition of the bonus, we have
\begin{align*}
    \Eb[d^{\wt\pi}_h]{b_h^2(s,a)}
    =&\; \sum_{i=1}^d \frac{1}{\lambda_i} \Eb[d^{\wt\pi}_{h}]{(\phi^\star_h(s,a)^Tv_i)^2}
    \\\geq&\; \frac{1}{\lambda_d} \Eb[d^{\wt\pi}_{h}]{(\phi^\star_h(s,a)^Tv_d)^2}
    \\\geq&\; \frac{\psi}{\lambda_{\min}(\Sigma)}.
\end{align*}
\end{proof}

We now prove our main lemma for reward-free exploration, \pref{lem:coverage}.
\rewardFreeExploreCoverage*
\begin{proof}[Proof of \pref{lem:coverage}]
In this proof, let
\begin{align*}
    &\Lambda_h^\star = N_{\LSVIUCB} \Eb[\rho_{h-1}^{+1}]{ \phi_h^\star(s_h,a_h)\phi_h^\star(s_h,a_h)^\top } + \lambda I,
    \\&\wh\Lambda_h = N_{\LSVIUCB} \Eb[\rho_{h-1}^{+1}]{ \wh\phi_h(s_h,a_h)\wh\phi_h(s_h,a_h)^\top } + \lambda I,
\end{align*}
where $\lambda = dH\log(N_{\LSVIUCB}/\delta) \geq 1$.
This setting of $\lambda$ satisfies the precondition for the Concentration of Inverse Covariances \citet[Lemma 39]{zanette2021cautiously}, which implies w.p. at least $1-\delta$ that
\begin{align*}
    \wh\Lambda_h^{-1} \preceq 2 \prns{\sum_{e=1}^{N_{\LSVIUCB}} \wh\phi_h(s_h^e,a_h^e)\wh\phi_h(s_h^e,a_h^e)^\top + \lambda I}^{-1} \preceq 2 \Lambda_{h,N_{\LSVIUCB}}^{-1},
\end{align*}
where we've also used the fact that $\lambda \geq 1$, so $(A+\lambda I)^{-1} \preceq (A+I)^{-1}$.

Under this event, for any $\pi$, we have,
\begin{align}
    \sum_{h=1}^H \EE_{\pi,\wh P} \nm{ \wh\phi_h(s_h,a_h) }_{\wh\Lambda_h^{-1}}
    \lesssim\; \sum_{h=1}^H \EE_{\pi,\wh P} \nm{\wh\phi_h(s_h,a_h)}_{\wh\Lambda_{h,N_{\LSVIUCB}}^{-1}}. \label{eq:coverage-proof-covariance-concentrate}
\end{align}
% We will soon apply this to $\wt\pi$, the policy which maximizes the bonuses.

Now let $h=0,1,...,H-2$ be arbitrary.
% Take a one-step back as in \pref{corr:bound-true-bonus-by-learned-bonus}.
By \pref{ass:reachability} (there exists some policy $\wt\pi$ with coverage) such that,
\begin{align*}
    &\;\frac{\psi}{\lambda_{\min}\prns{\Lambda^\star_{h+1}}} \\
    \leq&\; \EE_{\wt\pi} \nm{\phi_{h+1}^\star(s_{h+1},a_{h+1})}_{\prns{\Lambda^\star_{h+1}}^{-1}}^2 \tag{by \pref{lem:bonus-coverage-equivalence}}
    \\
    \leq&\; \EE_{\wt\pi} \nm{\phi_{h+1}^\star(s_{h+1},a_{h+1})}_{\prns{\Lambda^\star_{h+1}}^{-1}} \tag{by $\lambda \geq 1$}
    \\
    \leq&\; \EE_{\wt\pi,\wh P} \nm{\wh\phi_h(s_h,a_h)}_{\wh\Lambda_h^{-1}} \sqrt{ A(2d + \eps_{TV} N_{\LSVIUCB}) } + \eps_{TV}\tag{by \pref{corr:bound-true-bonus-by-learned-bonus}}
    \\
    \leq&\; \EE_{\wt\pi,\wh P} \nm{\wh\phi_h(s,a)}_{\wh\Lambda_h^{-1}} \sqrt{A(2d + 1)} + 1/N_{\LSVIUCB} \tag{by $\eps_{TV} = 1/N_{\LSVIUCB}$}
    \\
    \lesssim&\; A^{1.5} d^2H^3 \sqrt{\log(dHN_{\LSVIUCB}/\delta)/N_{\LSVIUCB}} + 1/N_{\LSVIUCB} \tag{by \pref{eq:coverage-proof-covariance-concentrate} and \pref{lem:reward-free-exploration-pi-star}}
    \\
    \lesssim&\; A^{1.5} d^2H^3 \sqrt{\log(dHN_{\LSVIUCB}/\delta)/N_{\LSVIUCB}}.
\end{align*}
Recall that $\lambda = dH\log(N_{\LSVIUCB}/\delta)$, we have,
\begin{align*}
    &\;\lambda_{\min}\prns{ \Eb[\rho_{h}^{+1}]{\phi_{h+1}^\star(s,a)\phi_{h+1}^\star(s,a)^T} }
    \\
    =&\; \frac{\lambda_{\min}\prns{\Lambda^\star_{h+1}}-\lambda}{N_{\LSVIUCB}}
    \\
    \geq&\; \frac{1}{N_{\LSVIUCB}} \prns{ \frac{C \psi }{ A^{1.5} d^2H^3 \sqrt{\log(dHN_{\LSVIUCB}/\delta)/N_{\LSVIUCB}} } - dH\log(N_{\LSVIUCB}/\delta) }
    \\\gtrsim&\; \frac{C \psi}{ A^{1.5}d^2H^3\sqrt{N_{\LSVIUCB}} } - \frac{dH}{N_{\LSVIUCB}},
\end{align*}
where we've omitted the $\log$ terms for simplicity in the $\gtrsim$.
Now we optimize $N_{\LSVIUCB}$ to maximize this bound.
For $a,b > 0$, to maximize a function of the form $f(x) = \frac{a}{\sqrt{x}} - \frac{b}{x}$, it's best to set $x^\star$ such that $\sqrt{x^\star} = \frac{2b}{a}$, resulting in value $f(x^\star) = \frac{a^2}{4b}$. Setting,
\begin{align*}
    &x=N_{\LSVIUCB},
    \\&a = \frac{C \psi}{A^{1.5}d^2H^3},
    \\&b=dH.
\end{align*}
Hence, we need to set
\begin{align*}
    N_{\LSVIUCB} = \wt\Theta\prns{b^2/a^2} =  \wt\Theta\prns{ \frac{A^3d^6H^8}{\psi^2} },
\end{align*}
which results in a $\lambda_{\min}$ lower bound of
\begin{align*}
    \lambda_{\min}\prns{ \Eb[\rho_{h}^{+1}]{\phi_{h+1}^\star(s_{h+1},a_{h+1})\phi_{h+1}^\star(s_{h+1},a_{h+1})^T} } = \wt\Omega\prns{ a^2/b } = \wt\Omega\prns{ \frac{\psi^2}{A^3d^5H^7} }.
\end{align*}
Finally, we used the fact that $\eps_{TV} = 1/N_{\LSVIUCB}$, which is set by the choice of $N_{\textsc{RewardFree}}$ in the lemma statement to satisfy \pref{eq:rep-ucb-tv-visitation-bound}.

The above proves coverage of $\rho_h^{+1}$ for $h=0,1,...,H-2$.
Finally to argue for $\rho_{-1}^{+1}$, which is simply taking a random action at time $h$, we can simply invoke \pref{ass:reachability} for $h=0$ to get a policy $\wt\pi$ that
\begin{align*}
    \Eb[\rho_{-1}^{+1}]{\phi^\star_0(s_0,a_0)\phi^\star_0(s_0,a_0)^\top} \succeq \frac{1}{A}\Eb[\wt\pi]{\phi^\star_0(s_0,a_0)\phi^\star_0(s_0,a_0)^\top} \succeq \frac{\psi}{A}.
\end{align*}
\end{proof}

\begin{corollary}\label{corr:bound-true-bonus-by-learned-bonus}
Let $\lambda,\Lambda^\star_h,\wh\Lambda_h$ be defined as in the proof of \pref{lem:coverage}.
For any $h=0,1,...,H-2$ and any policy $\pi$, we have
\begin{align*}
    \Eb[\pi]{ \nm{\phi_{h+1}^\star(s_{h+1},a_{h+1})}_{\prns{\Lambda_{h+1}^\star}^{-1}} } \leq \Eb[\pi,\wh P]{\nm{\wh\phi_h(s_h,a_h)}_{\wh\Lambda_h^{-1}}} \sqrt{ |\Acal|(2d + \eps_{TV} N_{\LSVIUCB}) } + \eps_{TV}.
\end{align*}
Intuitively, this means that coverage in the learned features implies coverage in the true features.
\end{corollary}
\begin{proof}
For shorthand, let $N = N_{\LSVIUCB}$.
Apply \pref{lem:one-step-back-linear-mdp} (one-step back) to the learned model $\wh P$ and the function $(s,a) \mapsto \|\phi_{h+1}^\star(s,a)\|_{\prns{\Lambda_{h+1}^\star}^{-1}}$, which is bounded by $\lambda^{-1/2}\leq 1$.
We have,
\begin{align*}
&\;\EE_{\pi,\wh P} \nm{\phi_{h+1}^\star(s_{h+1},a_{h+1})}_{\prns{\Lambda_{h+1}^\star}^{-1}}
\\\leq&\; \EE_{\pi,\wh P} \nm{\wh\phi_h(s_h,a_h)}_{\wh\Lambda_h^{-1}} \sqrt{ N|\Acal| \EE_{\rho_h^{+1},\wh P} \|\phi_{h+1}^\star(s_{h+1},a_{h+1})\|_{\prns{\Lambda_{h+1}^\star}^{-1}}^2 + d }
\\\leq&\; \EE_{\pi,\wh P} \nm{\wh\phi_h(s_h,a_h)}_{\wh\Lambda_h^{-1}} \sqrt{ N|\Acal| \EE_{\rho_h^{+1}} \|\phi_{h+1}^\star(s_{h+1},a_{h+1})\|_{\prns{\Lambda_{h+1}^\star}^{-1}}^2 + N|\Acal|\eps_{TV} + d }
\\\leq&\; \EE_{\pi,\wh P} \nm{\wh\phi_h(s_h,a_h)}_{\wh\Lambda_h^{-1}} \sqrt{ d|\Acal| + N|\Acal|\eps_{TV} + d },
\end{align*}
where we used the fact that
\begin{align*}
    &\EE_{\rho_h^{+1}} \|\phi_{h+1}^\star(s_{h+1},a_{h+1})\|_{\prns{\Lambda_{h+1}^\star}^{-1}}^2
    \\&= \op{Tr}\prns{ \Eb[\rho_h^{+1}]{ \phi_{h+1}^\star(s_{h+1},a_{h+1}) \phi_{h+1}^\star(s_{h+1},a_{h+1})^\top }\prns{N \Eb[\rho_h^{+1}]{ \phi^\star_h(s_h,a_h)\phi^\star_h(s_h,a_h)^\top }+\lambda I}^{-1} }
    \\&= \frac{1}{N}\op{Tr}(I-M) \leq \frac{d}{N},
\end{align*}
where $M$ is a positive definite matrix.
Thus, doing an initial change from $d^\pi_{h+1}$ to $d^\pi_{\wh P,h+1}$ concludes the proof.
\end{proof}

\section{Representation Transfer}\label{sec:transfer_proof}
First, we prove \pref{lem:mle_target}, restated below.
\mleTargetGenerative*
\begin{proof}[Proof of \pref{lem:mle_target}]
Fix an arbitrary $\pi$.
% denote $\rho_{K;h} = d^\pi_{K;h}$ and
Denote $\mu_h(s') =\sum_{k=0}^{K-1}\alpha_{k;h}(s')\wh\mu_{k;h}(s')$. First, note that
\begin{align*}
    &\;\max_{g:\Scal \to [0,1]}\left\|\int \mu_h(s)g(s)\rd(s)\right\|_2
    \\
    \leq &\; \max_{g:\Scal \to [0,1]}\sum_{k=0}^{K-1}\left\|\int \wh\mu_{k;h}(s)\alpha_{k;h}(s)g(s)\rd(s)\right\|_2
    \\
    \leq &\; \sum_{k=0}^{K-1} \max_s\alpha_{k;h}(s)\sqrt{d} \tag{Since $\int\wh\mu_{k;h}(s)g(s)\rd(s) \leq \sqrt{d}$ by~\ref{ass:fcn_class}}
    \\
    = &\; \bar\alpha\sqrt{d}
\end{align*}

For any $h=0,1,...,H-1$, we have
\begin{align*}
    &\;\EE_{\pi, P_\targ^\star} \nm{\wh\phi_h(s_h,a_h)^\top\mu_h(\cdot) - \phi^\star_{h}(s_h,a_h)^\top\mu_{K;h}^\star(\cdot)}_{TV} \\
    = &\;\EE_{\pi, P_\targ^\star} \bracks{ \sum_{s_{h+1}} \abs{ \sum_{k=1}^{K}\alpha_{k;h}(s_{h+1})\left(\wh\phi_h(s_h,a_h)^\top \wh\mu_{k;h}(s_{h+1}) - \phi^\star_{h}(s_h,a_h)^\top\mu_{k;h}^\star(s_{h+1})\right) } }\\
    \leq &\;\EE_{\pi, P_\targ^\star} \bracks{ \sum_{s_{h+1}} \sum_{k=1}^{K} \abs{ \alpha_{k;h}(s_{h+1}) } \abs{ \wh\phi_h(s_h,a_h)^\top \wh\mu_{k;h}(s_{h+1}) - \phi^\star_{h}(s_h,a_h)^\top\mu_{k;h}^\star(s_{h+1}) } }\\
    \leq &\;\alpha_{\max} \sum_{k=1}^{K} \EE_{\pi, P_\targ^\star} \nm{ \wh\phi_h(s_h,a_h)^\top \wh\mu_{k;h}(\cdot) - \phi^\star_{h}(s_h,a_h)^\top\mu_{k;h}^\star(\cdot) }_{TV}.
\end{align*}

First consider the case when $h=0$.
At $h=0$, the distribution under $P_\targ^\star$ is the same as $\nu_{k,h}$, and so, we directly get that the above quantity is at most $\alpha_{\max} \zeta_n^{1/2} \leq \eps$, which proves the $h=0$ case.

Now consider any $h=1,2,...,H-1$.
To simplify notation, let us denote
\begin{align*}
    &\text{err}_{k;h}(s_h,a_h) = \nm{ \wh\phi_h(s_h,a_h)^\top \wh\mu_{k;h}(\cdot) - \phi_h^\star(s_h,a_h)^\top\mu_{k;h}^\star(\cdot)}_{TV},
    \\&w_{k;h} = \int_{s_h} \diff\mu^\star_{\targ;h-1}(s_h) \EE_{a_h\sim\pi_h(s_h)} \text{err}_{k;h}(s_h,a_h),
    \\&\Sigma_{k,h} = \Eb[\pi_k, P_k^\star]{\phi^\star_h(s_h,a_h)\phi^\star_h(s_h,a_h)^\top}.
\end{align*}
Note that $\lambda_{\min}(\Sigma_{k,h}) \geq \lambda_{\min}$ by assumption.
Now continuing from where we left off, we take a one-step back as follows,
\begin{align*}
    &\; \alpha_{\max} \sum_{k=1}^{K} \EE_{\pi, P_\targ^\star} \text{err}_{k;h}(s_h,a_h)
    \\
    =&\; \alpha_{\max}\sum_{k=1}^{K} \EE_{\pi, P_\targ^\star} \left\langle \phi^\star_{h-1}(s_{h-1},a_{h-1}), w_{k;h} \right\rangle
    \\
    \leq&\; \alpha_{\max}\sum_{k=1}^{K} \prns{ \EE_{\pi, P_\targ^\star} \nm{\phi^\star_{h-1}(s_{h-1}, a_{h-1})}_{\Sigma_{k;h-1}^{-1}} } \nm{w_{k;h}}_{\Sigma_{k;h-1}}
    \intertext{By $\lambda_{\min}$ guarantee of $\Sigma_{k,h}$, and Jensen's inequality to push the square inside, }
    \leq&\; \frac{\alpha_{\max}}{\sqrt{\lambda_{\min}}}\sum_{k=1}^{K} \sqrt{\EE_{s_{h-1},a_{h-1}\sim \pi_k, P_k^\star} \EE_{s_h\sim P^\star_{\targ;h-1}(s_{h-1},a_{h-1}),a_h\sim\pi_h(s_h)}\text{err}_{k;h}(s_h,a_h)^2}
    \\
    \leq&\; \frac{A^{1/2} \alpha_{\max}}{\sqrt{\lambda_{\min}}}\sum_{k=1}^{K} \sqrt{\EE_{s_{h-1},a_{h-1}\sim \pi_k, P_k^\star} \EE_{s_h\sim P^\star_{\targ;h-1}(s_{h-1},a_{h-1}),a_h\sim\op{unif}(\Acal)}\text{err}_{k;h}(s_h,a_h)^2}
    \intertext{By \pref{ass:linear-span}, the expectation over $P_{\targ;h-1}^\star$ is a linear combination of expectations over $P_{j;h-1}^\star$, }
    \leq&\; \frac{A^{1/2} \alpha_{\max}^{3/2}}{\sqrt{\lambda_{\min}}}\sum_{k=1}^{K} \sqrt{\sum_{j=1}^{K}\EE_{s_{h-1},a_{h-1}\sim \pi_k, P_k^\star} \EE_{s_h\sim P^\star_{j;h-1}(s_{h-1},a_{h-1}),a_h\sim\op{unif}(\Acal)}\text{err}_{k;h}(s_h,a_h)^2}
    \\\leq&\; \frac{A^{1/2} \alpha_{\max}^{3/2} K^{1/2}}{\sqrt{\lambda_{\min}}}\sqrt{ \sum_{k=1}^{K}\sum_{j=1}^{K} \EE_{s_{h-1},a_{h-1}\sim \pi_k, P_k^\star} \EE_{s_h\sim P^\star_{j;h-1}(s_{h-1},a_{h-1}),a_h\sim\op{unif}(\Acal)} \text{err}_{k;h}(s_h,a_h)^2}
    \\
    \leq&\;\frac{A^{1/2}\alpha^{3/2}_{\max}K^{1/2}\zeta_n^{1/2}}{\sqrt{\lambda_{\min}}},
\end{align*}
where we used the MLE guarantee \eqref{eq:mle} in the last step.
\end{proof}

Next we state an analogous lemma for when we don't need generative access to the source task, but instead assume \pref{ass:reachability_raw}, and \pref{ass:density}.
\begin{lemma}\label{lem:mle_target_online}
Suppose \pref{ass:reachability_raw}, and \pref{ass:density}.
Now take the setup of \pref{lem:mle_target} with the only difference being that $\wh\phi$ is learned as in \pref{alg:online}.
Then, the same guarantee of \pref{lem:mle_target} holds with a slightly different right hand side for the bound on the TV-error,
% Given \pref{ass:linear-span}, \pref{ass:reachability_raw}, and \pref{ass:density} suppose  $\lambda_{min}\prns{ \EE_{\rho_{k;h}}[\phi^\star_{h}(s,a)\phi^\star_{h}(s,a)^T]} \geq c$ (guaranteed by \pref{lem:coverage}) and $\hat\phi$ is learned by \pref{alg:online} which satisfies
% $\EE_{\rho_{k;h}}[\|\hat \phi_h(s,a)^\top \hat \mu_{i;h}(\cdot) - \phi_h^\star(s,a)^\top\mu_{i;h}^\star(\cdot)\|_{TV}^2] \leq \zeta$ (guaranteed by \pref{lem:mle}) for all source task $k$ and step $h$, we have that there exists a $\hat\mu_h:\Scal\rightarrow \RR^d$, such that
% for any function $g:\Scal \to [0,1]$, $\|\int \hat\mu_h(s)g(s)\rd(s)\|_2\leq \alpha_{\max}\sqrt{d}$ and we have
\begin{align*}
   \sup_\pi\EE_{\pi,P_\targ^\star}\nm{\wh\phi_h(s_h,a_h)^\top \wh\mu_h(\cdot) - \phi_h^\star(s_h,a_h)^\top\mu_{K;h}^\star(\cdot)}_{TV} \leq
   \frac{\alpha_{\max}K^{1/2}\zeta_n^{1/2}}{\prns{\psi_{raw}\lambda_{\min}}^{1/2}}.
\end{align*}
\end{lemma}

\begin{proof}[Proof of \pref{lem:mle_target_online}]
Fix an arbitrary $\pi$.
% Let $\pi$ be given, denote $\rho_{K;h} = d^\pi_{K;h}$ and
Denote $\mu_h(s') =\sum_{k=0}^{K-1}\alpha_{k;h}(s')\hat\mu_{k;h}(s')$.
Then, some algebra with importance sampling gives us the bound,
\begin{align*}
    &\;\EE_{\pi,P_\targ^\star} \nm{\wh\phi_h(s_h,a_h)^\top\mu_h(\cdot) - \phi_h^\star(s_h,a_h)^\top\mu_{K;h}^\star(\cdot)}_{TV}
    \\
    \leq &\;\EE_{\pi,P_\targ^\star} \bracks{\sum_{s_{h+1}}\abs{\sum_{k=1}^{K-1}\alpha_{k;h}(s_{h+1})\left(\wh\phi_h(s_h,a_h)^\top \wh\mu_{k;h}(s_{h+1}) - \phi_h^\star(s_h,a_h)^\top\mu_{k;h}^\star(s_{h+1})\right)}}
    \\
    \leq &\;\alpha_{\max}\sum_{k=1}^{K} \EE_{\pi,P_\targ^\star} \nm{\wh\phi_h(s_h,a_h)^\top \wh\mu_{k;h}(\cdot) - \phi_h^\star(s_h,a_h)^\top\mu_{k;h}^\star(\cdot) }_{TV}
    \\\leq &\;\alpha_{\max}K^{1/2} \sqrt{ \sum_{k=1}^{K} \EE_{\pi,P_\targ^\star} \nm{\wh\phi_h(s_h,a_h)^\top \wh\mu_{k;h}(\cdot) - \phi_h^\star(s_h,a_h)^\top\mu_{k;h}^\star(\cdot) }_{TV}^2 }
    \intertext{By \pref{ass:reachability_raw}, \pref{ass:density}, for any $s,a$, we have $\frac{d^\pi_{K;h}(s,a)}{d^{\pi_k}_{k;h}(s,a)} \leq \frac{1}{\psi_{raw} \lambda_{\min}\prns{\Eb[\pi_k,P_k^\star]{\phi_h^\star(s_h,a_h)\phi_h^\star(s_h,a_h)^\top}} } \leq \frac{1}{\psi_{raw} \lambda_{\min}}$, where we used the coverage-under-$\pi_k$ assumption in the last inequality.
    In other words, for each $k=1,2,...,K-1$, we have $\nm{\frac{\diff d^\pi_{K;h}}{\diff d^{\pi_k}_{k;h}}}_\infty \leq \frac{1}{\psi_{raw}\lambda_{\min}}$, hence we can importance sample, }
    \leq &\;\frac{\alpha_{\max}K^{1/2}}{\prns{\psi_{raw}\lambda_{\min}}^{1/2}}\sqrt{ \sum_{k=1}^{K} \EE_{\pi_k,P_k^\star} \nm{\wh\phi_h(s_h,a_h)^\top \wh\mu_{k;h}(\cdot) - \phi_h^\star(s_h,a_h)^\top\mu_{k;h}^\star(\cdot) }_{TV}^2 }
    \\\leq &\;\frac{\alpha_{\max}K^{1/2}\zeta_n^{1/2}}{\prns{\psi_{raw}\lambda_{\min}}^{1/2}}.
\end{align*}
\end{proof}

\section{Proofs for \LSVIUCB{} under average-case misspecification}\label{sec:lsvi-with-model-misspecification}
\subsection{Auxiliary RL Lemmas}

\begin{lemma}[Self-normalized Martingale]\label{lem:self-normalized-martingale}
Consider filtrations $\braces{F_i}_{i=1,2,...}$, so that $\EE[\eps_i\mid F_{i-1}] = 0$ and $\braces{\eps_i\mid F_{i-1}}_{i=1,2,...}$ are sub-Gaussian with parameter $\sigma^2$. Let $\braces{X_i}_{i=1,2,...}$ be random variables in a hilbert space $\Hcal$. Suppose a linear operator $\Sigma_0: \Hcal\to\Hcal$ is positive definite. For any $t$, define $\Sigma_t = \Sigma_0 + \sum_{i=1}^t X_iX_i^T$. Then w.p. at least $1-\delta$, we have,
\begin{align*}
    \forall t \geq 1: \nm{\sum_{i=1}^t X_i\eps_i}_{\Sigma_t^{-1}}^2 \leq \sigma^2\log\prns{ \frac{\det(\Sigma_t)\det(\Sigma_0)^{-1}}{\delta^2} }.
\end{align*}
\end{lemma}
\begin{proof}
Lemma A.8 of \citep{rltheorybookAJKS}.
\end{proof}

\begin{lemma}\label{lem:elliptical-potential-ub}
Let $\Lambda_t = \lambda I + \sum_{i=1}^t x_i x_i^T$ for $x_i\in \RR^d$ and $\lambda >0$. Then $\sum_{i=1}^t x_i^T(\Lambda_t)^{-1}x_i \leq d$.
\end{lemma}
\begin{proof}
Lemma D.1 of \citep{jin2020provably}.
\end{proof}

\subsection{Proof of main result}
Previously, \citet{jin2020provably} analyzed \LSVIUCB{} under point-wise model-misspecification.
Here, we show that similar guarantees hold under a more general \emph{policy-distribution} model-misspecification $\eps_{ms}$, captured by \pref{ass:model-misspecification}.
\begin{assum}\label{ass:model-misspecification}
Suppose for every $h=0,1,...,H-1$, there exist $\wt\mu_h$ such that for any policy $\pi$,
\begin{align*}
    \Eb[\pi]{ \nm{\wt\mu_h(\cdot)^T\wh\phi_h(s_h,a_h) - P_h^\star(\cdot \mid s_h,a_h)}_{TV} } \leq \eps_{ms}.
\end{align*}
We further assume that $\sup_{s,a,h} \nm{\wt\mu_h(\cdot)^T\wh\phi(s,a)}_{TV}\leq M_\mu$ and $\|f^T\wt\mu_h\|_2\leq M_\mu\sqrt{d}\|f\|_\infty$ $\forall f:\Scal\to\RR$, for some positive constant $M_\mu$.
\end{assum}
In other words, we only need the model to be accurate \textit{on average} under the occupancy distributions realizable by policies.
We also make a slight generalization on the regularization constant $M_\mu$, which is set to $1$ in the original linear MDP definition \citep{jin2020provably}.
Later, we will later instantiate the above assumption with our transferred $\wt\mu_h(s') = \sum_{k=1}^{K} \alpha_{k;h}(s')\wh\mu_{k;h}(s')$, then for any $s,a$, we have
\begin{align*}
\|\wt\mu_h\wh\phi_h(s,a)\|_{TV}
&= \sum_{s'}\left|\sum_{k=1}^{K} \alpha_{k;h}(s')\wh\mu_{k;h}(s')^T\wh\phi_h(s,a)\right|
\\&\leq \sum_{s'} \sum_{k=1}^{K}|\alpha_{k;h}(s')| |\wh\mu_{k;h}(s')^T\wh\phi_h(s,a)|
\\&\leq \sum_{k=1}^{K}\max_{s'}|\alpha_{k;h}(s')| \tag{by $\nm{\wh\mu_{k;h}\wh\phi_h(s,a)}_{TV} \leq 1$}
\\&\leq \bar\alpha.
\end{align*}
Also,
\begin{align*}
    \|f^T\wt\mu_h\|_2
    &= \nm{\sum_{s'} \sum_{k=1}^{K-1} \alpha_{k;h}(s')\wh\mu_{k;h}(s') f(s')}_2
    \\&= \sum_{k=1}^{K-1} \max_{s'} |\alpha_{k;h}(s')| \nm{ \sum_{s'} \wh\mu_{k;h}(s') f(s') }_2
    \\&\leq \bar\alpha \sqrt{d} \|f\|_\infty. \tag{by $\|f^T \wh\mu_{k;h}\|_2\leq \sqrt{d}\|f\|_\infty$}
\end{align*}
So we will set $M_\mu = \bar\alpha$.

Note that we only need the existence of $\wt\mu_h$ here, and $\wt\mu_h(\cdot)^T\wh\phi_h(s,a)$ need not be a valid probability kernel. In fact, it may even be negative valued.
% We don't need to know how to find $\wt\mu_h$, as we don't need to sample from the synthetic environment $\wt\mu_h^T\wh\phi_h$, which was needed in \pref{alg:reward_free}.
% Instead, in this section and in the deployment phase of \pref{alg:main}, we focus on deploying a policy with sub-linear regret (which we prove in this section) in the real, target environment.

In this section, we make a model-based analysis of LSVI. Similar approaches have been used in prior works, e.g. \citet{lykouris2021corruption,rltheorybookAJKS,zhang2022efficient}.
For simplicity, we suppose that $\Scal$ is finite, but may be exponentially large, as we suffer no dependence on $|\Scal|$.
The proof can be easily extended to infinite state spaces by replacing inner products with $P$ by integrals.

% So, $P_h(s,a)$ and $\mu_h\phi_h(s,a)$ refers to $S$ dimensional vectors of probability distributions.
% Extending to infinite state spaces can be done by replacing $(P_h^\star)^T f$ (equivalent to finite-state expectation) by $\Eb[P_h^\star]{f}$, and similarly for the estimated transitions $(\wh\mu_h^e)^T\wh\phi_h$.
% Note that in the finite state case, the TV norm reduces to the $\ell_1$ norm, i.e. $\|P\|_{TV} = \|P\|_1$; but we write $\|\cdot\|_{TV}$ to keep generality.
Consider the following quantity,
\begin{align*}
    &\;\wh\mu_{h,e}=\prns{ \sum_{k=1}^{e-1}\delta(s_{h+1}^k) \wh\phi_h(s_h^k,a_h^k)^T } (\Lambda_{h,e})^{-1} \in\argmin_{\mu\in\RR^{S\times d}} \sum_{k=1}^{e-1} \|\mu\wh\phi_h(s_h^i,a_h^i) - \delta(s_{h+1}^i)\|_2 + \|\mu\|_F^2,
\end{align*}
where $\delta(s)$ is a one-hot encoding of the state $s$.
In words, this is the best choice for linearly (in $\wh\phi_h(s,a)$) predicting $\Eb[s'\sim P_h^\star(s,a)]{\delta(s')} = P_h^\star(s' \mid s,a)$. We highlight that this is just a quantity for analysis and not computed in the algorithm. Finally, denote
\begin{align*}
    &\wh P_{h,e} = \wh\mu_{h,e}\wh\phi_h,
    \\&\wt P_h = \wt\mu_h\wh\phi_h.
\end{align*}
We will also sometimes use the shorthand $P f(s,a)$ for $\EE_{s'\sim P(\cdot|s,a)}[f(s')]$.

For each $h=0,1,...,H-1$, let $\Vcal_h$ denote the class of functions
\begin{align*}
    \left\{s\mapsto \prns{  \max_a \braces{ w^T\wh\phi_h(s,a) + r_h(s,a) + \wt\beta\|\wh\phi_h(s,a)\|_{\Lambda^{-1}} } }_{\leq M_V} \bigg\vert \|w\|_2\leq N M_V, \wt\beta\in[0, B], \Lambda \succeq I \mbox{ symmetric}\right\}
\end{align*}
The motivation behind this construction is that $\Vcal_h$ satisfies the key property that all of the learned value functions $\wh V_{h,e}$ during \pref{alg:LSVI} are captured in this class.

\begin{lemma}\label{lem:V-properties}
For any $h=0,1,...,H-1$,
\begin{enumerate}
    \item $\sup_s \abs{\wh V_{h,e}(s)} \leq M_V$.
    \item For any $e=1,2,...,N$, we have $\wh V_{h,e}\in\Vcal_h$.
    \item $\forall f\in\Vcal_h$, we have $\sup_s \abs{f(s)} \leq M_V$.
\end{enumerate}
\end{lemma}
\begin{proof}
Recall that
\begin{align*}
    &\wh V_{h,e}(s) = \prns{ \max_a \braces{ \wh w_{h,e}^T\wh\phi_h(s,a) + r_h(s,a)  + \beta b_{h,e}(s,a) } }_{\leq M_V}
    \\\text{where } &\wh w_{h,e} = \Lambda_{h,e}^{-1} \sum_{k=1}^{e-1} \wh\phi_h(s_h^k,a_h^k)\wh V_{h+1,e}(s^k_{h+1}).
\end{align*}
From the thresholding, we have
\begin{align*}
    \abs{\wh V_{h,e}(s)} \leq M_V.
\end{align*}

We can bound the norm of $\wh w_{h,e}$ as follows,
\begin{align*}
    \nm{\wh w_{h,e}} \leq \nm{\Lambda_{h,e}^{-1}}_2 \sum_{k=1}^{e-1}\abs{\wh V_{h+1,e}(s_{h+1}^k)} \leq N \sup_s \abs{\wh V_{h+1,e}(s)} \leq N M_V.
\end{align*}
We also required $\beta\leq B$, and we regularized the covariance with $I$, so $\lambda_{\min}$ is at least 1. Hence $\wh V_{h,e}$ satisfies all the conditions to be in $\Vcal_h$.
\end{proof}

Now we control the metric entropy of $\Vcal_h$ in $\ell_\infty$, i.e. $d(f_1,f_2) = \sup_s \abs{f_1(s)-f_2(s)}$ for $f_i\in\Vcal_h$.
\begin{lemma}
Let $\eps > 0$ be arbitrary and let $\Ncal_\eps$ be the smallest $\eps$-net with $\ell_\infty$ of $\Vcal_h$. Then,
\begin{align*}
    \log\abs{\Ncal_\eps} \leq d\log(1 + 6L/\eps) + \log(1+6B/\eps) + d^2\log(1+18B^2\sqrt{d}/\eps^2).
\end{align*}
\end{lemma}
\begin{proof}
Let $f_1,f_2\in\Vcal_h$. Then,
\begin{align*}
    &\abs{ f_1(s)-f_2(s) }
    \\&\leq \max_a \abs{(w_1-w_2)^T\wh\phi_h(s,a) + \beta_1\nm{\wh\phi_h(s,a)}_{\Lambda_1^{-1}} - \beta_2\nm{\wh\phi_h(s,a)}_{\Lambda_2^{-1}}}
    \\&\leq \nm{w_1-w_2}_2 + \max_a \abs{(\beta_1-\beta_2) \nm{\wh\phi_h(s,a)}_{\Lambda_1^{-1}}} + \beta_2 \max_a\abs{ \nm{\wh\phi_h(s,a)}_{\Lambda_1^{-1}} - \nm{\wh\phi_h(s,a)}_{\Lambda_2^{-1}} }
    \\&\leq \nm{w_1-w_2}_2 + \abs{\beta_1-\beta_2} + B \max_a\sqrt{\nm{\wh\phi_h(s,a)}_{\Lambda_1^{-1}} - \nm{\wh\phi_h(s,a)}_{\Lambda_2^{-1}}} \tag{$\lambda_{\min}(\Lambda_1)\geq 1$}
    \\&\leq \nm{w_1-w_2}_2 + \abs{\beta_1-\beta_2} + B \sqrt{\nm{\Lambda_1^{-1} - \Lambda_2^{-1}}_2},
\end{align*}
where we used for any $a,b \geq 0$, we have $\abs{\sqrt{a}-\sqrt{b}} =\frac{\sqrt{\abs{a-b}}}{\sqrt{a}+\sqrt{b}} \sqrt{\abs{a-b}} \leq \sqrt{|a-b|}$.
Now proceeding like the Lemma 8.6 in the RL Theory Monograph \citep{rltheorybookAJKS}, we have the result.
\end{proof}

In this section, we'll use the following bonus scaling parameter,
\begin{align}
    % &\;\beta := C\prns{ \sqrt{dH^2 M\eps N} + HM\sqrt{ d^3\log(dHN/\delta) } } \label{eq:lsvi-bonus-param}
    &\;\beta := \Ocal\prns{ \sqrt{Nd} \eps_{ms} M_V + M_VM_\mu d\sqrt{ \log(dNM_V/\delta) } }. \label{eq:lsvi-bonus-param}
    % \\&\; b_h^e(s,a) = \beta\|\wh\phi_h(s,a)\|_{(\Lambda_h^e)^{-1}}. \nonumber
\end{align}

The following high probability event \pref{eq:E_model} is a key step in our proof.
Essentially, \pref{thm:E_model_whp} guarantees that, for all functions in $\Vcal_h$, the model we learn is an accurate predictor of the expectation, up to a bonus and some vanishing terms.

For all the following lemmas and theorems, suppose \pref{ass:model-misspecification} and the bonus scaling $\beta$ is set as in \pref{eq:lsvi-bonus-param}.
Throughout the section, $\zeta_h(\tau_h)$ refers to indicator functions of the trajectory $\tau_h$, where $\tau_h = (s_0,s_1,...,s_h)$. As before, the expectations $\Eb[\pi]{g(\tau_h)}$ are with respect to the distribution of trajectories when $\pi$ is executed in the environment $P^\star$.

\begin{theorem}\label{thm:E_model_whp}
Let $\delta\in(0,1)$. Then, w.p. $1-\delta$, for any time $h$, episode $e$, indicator functions $\zeta_1,\ldots,\zeta_H$, and policy $\pi$,  we have
\begin{align}\tag{$\Ecal_{model}$}\label{eq:E_model}
    \sup_{f\in\Vcal}\abs{ \Eb[\pi]{ \prns{ \wh P_{h,e}(s_h,a_h) - P_h^\star(s_h,a_h) } f \zeta_h(\tau_h) } } \leq \beta \Eb[\pi]{ b_h^e(s_h,a_h) \zeta_h(\tau_h) } + \|\Vcal_h\|_\infty \eps_{ms}.
\end{align}
\end{theorem}
\begin{proof}
Condition on the outcome of \pref{lem:projection-bound}, which implies that w.p. $1-\delta$, for any $h,e,\pi,\zeta_h$, we have
\begin{align*}
    \sup_{f\in\Vcal_h} \abs{ \Eb[\pi]{ \prns{ \wh P_{h,e}(s_h,a_h) - \wt P_h(s_h,a_h) } f  \zeta_h(\tau_h) } } \leq \beta \Eb[\pi]{ b_h^e(s_h,a_h) \zeta_h(\tau_h) }.
\end{align*}
Also, for any $h,e,\pi,\zeta_h$, by \pref{ass:model-misspecification}, we have (w.p. 1) that
\begin{align*}
    \sup_{f\in\Vcal_h} \abs{ \Eb[\pi]{ \prns{ \wt P_h(s_h,a_h) - P_h^\star(s_h,a_h) } f \zeta_h(\tau_h) } }
    &\leq \Eb[\pi]{ \sup_{f\in\Vcal_h} \abs{ \prns{ \wt P_h(s_h,a_h) - P_h^\star(s_h,a_h) } f } \zeta_h(\tau_h) }
    \\&\leq \Eb[\pi]{ \sup_{f\in\Vcal_h} \abs{ \prns{ \wt P_h(s_h,a_h) - P_h^\star(s_h,a_h) } f } }
    \\&\leq \|\Vcal_h\|_\infty \eps_{ms}.
\end{align*}
Combining these two yields the result, as
\begin{align*}
    &\;\sup_{f\in\Vcal_h} \abs{ \Eb[\pi]{  \prns{ \wh P_{h,e}(s_h,a_h) - P_h^\star(s_h,a_h) } f \zeta_h(\tau_h)} }
    \\\leq&\; \sup_{f\in\Vcal_h} \abs{ \Eb[\pi]{ \prns{ \wt P_h(s_h,a_h) - P_h^\star(s_h,a_h) } f \zeta_h(\tau_h) } }  +
    \sup_{f\in\Vcal_h} \abs{ \Eb[\pi]{ \prns{ \wh P_{h,e}(s_h,a_h) - \wt P_h(s_h,a_h) } f  \zeta_h(\tau_h) } }.
\end{align*}
\end{proof}

\begin{lemma}\label{lem:projection-bound}
Suppose \pref{ass:model-misspecification} and the bonus scaling $\beta$ is set as in \pref{eq:lsvi-bonus-param}.
For any $\delta\in(0,1)$, w.p. at least $1-\delta$, we have for any time $h$, episode $e$, and policy $\pi$,
\begin{align*}
    \forall s_h,a_h: \sup_{f\in\Vcal_h} \abs{ \prns{ \wh P_{h,e}(s_h,a_h) - \wt P_h(s_h,a_h) } f } \leq \beta b_{h,e}(s_h,a_h).
\end{align*}
\end{lemma}
\begin{proof}
Consider any $h,e,\pi$.
Define $\eps_h^k := -\delta(s_{h+1}^k) + P_h^\star(s^k_{h+1} | s_h^k, a_h^k)$, so that $\EE[\eps_h^k \mid \Hcal_{k-1}] = 0$, where $\Hcal_{k-1}$ contains the states and actions before episode $k$. In what follows, we slightly abuse notation, as $P(s,a) \wh\phi^T(s,a)$ will denote the outer product, and hence a $\RR^{S\times d}$ quantity.
\begin{align*}
    \wh\mu_{h,e}\Lambda_{h,e}
    =&\; \sum_{k=1}^{e-1}\delta(s_{h+1}^k)\wh\phi_h(s_h^k,a_h^k)^T
    \\=&\; \sum_{k=1}^{e-1} \prns{ P_h^\star(s_h^k,a_h^k)-\wt P_h(s_h^k,a_h^k) } \wh\phi_h(s_h^k,a_h^k)^T
    + \sum_{k=0}^{e-1} \prns{ \wt P_h(s_h^k,a_h^k)-\eps_h^k } \wh\phi_h(s_h^k,a_h^k)^T
    \\=&\; \sum_{k=1}^{e-1} \prns{ P_h^\star(s_h^k,a_h^k)-\wt P_h(s_h^k,a_h^k) } \wh\phi_h(s_h^k,a_h^k)^T
    + \wt\mu_h (\Lambda_{h,e}- I) - \prns{\sum_{k=0}^{e-1} \eps_h^k \wh\phi_h(s_h^k,a_h^k)^T}.
\end{align*}
Rearranging, we have
\begin{align*}
    \wh\mu_{h,e} - \wt\mu_h
    =&\; \prns{\sum_{k=0}^{e-1} \prns{ P_h^\star(s_h^k,a_h^k)- \wt P_h(s_h^k,a_h^k) } \wh\phi_h(s_h^k,a_h^k)^T} \prns{\Lambda_{h,e}}^{-1}
    \\&\;- \wt\mu_h \prns{\Lambda_{h,e}}^{-1} - \prns{\sum_{k=0}^{e-1} \eps_h^k \wh\phi_h(s_h^k,a_h^k)^T} \prns{\Lambda_{h,e}}^{-1}.
\end{align*}

Now let $f\in\Vcal_h$ be arbitrary. For any $s_h,a_h$, multiply the above with $\wh\phi_h(s_h,a_h)$ and multiply with $f$, we have
\begin{align*}
    &\abs{ \prns{ \wh P_{h,e}(s_h,a_h) - \wt P_h(s_h,a_h) } f }
    \\&=
    \abs{ f^T \prns{ \wh\mu_{h,e} - \wt\mu_h } \wh\phi_h(s_h,a_h) }
    \\&\leq\; \underbrace{ \abs{ f^T \prns{\sum_{k=1}^{e-1} \prns{ P_h^\star(s_h^k,a_h^k)- \wt P_h(s_h^k,a_h^k) } \wh\phi_h(s_h^k,a_h^k)^T} \Lambda_{h,e}^{-1}\wh\phi_h(s_h,a_h) }}_{\sf{Term (a)}}
    \\&\;+ \underbrace{ \abs{ f^T\wt\mu_h\Lambda_{h,e}^{-1}\wh\phi_h(s_h,a_h)} }_{\sf{Term (b)}}
    \\&\;+ \underbrace{ \abs{ f^T \prns{\sum_{k=1}^{e-1} \eps_h^k \wh\phi_h(s_h^k,a_h^k)^T} \Lambda_{h,e}^{-1} \wh\phi_h(s_h,a_h)} }_{\sf{Term (c)}}.
\end{align*}

We can deterministically bound \textsf{Term (b)} as follows,
\begin{align*}
    &\;\sup_{f\in\Vcal_h} \abs{ f^T\wt\mu_h\Lambda_{h,e}^{-1}\wh\phi_h(s_h,a_h) }
    \\=&\; \sup_{f\in\Vcal_h}\abs{ (\Lambda_{h,e}^{-1/2} f^T\wt\mu_h)^T \prns{ \Lambda_{h,e}^{-1/2} \wh\phi_h(s_h,a_h) } }
    \\\leq&\; \sup_{f\in\Vcal_h} \nm{ \Lambda_{h,e}^{-1/2} }_2 \|f^T\wt\mu_h\|_2 b_{h,e}(s_h,a_h)
    \\\leq&\; \|\Vcal_h\|_\infty M_{\mu}\sqrt{d} b_{h,e}(s_h,a_h).  \tag{by \pref{ass:model-misspecification}}
\end{align*}
This term will be lower order compared to the other two.

% For \textsf{Term (a)} and \textsf{Term (c)}, we apply martingale concentration to each element of an $\eps_{net}$-net of $\Vcal_h$.
% The bound for \textsf{Term (a)} is derived in \pref{lem:lsvi-first-term-bound}.
We now derive the bound for \textsf{Term (c)} for any fixed $f\in\Vcal_h$. Observe that
\begin{align*}
    \;\abs{ f^T \prns{\sum_{k=1}^{e-1} \eps_h^k \wh\phi_h(s_h^k,a_h^k)^T} \Lambda_{h,e}^{-1} \wh\phi_h(s_h,a_h) }
    =&\; \abs{ \prns{\Lambda_{h,e}^{-1/2} \sum_{k=1}^{e-1} \wh\phi_h(s_h^k,a_h^k) (f^T \eps_h^k) }^T \prns{ \Lambda_{h,e}^{-1/2} \wh\phi_h(s_h,a_h) } }
    \\\leq&\; \nm{ \sum_{k=1}^{e-1} \wh\phi_h(s_h^k,a_h^k) (f^T \eps_h^k) }_{\Lambda_{h,e}^{-1}} b_{h,e}(s_h,a_h).
    % \leq&\; \prns{ 2\|\Vcal_h\|_\infty \sqrt{ 2\log(1/\delta) + d\log(N+1) } } b_{h,e}(s_h,a_h).
\end{align*}

Now we argue w.p. $1-\delta$, for any $e,h$ we have
\begin{align*}
    \nm{ \sum_{k=1}^{e-1} \wh\phi_h(s_h^k,a_h^k) (f^T \eps_h^k) }_{\Lambda_{h,e}^{-1}} \leq \prns{ 2\|\Vcal_h\|_\infty \sqrt{ 2\log(1/\delta) + d\log(N+1) } },
\end{align*}
which implies the claim about all $s_h,a_h$.
Indeed, we can apply \pref{lem:self-normalized-martingale}. Checking the preconditions,  $\Eb[P_h^\star(s_h,a_h)]{f^T \eps_h^k \mid \Hcal_{k-1}} = 0$, $\sigma \leq |f^T\eps_h^k| \leq \|f\|_\infty\|\eps_h^k\|_1 \leq 2 \|\Vcal_h\|_\infty$, $\det(\Sigma_0) = \det I = 1$, and $\det(\Sigma_t) = \det(\Lambda_{h,e}) \leq (e+1)^d$ since the largest eigenvalue is $e+1$. So, w.p. at least $1-\delta$, for all $e$, we have the above inequality.

Thus, for any fixed $f\in\Vcal_h$, w.p. $1-\delta$, for all $e,h$ we have,
\begin{align*}
    &\abs{ \prns{ \wh P_{h,e}(s_h,a_h) - \wt P_h(s_h,a_h) } f }
    \\&\leq \;\sf{Term (a)} + \sf{Term (b)} + \sf{Term (c)}
    \\&\leq\; \prns{ 4\|\Vcal_h\|_\infty(1+M_\mu)\sqrt{\log(1/\delta) + d\log(N)} + \sqrt{dN}\|\Vcal_h\|_\infty\eps_{ms}}  b_{h,e}(s_h,a_h)
    \\&+\; \prns{ \|\Vcal_h\|_\infty M_\mu\sqrt{d}  }  b_{h,e}(s_h,a_h)
    \\&+\; \prns{ 4\|\Vcal_h\|_\infty \sqrt{\log(1/\delta) + d\log(N)} } b_{h,e}(s_h,a_h)
    \\&\lesssim\; \prns{ \sqrt{ dN }\|\Vcal_h\|_\infty\eps_{ms}  + \|\Vcal_h\|_\infty M_\mu \sqrt{\log(1/\delta) + d\log(N)} } b_{h,e}(s_h,a_h) .
\end{align*}

Now we apply a covering argument.
Namely, union bound the above argument to every element in an $\eps_{net}$-net of $\Vcal_h$.
For any $f\in\Vcal_h$, let $\wt f$ be its neighbor in the net s.t. $\|\wt f-f\|_\infty \leq \eps_{net}$, so we have
\begin{align*}
    \abs{ \prns{ \wh P_{h,e}(s_h,a_h) - \wt P_h(s_h,a_h) } f } &\leq \abs{ \prns{ \wh P_{h,e}(s_h,a_h) - \wt P_h(s_h,a_h) } \wt f } + \abs{ \prns{ \wh P_{h,e}(s_h,a_h) - \wt P_h(s_h,a_h) } (\wt f - f) }
\end{align*}
and
\begin{align*}
    \abs{ \prns{ \wh P_{h,e}(s_h,a_h) - \wt P_h(s_h,a_h) } (\wt f - f) } \lesssim \|\wt f-f\|_\infty (N+1) \lesssim \eps_{net} N.
\end{align*}
Setting $\eps_{net} = N$, the metric entropy is of the order $d\log(N(M_V + B)) + \log(BN) + d^2\log(BdN)$.
The error incurred with this epsilon net is a constant, which is lower order.

Thus, we have
\begin{align*}
\forall s_h,a_h:
&\sup_{f\in\Vcal_h} \abs{ \prns{ \wh P_{h,e}(s_h,a_h) - \wt P_h(s_h,a_h) } f  }
\\&\lesssim \prns{ \sqrt{ dN } \|\Vcal_h\|_\infty\eps_{ms} + \|\Vcal_h\|_\infty M_\mu \sqrt{\log(1/\delta) + d\log(M_V) + d^2\log(BdN)} } b_{h,e}(s_h,a_h)
\\&\lesssim \prns{ \sqrt{ dN } M_V \eps_{ms} + M_V M_\mu \sqrt{\log(1/\delta) + d\log(M_V) + d^2\log(BdN)} } b_{h,e}(s_h,a_h)
\end{align*}

Note that $\beta$ scales as $\sqrt{\log B}$, so one can find a valid $B$ by solving $\beta\leq B$ for $B$.

\end{proof}

\begin{lemma}\label{lem:lsvi-first-term-bound}
Let $f\in\Vcal_h$.
For any $\delta\in(0,1)$, w.p. at least $1-\delta$, for any time $h$, episode $e$, we have
\begin{align*}
    \forall s_h,a_h:  &\; \abs{ f^T\prns{ \sum_{k=1}^{e-1}P_h^\star(s_h^k,a_h^k)-\wt P_h(s_h^k,a_h^k) } \wh\phi_h(s_h^k,a_h^k)^T\Lambda_{h,e}^{-1}\wh\phi_h(s_h,a_h) }
    \\\leq&\; \prns{ 4\|\Vcal_h\|_\infty(1+M_\mu)\sqrt{\log(1/\delta)+d\log(N)} + \sqrt{dN} \|\Vcal_h\|_\infty \eps_{ms}} b_{h,e}(s_h,a_h).
\end{align*}
\end{lemma}
\begin{proof}
First observe that
\begin{align*}
    &\;\abs{ f^T \prns{\sum_{k=1}^{e-1} (P_h^\star(s_h^k,a_h^k)-\wt\mu_h\wh\phi_h(s_h^k,a_h^k)) \wh\phi_h(s_h^k,a_h^k)^T} \Lambda_{h,e}^{-1} \wh\phi_h(s_h,a_h) }
    \\=&\; \abs{ \prns{\Lambda_{h,e}^{-1/2}\sum_{k=1}^{e-1} \wh\phi_h(s_h^k,a_h^k) f^T (P_h^\star(s_h^k,a_h^k)-\wt\mu_h\wh\phi_h(s_h^k,a_h^k)) }^T \prns{ \Lambda_{h,e}^{-1/2} \wh\phi_h(s_h,a_h) } }
    \\\leq&\; \left\| \sum_{k=1}^{e-1} \wh\phi_h(s_h^k,a_h^k)\wt\eps_k \right\|_{\Lambda_{h,e}^{-1}} b_{h,e}(s_h,a_h) ,
\end{align*}
where $\wt\eps_k = \prns{ P_h^\star(s_h^k,a_h^k)-\wt P_h(s_h^k,a_h^k) } f$.

Now we will argue that w.p. $1-\delta$, for all $e,h$,
\begin{align*}
     \left\| \sum_{k=1}^{e-1} \wh\phi_h(s_h^k,a_h^k)\wt\eps_k \right\|_{\Lambda_{h,e}^{-1}} \leq \prns{ 4\|\Vcal_h\|_\infty(1+M_\mu)\sqrt{\log(1/\delta)+d\log(N)} + \sqrt{dN} \|\Vcal_h\|_\infty \eps_{ms}},
\end{align*}
which will imply the claim for all $s_h,a_h$.

Apply self-normalized martingale concentration (\pref{lem:self-normalized-martingale}) to $X_i = \wh\phi_h(s_h^i,a_h^i)$ and $\eps_i = \wt\eps_i - \Eb{\wt\eps_i\mid \Hcal_{i-1}}$, where the expectation is over $(s_h^i,a_h^i)$ in the definition of $\wt\eps_i$.
To see sub-Gaussianity, bound the envelope, $|\wt\eps_k| \leq \|f\|_\infty \|P_h^\star(s_h^k,a_h^k)-\wt\mu_h\wh\phi_h(s_h^k,a_h^k)\|_{TV} \leq \|\Vcal_h\|_\infty (1+M_\mu)$, and thus $\sigma \leq |\tilde \eps_k| \leq 2\|\Vcal_h\|_\infty(1+M_\mu)$.
Now compute the determinants: $\det(\Lambda_{h,0}) = 1$ and since $\lambda_{\max}(\Lambda_{h,e}) \leq e+1$, we have that $\log\det(\Lambda_{h,e}) \leq d\log(e+1)$.
Hence, w.p. at least $1-\delta$, we have
\begin{align*}
    \forall e: \nm{ \sum_{k=1}^{e-1} \wh\phi_h(s_h^k,a_h^k) \prns{ \wt\eps_k-\Eb{\wt\eps_k\mid\Hcal_{k-1}} } }_{\Lambda_{h,e}^{-1}} \leq 2\|\Vcal_h\|_\infty(1+M_\mu)\sqrt{2 \log(1/\delta)+d\log(N+1)}.
\end{align*}

By \pref{ass:model-misspecification} applied to $\pi^k$ (the data-generating policy for episode $k$), we have $\abs{ \Eb{\wt\eps_k\mid \Hcal_{k-1}} } \leq \|\Vcal_h\|_\infty \eps_{ms}$.
% By \pref{lem:elliptical-potential-ub}, we have $\sum_{k=1}^{e-1} \| \wh\phi_h(s_h^k,a_h^k) \|_{\Lambda_{h,e}^{-1}}^2 \leq d$.
Recall for any scalars $c_i$ and vectors $x_i$, we have $\|\sum_i c_i x_i\| \leq \sum_i |c_i|\|x_i\| \leq \sqrt{\sum_i c_i^2} \sqrt{\sum_i \|x_i\|^2}$.
Thus,
\begin{align*}
    &\nm{ \sum_{k=1}^{e-1} \wh\phi_h(s_h^k,a_h^k) \Eb{\wt\eps_k\mid\Hcal_{k-1}} }_{\Lambda_{h,e}^{-1}}
    \\&\leq \sqrt{ \sum_{k=1}^{e-1} \| \wh\phi_h(s_h^k,a_h^k) \|_{\Lambda_{h,e}^{-1}}^2 } \sqrt{ \sum_{k=1}^{e-1} \Eb{\wt\eps_k\mid\Hcal_{k-1}}^2 }
    \\&\leq \sqrt{d} \sqrt{(e-1) } \|\Vcal_h\|_\infty \eps_{ms}  \tag{by \pref{lem:elliptical-potential-ub}}.
\end{align*}
Combining these two bounds concludes the proof.

\end{proof}

\begin{lemma}[Optimism]\label{lem:lsvi-optimism}
Suppose \pref{eq:E_model} holds. Let $\iota = \|\Vcal_h\|_\infty \eps_{ms}$.
Then, for any episode $e=1,2,...,N$, we have
\begin{align*}
    \forall h=0,1,...,H-1: \Eb[\pi^\star]{ \prns{ Q^\star_h(s_h,a_h)- \wh Q_{h,e}(s_h,a_h) } \zeta_h(\tau_h) } \leq (H-h)\iota,
\end{align*}
and
\begin{align*}
    \forall h=0,1,...,H-1: \Eb[\pi^\star]{  \prns{ V^\star_h(s_h) - \wh V_{h,e}(s_h) } \zeta_{h-1}(\tau_{h-1})  } \leq (H-h)\iota,
\end{align*}
where
\begin{align*}
    &\zeta_h(s_h) := \II\bracks{ \wh Q_{h,e}(s_h,\wh\pi_h^e(s_h)) \leq M_V }
    \\&\zeta_h(\tau_h) = \prod_{h'=0}^h \zeta_{h'}(s_{h'}).
\end{align*}
Abusing notation, $\zeta_{-1}(\cdot)$ is the constant function 1.

In particular, we have that
\begin{align*}
    \Eb[d_0]{ V^\star_0(s_0) - \wh V_{0,e}(s_0) } \leq H\iota.
\end{align*}
\end{lemma}
\begin{proof}
Fix any episode $e$.
We prove both claims via induction on $h = H,H-1,H-2...,1,0$.
The base case holds trivially since $\wh V_{H,e}$ and $V^\star_H$ are zero at every state by definition.
Indeed, we have that for any $\pi$, including $\pi^\star$, that
\begin{align*}
    &\Eb[\pi]{ \prns{ P_{H-1}^\star(s_{H-1},a_{H-1}) (V^\star_H-\wh V_{H,e}) } \zeta_{H-1}(\tau_{H-1}) }
    \\&= \Eb[\pi]{ \prns{ 0-0 } \zeta_{H-1}(\tau_{H-1}) }
    = 0.
\end{align*}

Now let's show the inductive step. Let $h \in \{H-1,H-2,...,1,0\}$ be arbitrary and suppose the inductive hypothesis. So suppose that $V$-optimism holds at $h+1$ (we don't even need $Q$-optimism in the future), i.e.
\begin{align}
    \Eb[\pi^\star]{ \prns{ P_h^\star(s_h,a_h) (V^\star_{h+1}-\wh V_{h+1,e}) } \zeta_h(\tau_h) }
    &= \Eb[\pi^\star]{ \prns{ V^\star_{h+1}(s_{h+1})-\wh V_{h+1,e}(s_{h+1}) } \zeta_h(\tau_h) } \nonumber
    \\&\leq (H-h-1)\iota \label{eq:optimism-inductive-hypothesis}\tag{IH}
\end{align}

% Also, one useful fact is that for any $x,y$ s.t. $y\in [0,c]$, we have $y-(x)_{\leq c} \leq |y-x|$.
% If $x \leq c$, then $y-(x)_{\leq c} =y-x$. Otherwise, if $x > c$, then $y-(x)_{\leq c} = y-c \leq 0 \leq |y-x|$.

Recalling that $\wh Q_{h,e}(s_h,a_h) = r_h(s_h,a_h) + \wh P_{h,e}(s_h,a_h) \wh V_{h+1,e} + \beta b_{h,e}(s_h,a_h)$, we have
\begin{align*}
    &\;\Eb[\pi^\star]{\prns{ Q^\star_h(s_h,a_h) - \wh Q_{h,e}(s_h,a_h)} \zeta_h(\tau_h) }
    \\=&\;\Eb[\pi^\star]{ \prns{ P_h^\star(s_h,a_h) V_{h+1}^\star - \wh P_{h,e}(s_h,a_h)\wh V_{h+1,e} - \beta b_{h,e}(s_h,a_h) } \zeta_h(\tau_h) }
    \\\leq&\;\Eb[\pi^\star]{  \prns{ \prns{ P_h^\star(s_h,a_h) - \wh P_{h,e}(s_h,a_h) }\wh V_{h+1,e} - \beta b_{h,e}(s_h,a_h) } \zeta_h(\tau_h) } + (H-h-1)\iota \tag{by \pref{eq:optimism-inductive-hypothesis}}
    \\\leq&\; \abs{ \Eb[\pi^\star]{ \prns{ \wh P_{h,e}(s_h,a_h) - P_h^\star(s_h,a_h) } \wh V_{h+1,e} \zeta_h(\tau_h) } } - \Eb[\pi^\star]{ \beta b_{h,e}(s_h,a_h) \zeta_h(\tau_h) } + (H-h-1)\iota
    \\\leq&\; \iota + (H-h-1)\iota = (H-h)\iota,  \tag{by \pref{eq:E_model} and $\wh V_{h+1,e} \in \Vcal_h$ (\pref{lem:V-properties})}
\end{align*}
which proves the $Q$-optimism claim.

Now let's prove $V$-optimism.
\begin{align*}
    &\Eb[\pi^\star]{ \prns{ V^\star_h(s_h) - \wh V_{h,e}(s_h) } \zeta_{h-1}(\tau_{h-1}) }
    \\&=\Eb[\pi^\star]{ \prns{ Q^\star_h(s_h,a_h) - \prns{ \wh Q_{h,e}(s_h,\wh\pi_h^e(s_h)) }_{\leq M_V} } \zeta_{h-1}(\tau_{h-1}) }
    % \\&=\Eb[\pi^\star]{ \prns{ Q^\star_h(s_h,a_h) } - \Eb[\pi^\star]{ M_V (1-\zeta_h(s_h)) } - \Eb[\pi^\star]{ \wh Q_{h,e}(s_h,\wh\pi_h^e(s_h)) } \zeta_h(s_h) }
    \\&=\Eb[\pi^\star]{ (Q^\star_h(s_h,a_h) - M_V) \zeta_{h-1}(\tau_{h-1}) (1-\zeta_h(s_h)) }
    \\&+ \Eb[\pi^\star]{ \prns{ Q^\star_h(s_h,a_h) - \wh Q_{h,e}(s_h,\wh\pi_h^e(s_h)) } \zeta_{h-1}(\tau_{h-1}) \zeta_h(s_h) }
    \\&\leq \Eb[\pi^\star]{ \prns{ Q^\star_h(s_h,a_h) - \wh Q_{h,e}(s_h,\wh\pi_h^e(s_h)) } \zeta_h(\tau_h) }
    \\&\leq \Eb[\pi^\star]{ \prns{ Q^\star_h(s_h,a_h) - \wh Q_{h,e}(s_h,\pi^\star_h(s_h)) } \zeta_h(\tau_h) }
    \\&\leq (H-h)\iota,
\end{align*}
by $Q$-optimism.
\end{proof}
\begin{remark}
We did not require $\wh P_{h,e}$ to be a valid transition!
It is in general unbounded and can even have negative entries!
\end{remark}

\begin{lemma}[Simulation]\label{lem:lsvi-simulation}
For any episode $e=1,2,...,N$, we have
\begin{align*}
     \Eb[d_0]{\wh V_{0,e}(s_0)-V_{0}^{\pi^e}(s_0)} \leq \sum_{h=0}^{H-1} \Eb[\wh\pi^e]{ b_{h,e}(s_h,a_h) + (\wh P_h(s_h,a_h)-P_h^\star(s_h,a_h))\wh V_{h+1,e}] }
\end{align*}
% \begin{align*}
%     V^\pi_{P',r+b,h} - V^\pi_{P,r,h}
%     =&\; \sum_{t=h}^{H-1} \Eb[d^\pi_{P',t}]{ b(s,a) + (\EE_{P_t'(s,a)}-\EE_{P_t(s,a)})[Q^\pi_{P,r,t+1}(s',\pi)] }
%     \\=&\; \sum_{t=h}^{H-1} \Eb[d^\pi_{P,t}]{ b(s,a) + \Eb[P'(s,a)]{ Q^\pi_{P,r+b,t+1}(s',\pi)-Q^\pi_{P',r+b,t+1}(s',\pi) } }.
% \end{align*}
\end{lemma}
\begin{proof}
We progressively unravel the left hand side.
For any $s_0$,
\begin{align*}
    &\; \wh V_{0,e}(s_0) - V_0^{\pi^e}(s_0)
    \\\leq&\; \wh Q_{0,e}(s_0,\pi_0^e(s_0)) - Q_0^{\pi^e}(s_0,\pi_0^e(s_0))
    \\=&\; b_{0,e}(s_0,\pi_0^e(s)) + \prns{ \wh P_{0,e}(s_0,\pi_0^e(s)) - P_0^\star(s_0,\pi_0^e(s_0)) } \wh V_{1,e} + P_0^\star(s_0,\pi_0^e(s_0)) \prns{ \wh V_{1,e} - V_1^{\pi^e} },
\end{align*}
where the inequality is due to the thresholding on the value function.
Now, perform this recursively on the $P_0^\star(s_0,\pi_0^e(s_0)) \prns{ \wh V_{1,e} - V_1^{\pi^e} }$ term.
Doing this unravelling $h$ times gives the result.
\end{proof}

\begin{theorem}\label{thm:deployment_regret}
Suppose \pref{ass:model-misspecification}.
Then, for any $\delta\in(0,1)$, w.p. at least $1-\delta$,
we have that \LSVIUCB{} (with $\beta$ set to \pref{eq:lsvi-bonus-param}) has regret at most,
\begin{align*}
    &NV^\star - \sum_{e=0}^{N-1} V^{\wh\pi^e} \leq \wt\Ocal\prns{ dH N M_V \sqrt{\log(HN/\delta)}\eps_{ms} + d^{1.5} H \sqrt{N} M_V M_\mu \log(dHN/\delta)  }
\end{align*}
where $\wt\Ocal$ hides $\log$ dependence.

\end{theorem}
\begin{proof}
We first condition on the high-probability event \pref{eq:E_model}, which occurs w.p. at least $1-\delta$.
Fix any arbitrary episode $e$.
By optimism \pref{lem:lsvi-optimism} and the simulation lemma \pref{lem:lsvi-simulation},
\begin{align*}
    \Eb[d_0]{V_0^\star(s_0) - V_0^{\wh\pi^e}(s_0)}
    \leq&\; \Eb[d_0]{ \wh V_0^e(s_0) - V_0^{\wh\pi^e}(s_0) } + H\iota
    \\\leq&\; \sum_{h=0}^{H-1} \Eb[\wh\pi^e]{ \beta b_h^e(s_h,a_h) + \prns{\wh P_{h,e}(s_h,a_h)-P_h^\star(s_h,a_h)}\wh V_{h+1,e} } + H\iota
    \intertext{Applying \pref{eq:E_model} with no indicators, i.e. $\zeta_h(\tau_h) = 1$ always, gives, }
    \leq&\; \sum_{h=0}^{H-1} \Eb[\wh\pi^e]{ 2\beta b_{h,e}(s_h,a_h) } + 2H\iota.
\end{align*}
Now, summing over $e=1,2,...,N$, we have
\begin{align*}
    &\sum_{e=1}^N \Eb[d_0]{V_0^\star(s_0) - V_0^{\wh\pi^e}(s_0)}
    \\&\leq 2HN\iota + 2\beta \sum_{h=0}^{H-1} \sum_{e=1}^N \Eb[\wh\pi^e]{ b_{h,e}(s_h,a_h) }
    \intertext{By Azuma's inequality applied to the martingale difference $\Delta_e = \Eb[\wh\pi^e]{ b_{h,e}(s_h,a_h) } - b_{h,e}(s_h^e,a_h^e)$, which has envelope bounded by $2$, implies w.p. $1-\delta$, }
    &\leq 2HN\iota + 2\beta \sum_{h=0}^{H-1} \sum_{e=1}^N b_{h,e}(s_h^e,a_h^e) + 4\beta \sqrt{ N\log(HN/\delta) }.
\end{align*}
It remains to bound the sum of expected bonuses.
By \pref{lem:elliptical-potential-ub}, we know that almost surely,
\begin{align*}
    \sum_{h=0}^{H-1} \sum_{e=1}^N b_{h,e}(s_h^e,a_h^e) \leq H\sqrt{dN\log(N)}.
\end{align*}
So, putting everything together,
\begin{align*}
    &\sum_{e=1}^N \Eb[d_0]{V_0^\star(s_0) - V_0^{\wh\pi^e}(s_0)}
    \\&\lesssim HN\iota + \beta H\sqrt{dN\log(N)} + \beta \sqrt{ N\log(HN/\delta) }
    \\&\lesssim HN\iota + \prns{ \sqrt{dN} M_V \eps_{ms} + M_V M_\mu d \sqrt{ \log(dHN/\delta) } } \cdot H\sqrt{dN\log(HN/\delta)}
    \\&= HN\iota + dH N M_V \sqrt{\log(HN/\delta)}\eps_{ms} + d^{1.5} H \sqrt{N} M_V M_\mu \log(dHN/\delta).
\end{align*}
Note that $H\iota = H \|\Vcal\|_\infty \eps_{ms} = HM_V\eps_{ms}$ is of lower order (with respect to $N$), we can simply drop it. This concludes the proof.
\end{proof}

\begin{corollary}
By setting $\delta=1/N$, we have that expected regret also has the same rate as above.
\end{corollary}
\begin{proof}
The expected regret by law of total probability, since regret is at most $NH$,
\begin{align*}
    \Eb{\op{Reg}_N}
    \leq&\; \Eb{\op{Reg}_N \mid \pref{eq:E_model}} + NH (1-\PP(\pref{eq:E_model}))
    \\\leq&\; \Eb{\op{Reg}_N \mid \pref{eq:E_model}} + H.
    % \intertext{By \pref{thm:E_model_whp} and setting $\delta = 1/N$, }
    % \leq&\; CdH^2\sqrt{\log(N)} \prns{ \sqrt{\eps}N  + dH^2\sqrt{\log(N^2V)} } + H
    % \\\leq&\; \wt\Ocal\prns{ dH^2N\sqrt{\eps} + d^2H^4\sqrt{N} }.
\end{align*}
Since $H$ is lower-order, we have the same rate.
\end{proof}

% Thus, total regret is at most,
% \begin{align*}
%     \Eb{NV^\star - \sum_{e=0}^{N-1} V^{\wh\pi^e}}
%     \leq&\; C \prns{ NH\zeta + \sqrt{\eps}Nd\sqrt{\log(N)}  + H^3d^{2}\sqrt{n\log^2(NV/\delta)} } + \delta NH
%     \\\leq&\; C \prns{ NH^2\sqrt{\eps} + \sqrt{\eps}Nd\sqrt{\log(N)}  + H^3d^{2}\sqrt{N\log^2(NV/\delta)} } + \delta NH
%     \\\leq&\; \wt\Ocal\prns{ (H^2+d)N\sqrt{\eps} + H^3d^2N^{1/2} }.
% \end{align*}
% $\eps$ can be made arbitrarily small, since its sample complexity only depends on the source tasks.

\section{Proof of Main Theorems}
First we prove \pref{thm:main} and \pref{thm:generative-end-to-end}.
\generativeThm*
\generativeCor*

\begin{proof}[Proof of \pref{thm:main} and \pref{thm:generative-end-to-end}]
For the regret bound, set $M_V=H$ and $M_\mu=\bar\alpha$ and apply \pref{thm:deployment_regret}. This choice of $M_\mu$ is valid by the argument following \pref{ass:model-misspecification}. This gives us a regret bound of
\begin{align*}
    \wt\Ocal\prns{ dH^2T\eps_{ms} + \bar\alpha d^{1.5}H^2\sqrt{T}\log(1/\delta) },
\end{align*}
where $\eps_{ms}$ can be made smaller than $1/\sqrt{T}$, in which the second term dominates.

Now, we calculate the pre-training phase sample complexity in a source task.

First, let's calculate the reward-free model learning sample complexity, i.e. this is the number of samples required for learning $\wh P_k$. Recall that we need this to be sufficiently large such that $\eps_{TV} = 1/N_{\LSVIUCB}$.
As required by \pref{lem:coverage}, we need,
\begin{align*}
    N_{\textsc{RewardFree}} =&\; \wt\Ocal\prns{A^3d^4H^6\log\prns{|\Phi||\Upsilon|/\delta} N_{\LSVIUCB}^{2}}\\
    =&\; \wt\Ocal\prns{A^3d^4H^6\log\prns{|\Phi||\Upsilon|/\delta} \prns{ A^3d^6H^8 \psi^{-2}  }^2 } \\
    =&\; \wt\Ocal\prns{ A^9d^{16}H^{22}\psi^{-4} \log\prns{|\Phi||\Upsilon|/\delta} }.
\end{align*}

Second, we calculate the cross-sampling sample complexity. Recall that $n$ is the number of samples in each pairwise dataset. In order to reduce $\epsilon_{ms}$ to $1/\sqrt{T}$, by \pref{lem:mle_target}, we need
\begin{align*}
    1/\sqrt{T} \leq \eps_{ms} \leq &\; \prns{A\alpha_{\max}^3 K/\lambda_{\min}}^{1/2}\sqrt{\zeta_n}
    \\
     \leq&\; \prns{ A\alpha_{\max}^3 K/\lambda_{\min} }^{1/2}\sqrt{\frac{1}{n}\left(\log\frac{|\Phi|}{\delta}+K\log|\Upsilon|\right)}\tag{by \pref{eq:mle}}
\end{align*}
which implies that we need
\begin{align*}
    n \leq \lambda_{\min}^{-1} A\alpha_{\max}^3 K T \prns{ \log\frac{|\Phi|}{\delta}+K\log|\Upsilon| }.
\end{align*}
Incorporating the coverage result from \pref{lem:coverage} gives,
\begin{align*}
    n \leq A^4 \alpha_{\max}^3 d^5 H^7 K T \psi^{-2} \prns{ \log\frac{|\Phi|}{\delta}+K\log|\Upsilon| }.
\end{align*}

Since each task is in at most $K-2$ pairwise datasets, each of size $n$,
the total pre-training sample complexity per task is at most,
\begin{align*}
    &N_{\textsc{RewardFree}}+(K-2) \cdot n
    \\&= \wt\Ocal\prns{ A^9d^{16}H^{22}\psi^{-4} \log\prns{|\Phi||\Upsilon|/\delta} + A^4\alpha_{\max}^3d^5 H^7 K^2 T \psi^{-2} \prns{ \log\frac{|\Phi|}{\delta}+K\log|\Upsilon| }  }.
\end{align*}
\end{proof}

Now we prove \pref{thm:online}, restated below.
\onlineThm*
% \begin{theorem}[detailed version of \pref{thm:online}]\label{thm:online1}
% Under \pref{ass:shared_phi}, \pref{ass:fcn_class}, \pref{ass:reachability}, \pref{ass:span}, \pref{ass:reachability_raw}, and \pref{ass:density} with probability $1-\delta$, \pref{alg:online} with appropriate hyper-parameter achieves a regret bound of
% \begin{equation*}
%     \wt\Ocal\prns{ H^2\sqrt{\bar\alpha d^{2}\epsilon}\cdot T + \bar\alpha H^2d^{1.5}\cdot\sqrt{T\log(1/\delta)} }
% \end{equation*}
% on the target task, where $\epsilon$ can be made smaller than $1/T$ (in which case the second term dominates) after exploring for at most \begin{equation*} \wt\Ocal\left(\psi^{-2}H^{10}d^6|\Acal|^3 \log\prns{|\Phi||\Upsilon|/\delta}+(\alpha_{\max}KT/c_0)^2\psi^{-4}H^{10}d^6\left(\log\frac{|\Phi|}{\delta}+K\log|\Upsilon|\right)\right).
% \end{equation*}
% iterations in the source tasks.
% \end{theorem}
\begin{proof}[Proof of \pref{thm:online}]
We follows the same format as the proof of \pref{thm:main}.
The regret bound is identical.

Now let's compute the pre-training sample complexity.
The regret bound requires us to set $\eps_{ms} \leq 1/\sqrt{T}$. Here, our $\eps_{ms}$ comes from \pref{lem:mle_target_online}, so
\begin{align*}
    1/\sqrt{T} \leq \frac{\alpha_{\max}K^{1/2}}{\prns{\psi_{raw}\lambda_{\min}}^{1/2}} \sqrt{ \frac{1}{n}\prns{ \log\frac{|\Phi|}{\delta}+K\log|\Upsilon| }},
\end{align*}
which implies we need
\begin{align*}
    n \leq \frac{\alpha_{\max}^2K}{\psi_{raw}\lambda_{\min}} \prns{ \log\frac{|\Phi|}{\delta}+K\log|\Upsilon| }.
\end{align*}
Plugging in the coverage of \pref{lem:coverage},
\begin{align*}
    n
    &\leq \frac{\alpha_{\max}^2K T}{\psi_{raw}} \prns{ \log\frac{|\Phi|}{\delta}+K\log|\Upsilon| } \prns{ A^{-3}d^{-5}H^{-7}\psi^2 }^{-1}
    \\&\leq \frac{A^3\alpha_{\max}^2d^{5}H^{7}K T}{\psi_{raw} \psi^2} \prns{ \log\frac{|\Phi|}{\delta}+K\log|\Upsilon| }.
\end{align*}

Here, we only collect one dataset, so the total pre-training sample complexity is
\begin{align*}
    &N_{\textsc{RewardFree}} + n
    \\&= \wt\Ocal\prns{ A^9d^{16}H^{22}\psi^{-4} \log\prns{|\Phi||\Upsilon|/\delta} + A^3\alpha_{\max}^2d^{5} H^{7} K T \psi_{raw}^{-1}\psi^{-2} \prns{ \log\frac{|\Phi|}{\delta}+K\log|\Upsilon| }  }.
\end{align*}

\end{proof}

\section{Experiment Details}\label{app:sec:exp}

\subsection{Construction of Comblock}\label{app:comblock}
In this section we first introduce the vanilla Combination lock (Comblock) environment that is widely used as the benchmark for algorithms for Block MDPs. We provide a visualization of the comblock environment in Fig.~\ref{exp:comblock_and_table}(a). Concretely, the environment has a horizon $H$, and 3 latent states $z_{i;h}, i \in \{0,1,2\}$ for each timestep $h$ and 10 actions. Among the three latent states, we denote $z_0,z_1$ as the good states which leads to the final reward and $z_2$ as the bad states. At the beginning of the task, the environment will uniformly and independent sample 1 out of the 10 actions for each good state $z_{0;h}$ and $z_{1;h}$ for each timestep $h$, and we denote these actions $a_{0;h}, a_{1;h}$ as the optimal actions (corresponding to each latent state). These optimal actions, along with the task itself, determines the dynamics of the environment. At each good latent state $s_{0;h}$ or $s_{1;h}$, if the agent takes the correct action, the environment transits to the either good state at the next timestep (i.e., $s_{0;h+1}$, $s_{1;h+1}$) with equal probability. Otherwise, if the agent takes any 9 of the bad actions, the environment will transition to the bad state $s_{2;h+1}$ deterministicly, and the bad states transit to only bad states at the next timestep deterministicly. There are two situations where the agent receives a reward: one is uponing arriving the good states at the last timestep, the agent receives a reward of 1. The other is upon the first ever transition into the bad state, the agent receives an ``anti-shaped'' reward of 0.1 with probability 0.5. Such design makes greedy algorithms without strategic exploration such as policy optimization methods easily fail. For the initial state distribution, the environment starts in $s_{0;0}$ or $s_{1;0}$ with equal probability. The dimension of the observation is $2^{\ceil{\log(H+|\mathcal{S}|+1)}}$. For the emission distribution, given a latent state $s_{i;h}$, the observation is generated by first concatenate the one hot vectors of the state and the horizon, adding i.i.d. $\Ncal$(0,0.1) noise at each entry, appending 0 at the end if necessary. Then finally we apply a linear transfermation on the observation with a Hadamard matrix. Note that without a good feature or strategic exploration, it takes $10^H$ actions to reach the final goal with random actions.
\subsection{Construction of transfer setup in the observational coverage setting}
In this section we introduce the detailed construction of our first experiment. For the source environment, we simply generate 5 random vanilla comblock environment described in Section.\ref{app:comblock}. Note that in this way we ensure that the emission distribution shares across the sources, but the latent dynamics are different because the optimal actions are independently randomly selected. For the target environment, for each timestep $h$, we randomly acquire the optimal actions at $h$ from one of the sources and set it to be the optimal action of the target environment at timestep $h$, if the selected optimal actions are different for the two good states. Otherwise we keep sampling until they are different. Note that under such construction, since we fix the emission distribution, \pref{ass:linear-span} is satisfied if we set $\alpha=1$ for the source environment where we select the optimal action and $\alpha = 0$ for the other sources, at each timestep. To see how \pref{ass:reachability_raw} is satisfied, recall that comblock environment naturally satisfies \pref{ass:reachability}, and identical emission implies that the conditional ratio of all observations between source and target is 1.

\subsection{Construction of transfer setup in feature coverage setting}
Now we introduce the construction of the Comblock with Partitioned Observation (Comblock-PO) environment, which we use in our second experiment. Comparing with the vanilla comblock environment, the major difference is in the observation space. In this setting, the size of the observation depends on the number of source environments $K$. Let the size of the original observation space be $O = |\mathcal{O}|$, the size of the observation for comblock-PO is $KO$. For the $k$-th source environment, where $k \in [K]$, the environment first generates the $O$-dimensional observation vector as in the original comblock, and then embed it to the $(k-1)O$-th to $kO$-th entries of the $KO$-dimensional observation vector, where it is 0 everywhere else. Thus we can see that the observation space for each source environment is disjoint (and thus the name partitioned observations). For the target enviornment, since the latent dynamcis are the same, we only need to design the emission distribution: for each latent state $s_{i;h}$, we assign the emission distribution uniformly at random from one of the sources.

\subsection{Implementation details}
Our implementation builds on \algname \citep{zhang2022efficient}.
% \footnote{Code based on public repository: \url{https://github.com/yudasong/briee}.}.
In the Multi-task \textsc{RepLearn} stage, we requires our learned feature to predict the Bellman backup of all the sources simultaneously. Therefore, in each iteration we have $k$ discriminators and $k$ sets of linear weights (instead of 1 in \algname), where $k$ is the number of source environments. For the deployment stage we implement LSVI following \pref{alg:LSVI}.

To create the training dataset for Multi-task \textsc{RepLearn}, for each $(i,j)$ environment pairs where $i\neq j$, we collect $500$ samples for each timestep $h$. For each $(i,i)$ environment pairs, we collect $500 \times (k-1) \times k$ samples for each timestep $h$, where $k$ denotes the number of sources. Thus we ensure that the number of samples from cross transition of different environments is the same as the number of samples from cross transition of the same environment. For the online setting, we simply sample $1000 \times (k-1) \times k$ samples for each $(i,i)$ cross transition to ensure that the total number of samples is the same for \ouralgg and \ouralgo.

To sample the initial state action pair (i.e., $(\tilde s, \tilde a)$ pair as in \pref{eq:data_sample}), for $90\%$ of the samples, we follow the final policy from each source environment trained using \algname. For the remaining $10\%$, we follow the same policy to state $\tilde s$, and then take a uniform random policy to get $\tilde a$. With this sampling scheme we ensure that \pref{ass:reachability} is satisfied. In the setting of Section.~\ref{exp:sec:partition}, we follow a more simple procedure to ensure that the samples are more balanced among the three states: we skip the first sampling step from environment $i$ (i.e., sample $s$ given $(\tilde s, \tilde a)$), and simply reset environment $i$ to $s$, where $s$ is one of the three states with equal probability, and generate the observation accordingly. Note that such visitation distribution is also possible in the online setting with a more nuanced sampling procedure, and in the experiment we use the same sampling procedure for both \ouralgg and \ouralgo for a fair comparison.

%\newpage
\subsection{Hyperparameters}
In this section, we record the hyperparameters we try and the final hyperparameter we use for each baselines. The hyperparameters for \ouralg in the first experiment is in Table.~\ref{app:table:hyperparam:our:comblock}. The hyperparameters for \ouralg in the second experiment is in Table.~\ref{app:table:hyperparam:our:comblock-po}. The hyperparameters for \algname is in Table.~\ref{app:table:hyperparam:briee}. We use the same set of hyperparameters for \ouralgg and \ouralgo.

\begin{table}[h]
\caption{Hyperparameters for \ouralg in Comblock.}
%\vspace{-2em}
\label{app:table:hyperparam:our:comblock}
\centering
\begin{tabular}{ccc}
\toprule
                                                &\; Value Considered          &\; Final Value  \\
\hline
Decoder $\phi$ learning rate                    &\; \{1e-2\}                  &\; 1e-2         \\
Discriminator $f$ learning rate                 &\; \{1e-2\}                  &\; 1e-2         \\
Discriminator $f$ hidden layer size             &\; \{256\}                   &\; 256          \\
RepLearn Iteration $T$                          &\; \{30\}                    &\; 30           \\
Decoder $\phi$ number of gradient steps         &\; \{64\}                    &\; 64           \\
Discriminator $f$ number of gradient steps      &\; \{64\}                    &\; 64          \\
Decoder $\phi$ batch size                       &\; \{256\}                   &\; 256          \\
Discriminator $f$ batch size                    &\; \{512\}                   &\; 512          \\
RepLearn regularization coefficient $\lambda$   &\; \{0.01\}                  &\; 0.01         \\
Decoder $\phi$ softmax temperature              &\; \{1\}                     &\; 1            \\
Decoder $\phi_0$ softmax temperature            &\; \{0.1\}                   &\; 0.1          \\
LSVI bonus coefficient $\beta$                  &\; \{1,$\frac{H}{5}$\}       &\; 1            \\
LSVI regularization coefficient $\lambda$       &\; \{1\}                     &\; 1            \\
Buffer size                                     &\; \{1e5\}                   &\; 1e5          \\
Update frequency                                &\; \{50\}                    &\; 50           \\
Optimizer                                       &\; \{SGD\}                   &\; SGD          \\
\toprule
\end{tabular}
\end{table}

\begin{table}[h]
\caption{Hyperparameters for \ouralg in Comblock-PO.}
\centering
\begin{tabular}{ccc}
\toprule
                                                &\; Value Considered          &\; Final Value  \\
\hline
Decoder $\phi$ learning rate                    &\; \{1e-2\}                  &\; 1e-2         \\
Discriminator $f$ learning rate                 &\; \{1e-2\}                  &\; 1e-2         \\
Discriminator $f$ hidden layer size             &\; \{256,512\}               &\; 256          \\
Discriminator $f$ hidden layer number           &\; \{2,3\}                   &\; 3            \\
RepLearn Iteration $T$                          &\; \{30,40,50,100,150\}      &\; 50           \\
Decoder $\phi$ number of gradient steps         &\; \{64,80,128,256\}         &\; 64           \\
Discriminator $f$ number of gradient steps      &\; \{64,80,128,256\}         &\; 64           \\
Decoder $\phi$ batch size                       &\; \{256,512\}               &\; 512          \\
Discriminator $f$ batch size                    &\; \{256,512\}               &\; 512          \\
RepLearn regularization coefficient $\lambda$   &\; \{0.01\}                  &\; 0.01         \\
Decoder $\phi$ softmax temperature              &\; \{1\}                     &\; 1            \\
Decoder $\phi_0$ softmax temperature            &\; \{0.1,1\}                 &\; 1            \\
LSVI bonus coefficient $\beta$                  &\; \{1,$\frac{H}{5}$\}       &\; 1            \\
LSVI regularization coefficient $\lambda$       &\; \{1\}                     &\; 1            \\
Buffer size                                     &\; \{1e5\}                   &\; 1e5          \\
Update frequency                                &\; \{50\}                    &\; 50           \\
Optimizer                                       &\; \{SGD, Adam\}             &\; Adam         \\
\toprule
\end{tabular}
\label{app:table:hyperparam:our:comblock-po}
\end{table}

\begin{table}[h]
\caption{Hyperparameters for \algname in Comblock and Comblock-PO.}
%\vspace{-2em}
\centering
\begin{tabular}{ccc}
\toprule
                                                &\; Value Considered          &\; Final Value  \\
\hline
Decoder $\phi$ learning rate                    &\; \{1e-2\}                  &\; 1e-2         \\
Discriminator $f$ learning rate                 &\; \{1e-2\}                  &\; 1e-2         \\
Discriminator $f$ hidden layer size             &\; \{256\}                   &\; 256          \\
RepLearn Iteration $T$                          &\; \{30\}                    &\; 30           \\
Decoder $\phi$ number of gradient steps         &\; \{64\}                    &\; 64           \\
Discriminator $f$ number of gradient steps      &\; \{64\}                    &\; 64          \\
Decoder $\phi$ batch size                       &\; \{512\}                   &\; 512          \\
Discriminator $f$ batch size                    &\; \{512\}                   &\; 512          \\
RepLearn regularization coefficient $\lambda$   &\; \{0.01\}                  &\; 0.01         \\
Decoder $\phi$ softmax temperature              &\; \{1\}                     &\; 1            \\
Decoder $\phi_0$ softmax temperature            &\; \{0.1\}                   &\; 0.1          \\
LSVI bonus coefficient $\beta$                  &\; \{$\frac{H}{5}$\}         &\; $\frac{H}{5}$\\
LSVI regularization coefficient $\lambda$       &\; \{1\}                     &\; 1            \\
Buffer size                                     &\; \{1e5\}                   &\; 1e5          \\
Update frequency                                &\; \{50\}                    &\; 50           \\
\toprule
\end{tabular}
\label{app:table:hyperparam:briee}
\end{table}

\newpage
\subsection{Visualizations}\label{app:vis}
In this section we provide a comprehensive visualization of the decoders for all baselines in the target environment. We observe that the behaviors of all baselines are similar across the 5 random seeds. Thus to avoid redundancy, we only show the visualization from 1 random seed. We provide an example in Fig.~\ref{fig:comb:source1} on how to interpret the visualization: let the emission function of the target environment be $o$, and let the decoder that we are evaluating be $\phi$, and to generate the blue block in Fig.~\ref{fig:comb:source1}, we sample 30 observations $\{s_n\}_{n=1}^{30}$from the target environment at $z_{1,13}$, the latent state 1 (the title of the subplot) from timestep $13$ (the x-axis). Concretely, $\{s_n\}_{n=1}^{30} \sim o(\cdot\mid z_{1,13})$. The blue block denotes the three-dimensional decoded latent states $\hat z$ from these 30 observations: $\hat z = \frac{1}{30} \sum_{n=1}^{30} \phi(s_n)$.

In Figure.~\ref{fig:vis}, we provide a runnning example that explains the results showed in Figure.~\pref{exp:comblock_and_table} (b). We then follow the detailed visualizations in the following sections.
\begin{figure}
\begin{minipage}{1\textwidth}\label{fig:source_v_trans}
    \centering
    \includegraphics[width=0.4\linewidth]{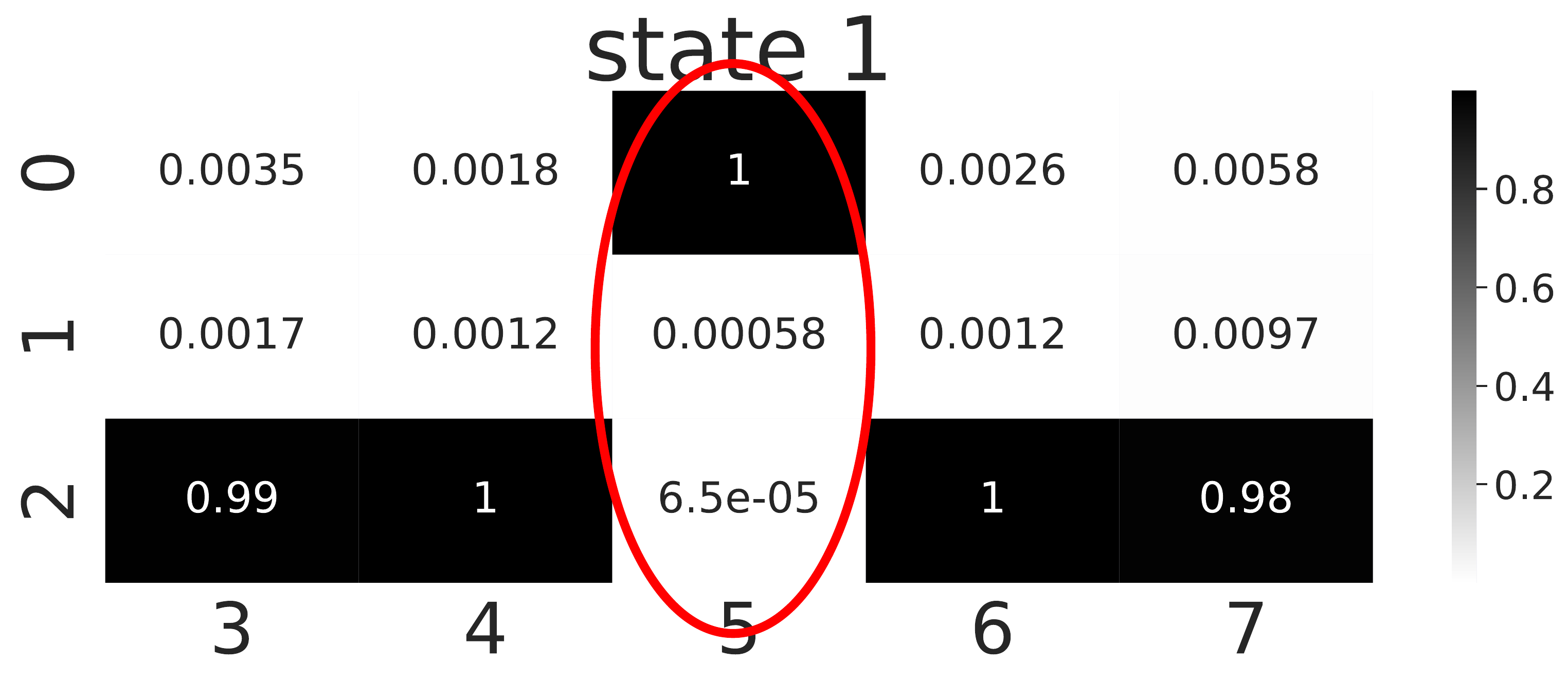}
    \includegraphics[width=0.4\linewidth]{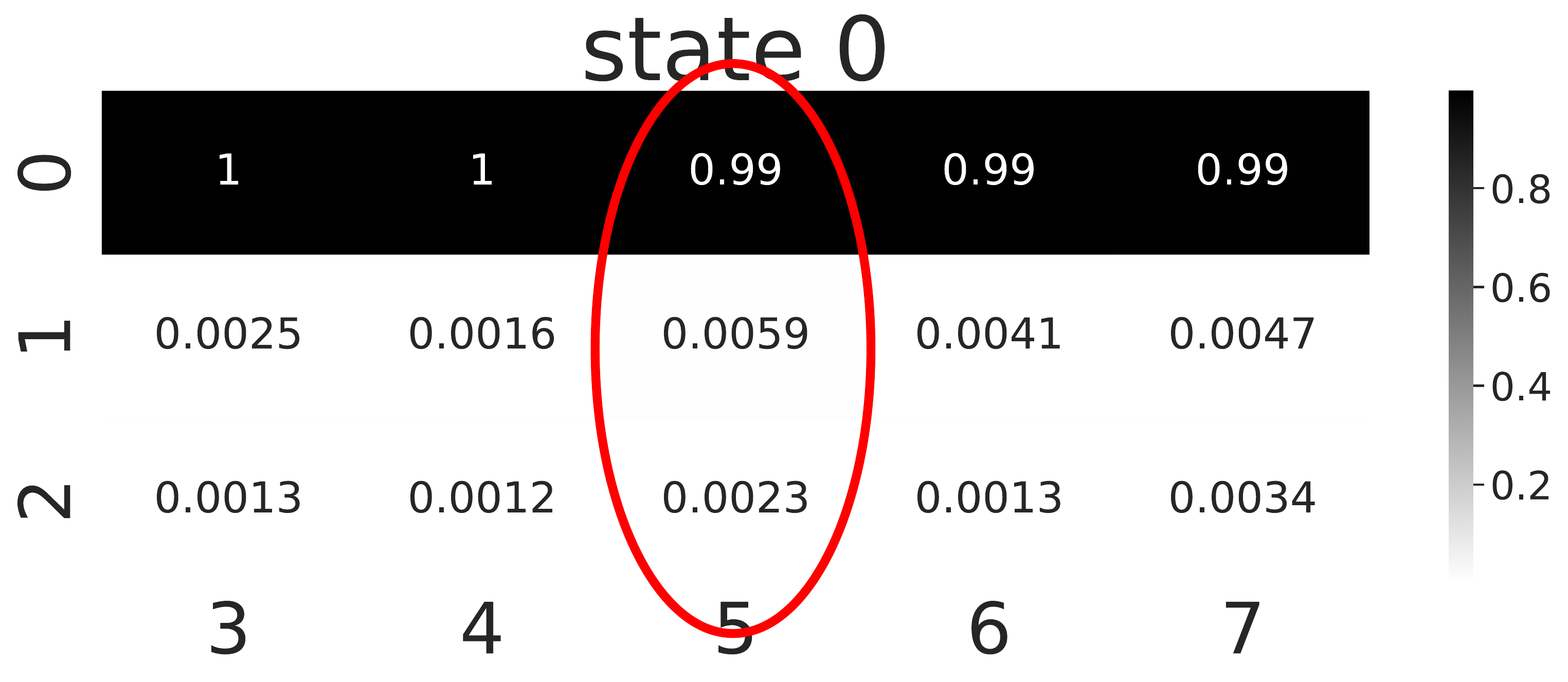}\\
    \includegraphics[width=0.4\linewidth]{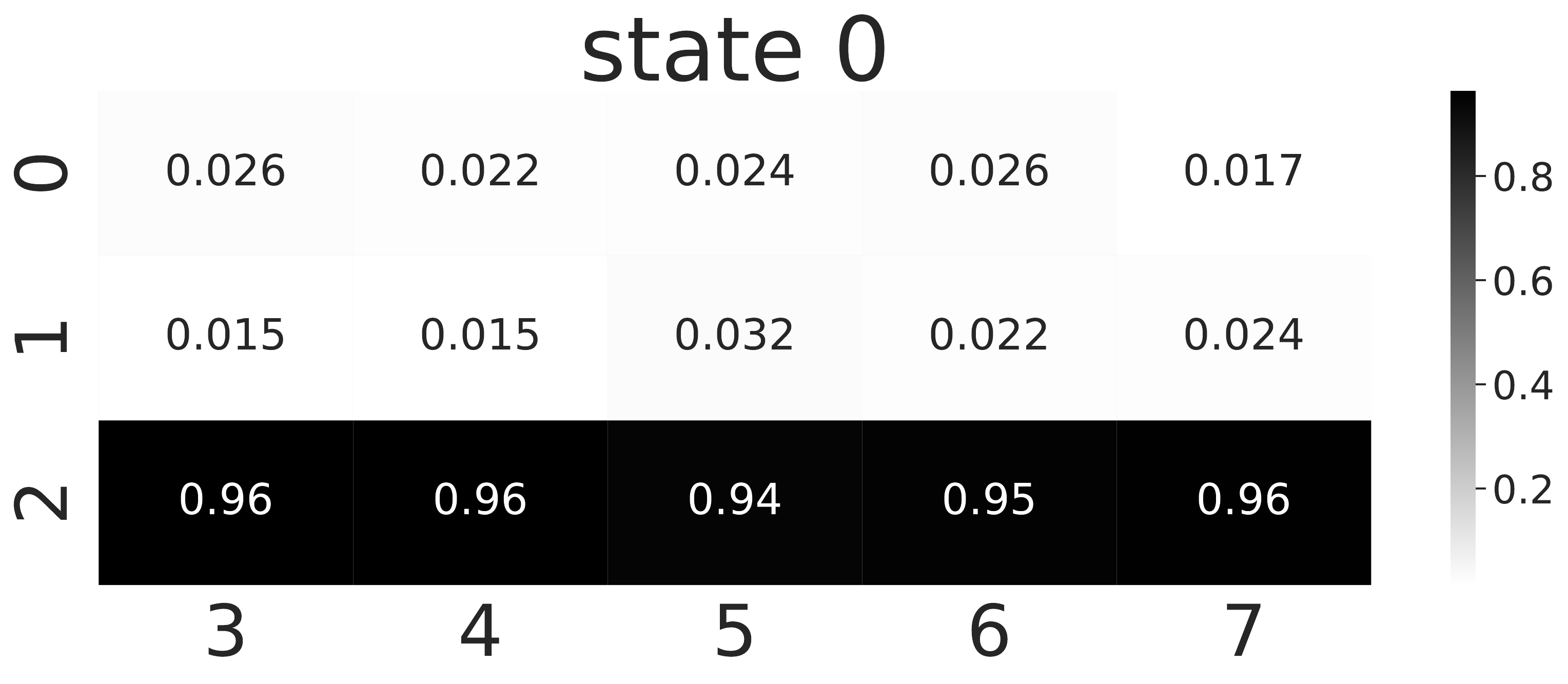}
    \includegraphics[width=0.4\linewidth]{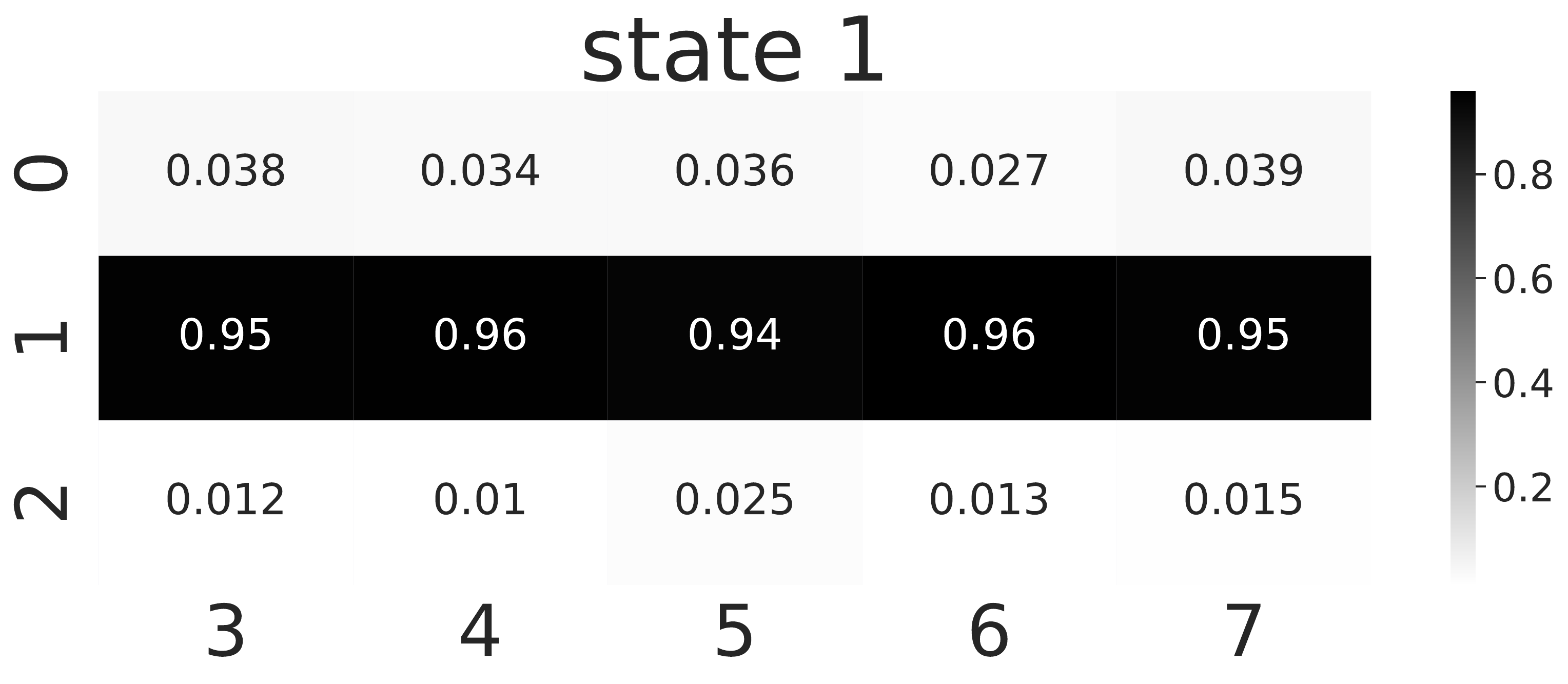}\\
    (a)
\end{minipage}
\begin{minipage}{1\textwidth}
    \centering
    \includegraphics[width=0.4\linewidth]{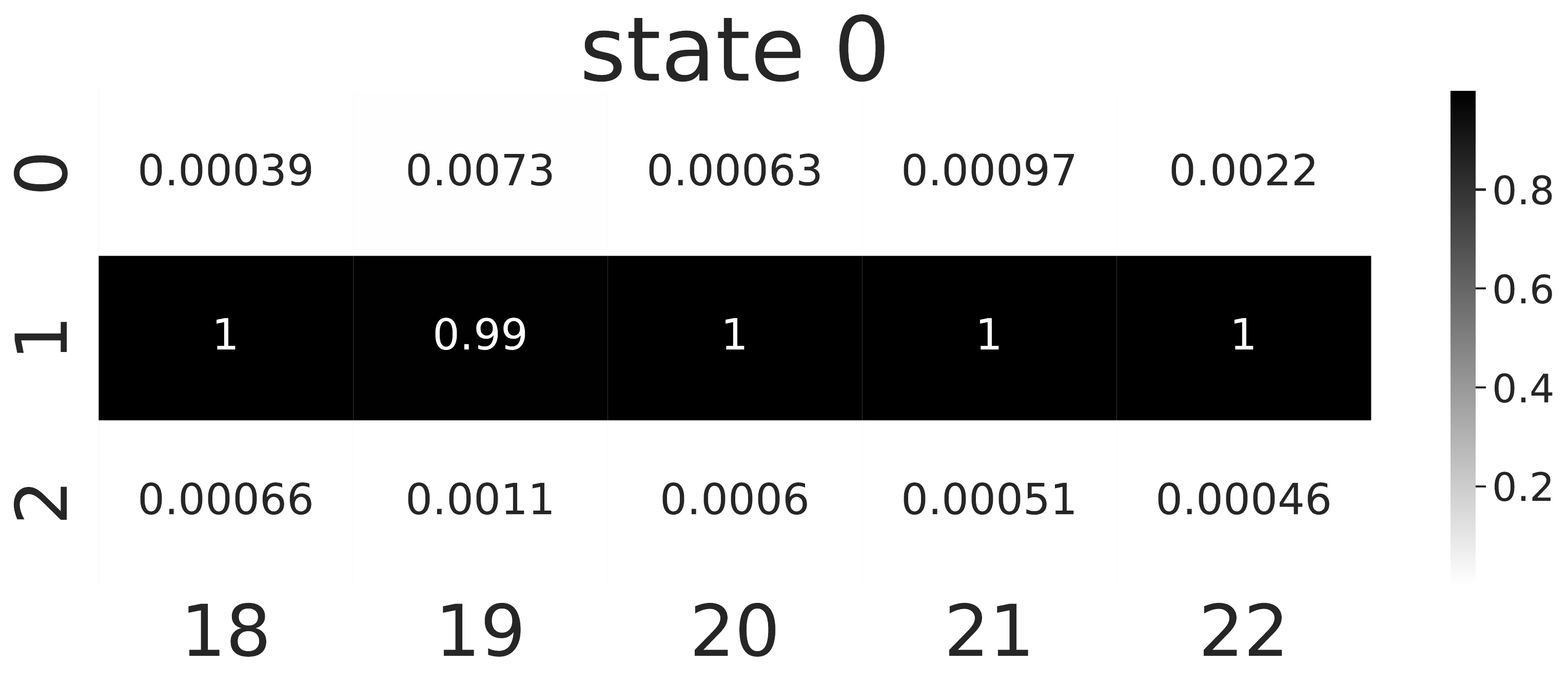}
    \includegraphics[width=0.4\linewidth]{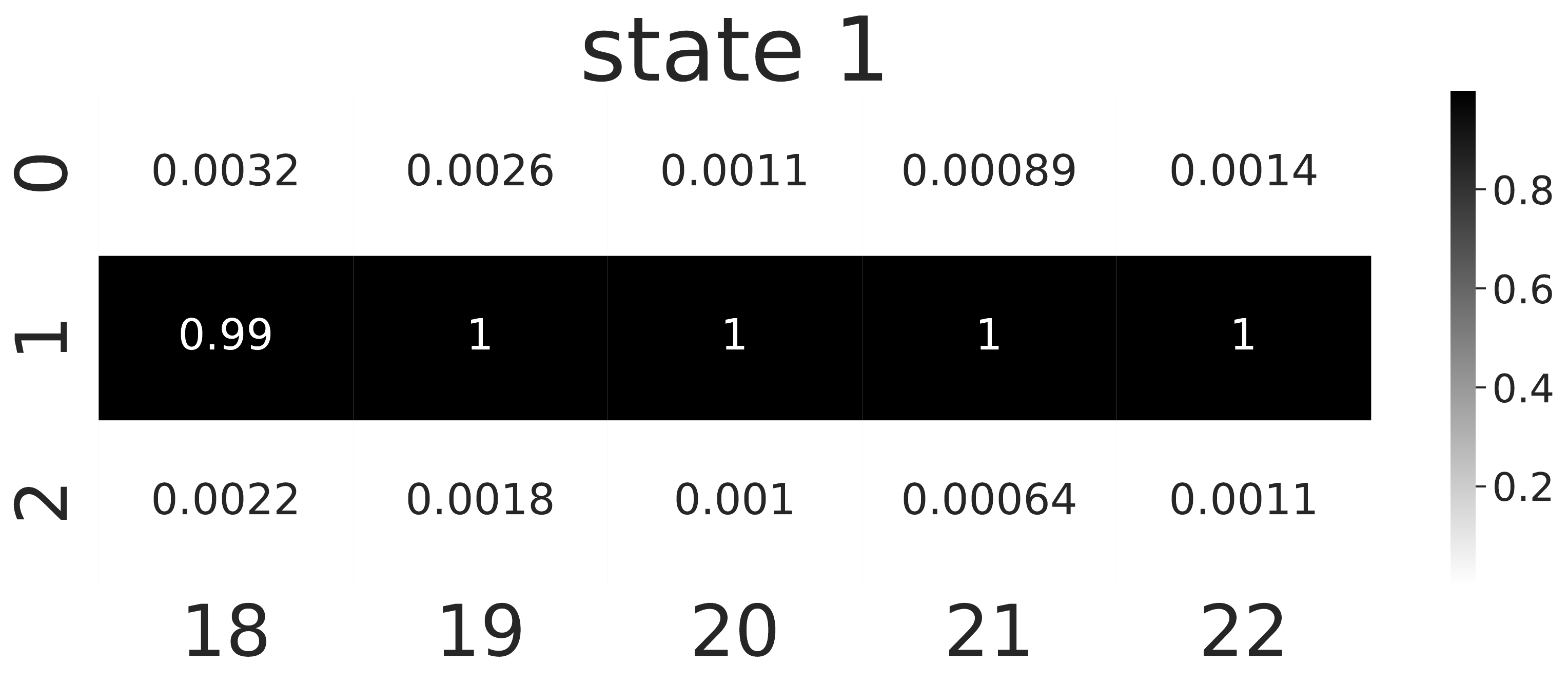}\\
    \includegraphics[width=0.4\linewidth]{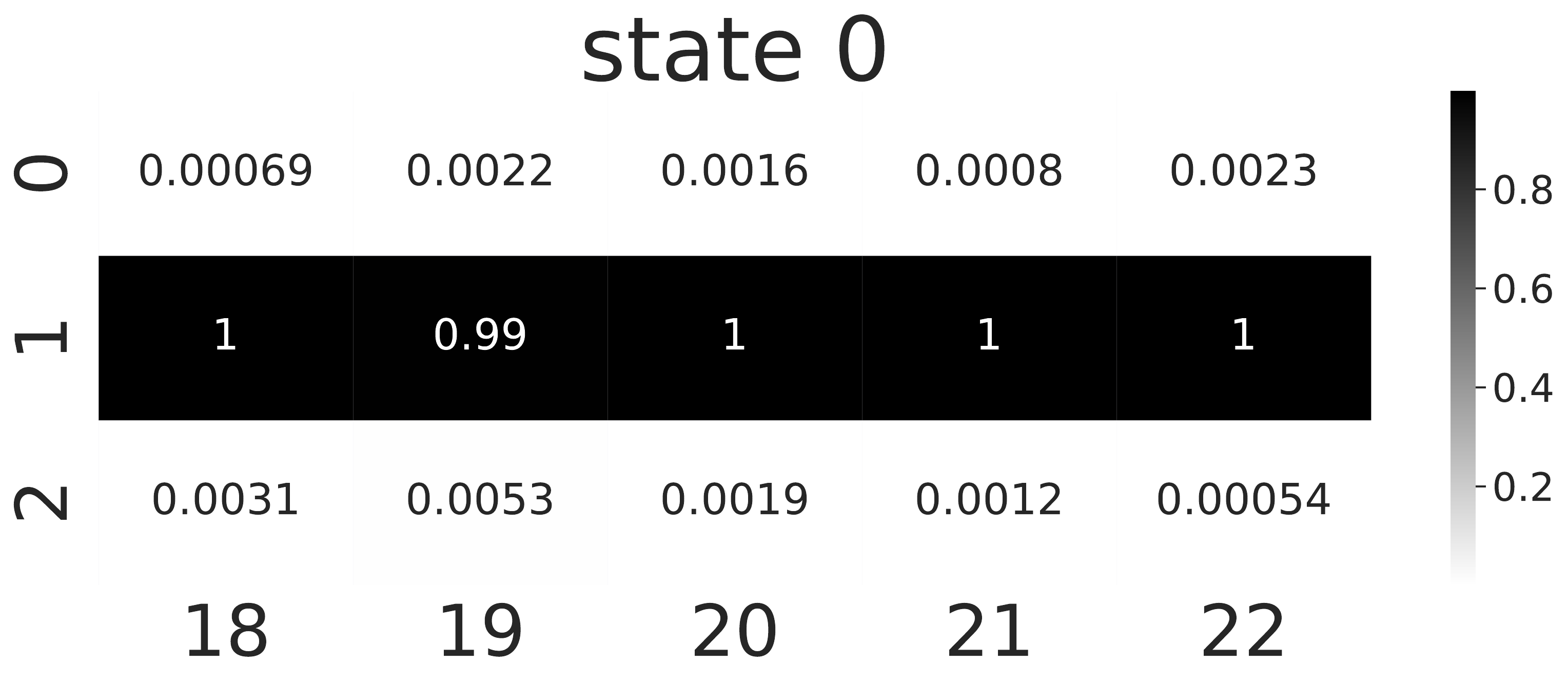}
    \includegraphics[width=0.4\linewidth]{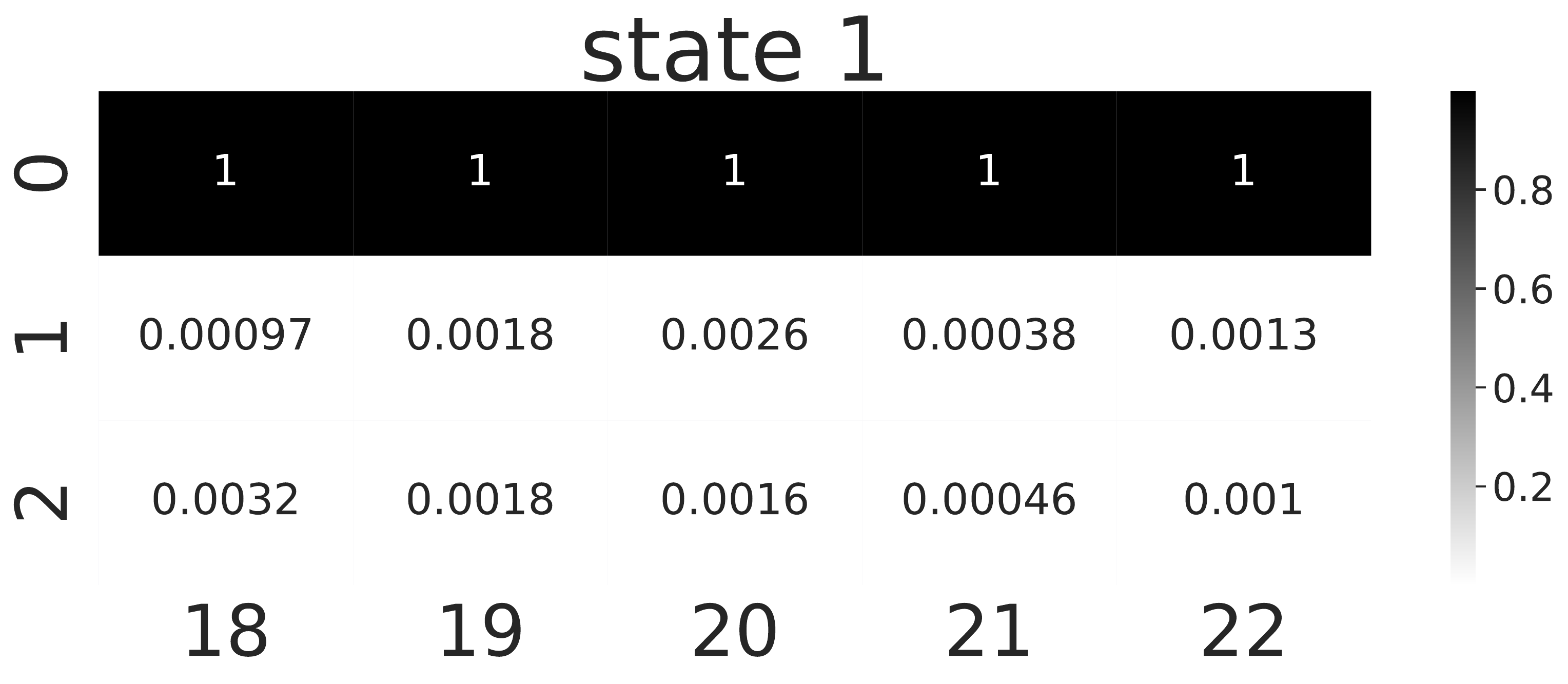}\\
    (b)
\end{minipage}
\caption{ \textbf{(a):} Visualization of the decoder source (top) and \ouralgg (bottom). \textbf{(b):} Visualization of the decoder \ouralgo (top) and \ouralgg (bottom). For each baseline, The $h$-th column in the $i$-th image denotes the averaged decoded states from the 30 observations generated by latent state $z_{i,h}$, for $i \in \{0, 1, 2\}$ and $h \in [25]$, from the corresponding \textit{target} environment. The optimal decoder should recover the latent states up to a permutation. In Fig a (top), note that the learned features in source task fail to solve the target because of the collapse at timestep 5: both observations from state 0 and 1 are mapped to state 0. Note in the source task where this feature is trained, such collapse can happen when state $0$ and $1$ have identical latent transition (for detailed discussion we refer to \cite{misra2020kinematic}). In Fig b (top), \ouralg with only online access learns an incorrect decoder when the source tasks' observation spaces are disjoint. This is because the learned feature can decode each source task with a different permutation.} %\textcolor{red}{More details can be found in \pref{app:vis}}}
%\wen{ we should say that the decoder is always visualized on the target task; }}
\label{fig:vis}
\end{figure}

\subsubsection{Visualizations from the observational coverage experiment}
%\yuda{Didn't realize the section numbers match perfectly lol.}
We record the visualization of the 5 sources from Fig.~\ref{fig:comb:source1} to Fig.~\ref{fig:comb:source5}; \ouralgo in Fig.\ref{fig:comb:online}; \ouralgg in Fig.~\ref{fig:comb:generative}; running \algname on target in Fig.~\ref{fig:comb:briee}.
\begin{figure}[h]
    \centering
    \includegraphics[width=0.6\linewidth]{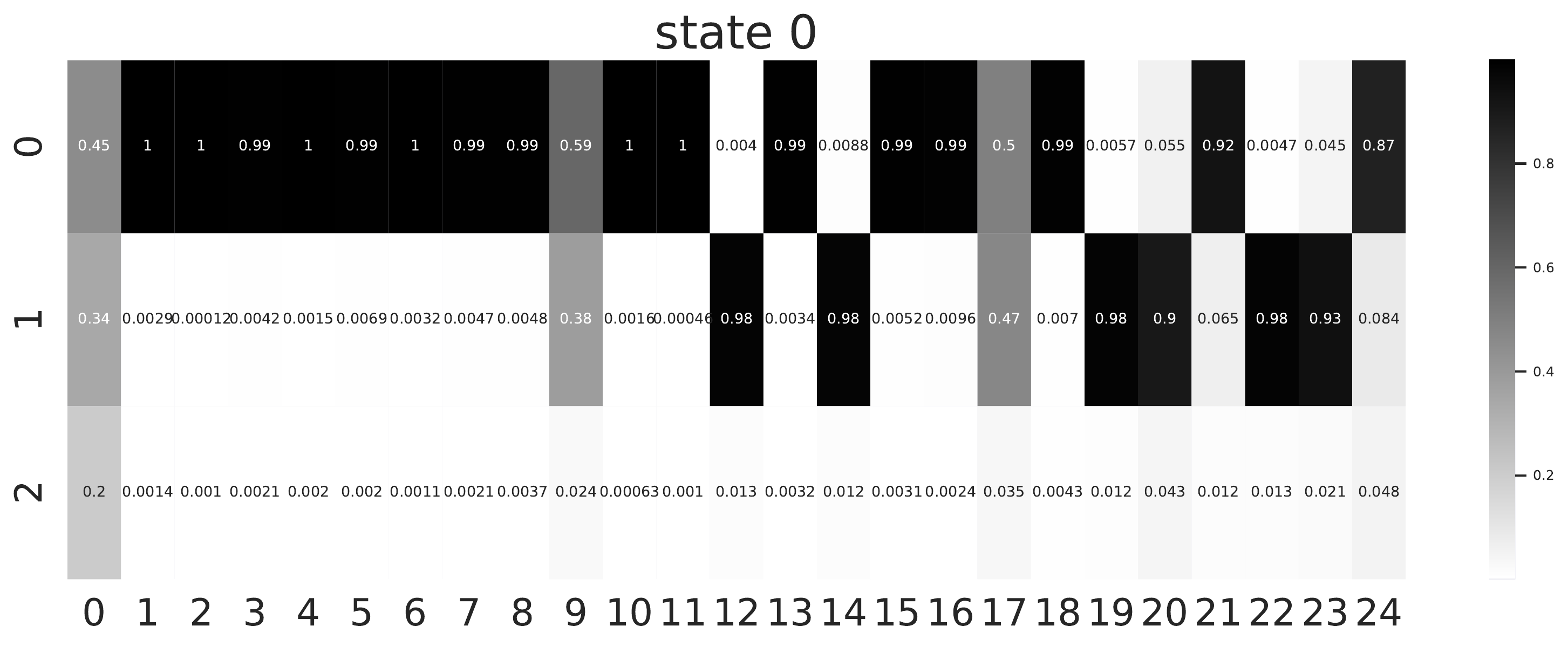}
    \includegraphics[width=0.6\linewidth]{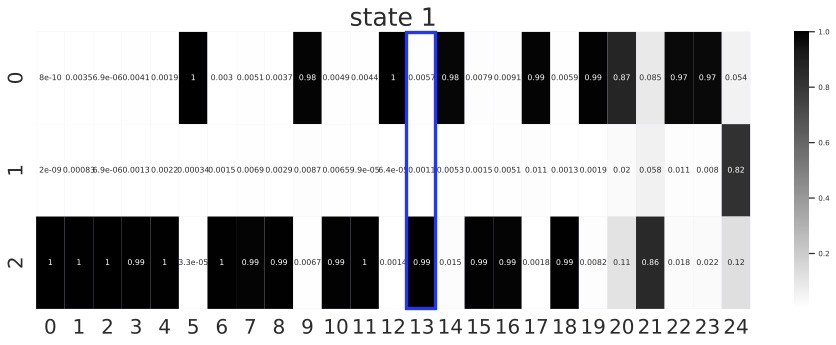}
    \includegraphics[width=0.6\linewidth]{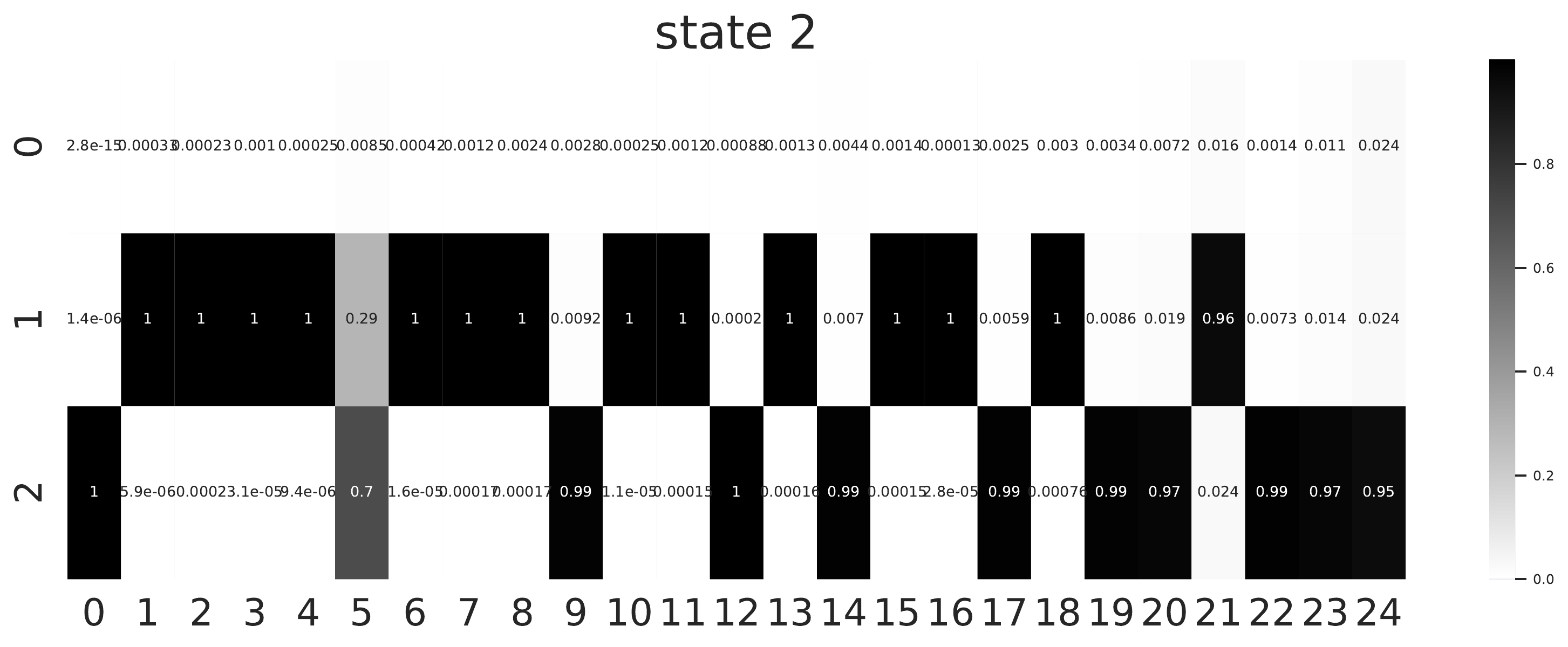}
    \caption{Visualization of decoders from source 1. Note the collapse happens at timestep 5, 9 and 17.}
    \label{fig:comb:source1}
\end{figure}

\begin{figure}[h]
    \centering
    \includegraphics[width=0.6\linewidth]{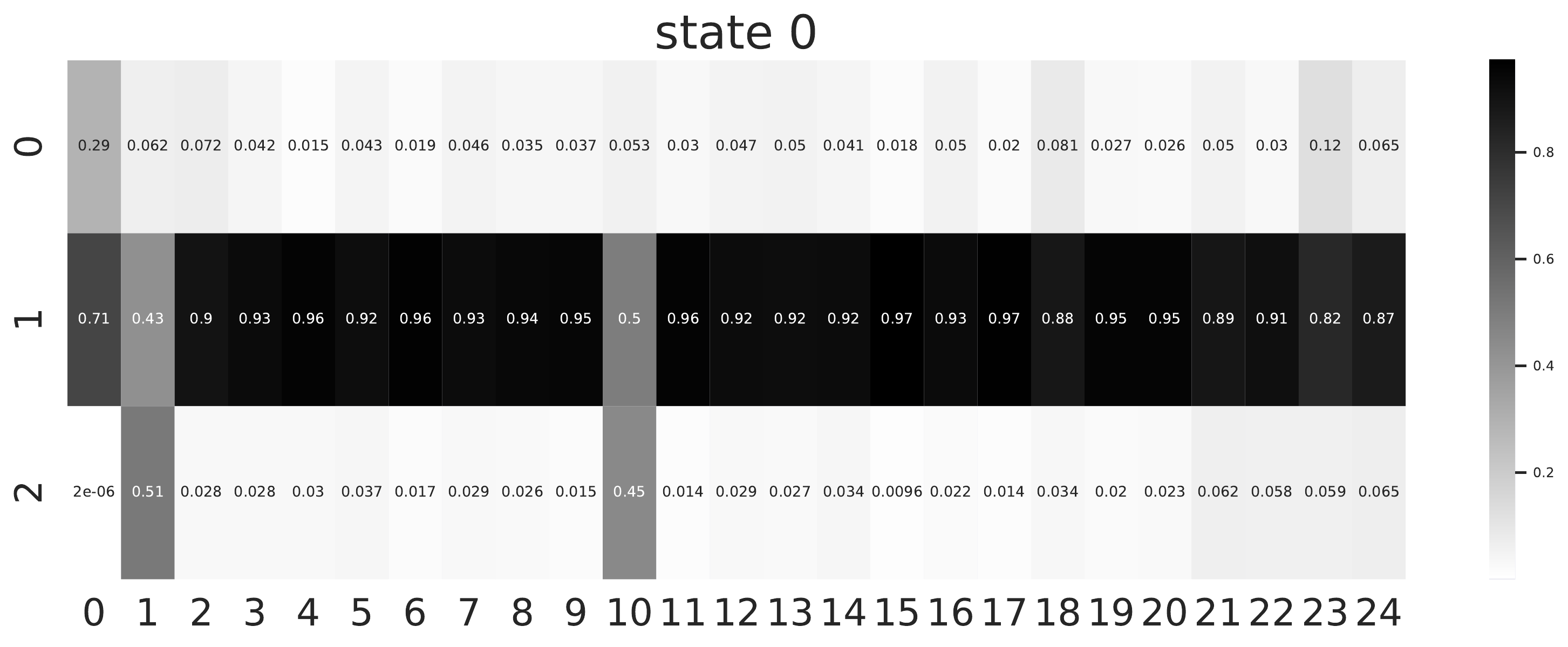}
    \includegraphics[width=0.6\linewidth]{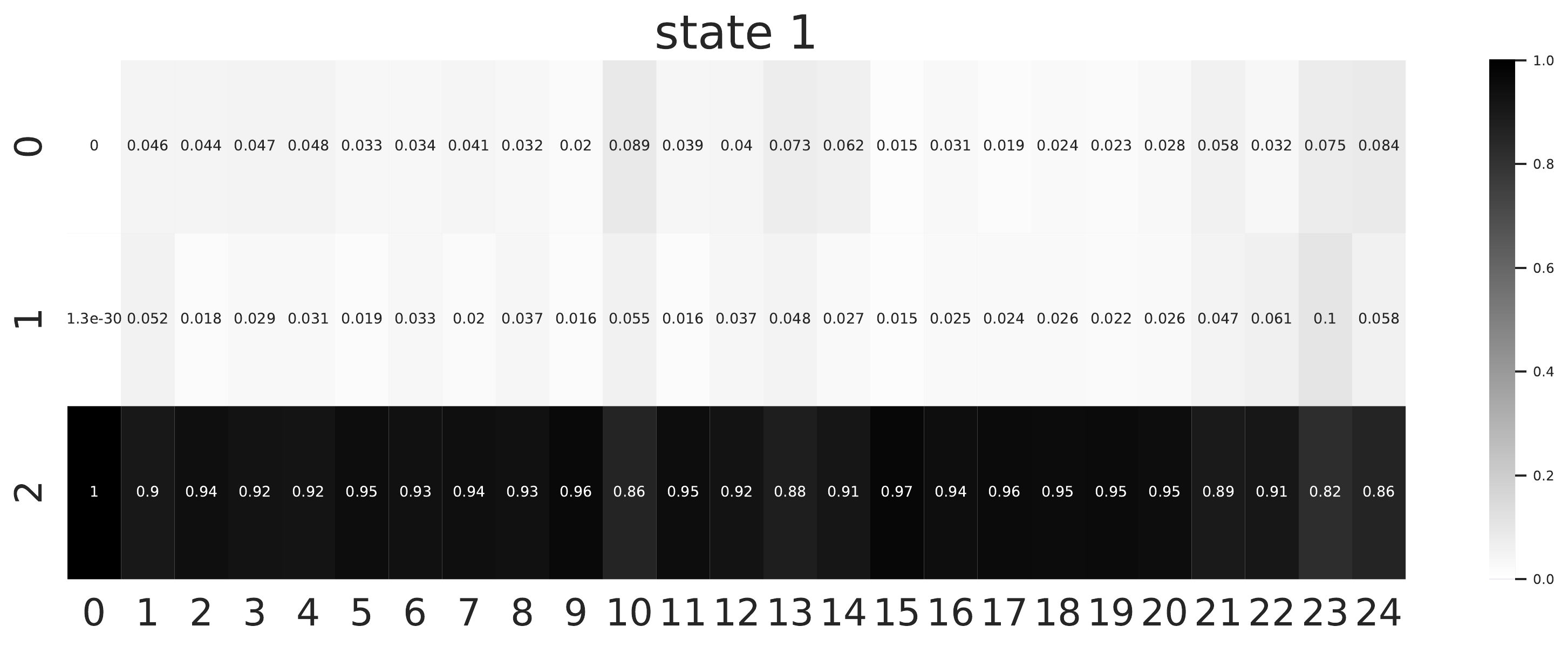}
    \includegraphics[width=0.6\linewidth]{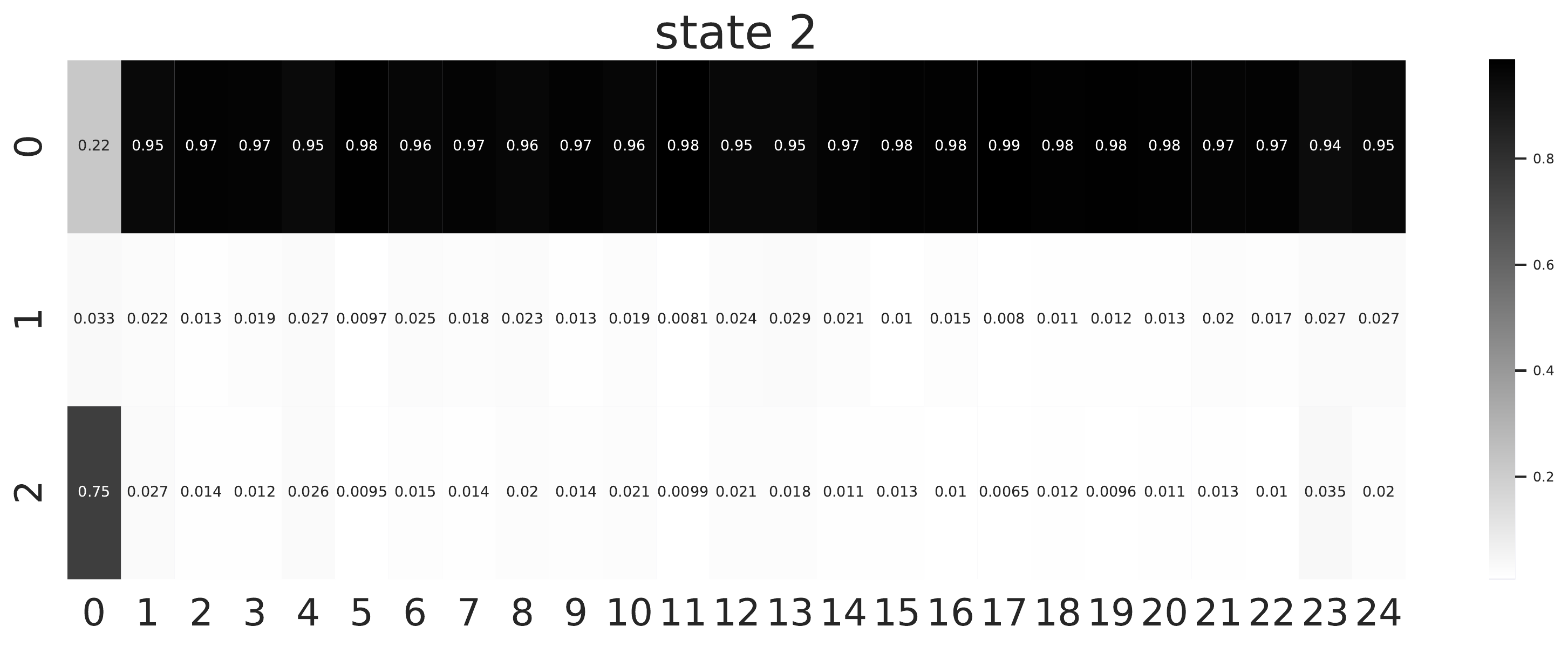}
    \caption{Visualization of decoders from source 2. Note the collapse happens at timestep 1 and 10.}
    \label{fig:comb:source2}
\end{figure}

\begin{figure}[h]
    \centering
    \includegraphics[width=0.6\linewidth]{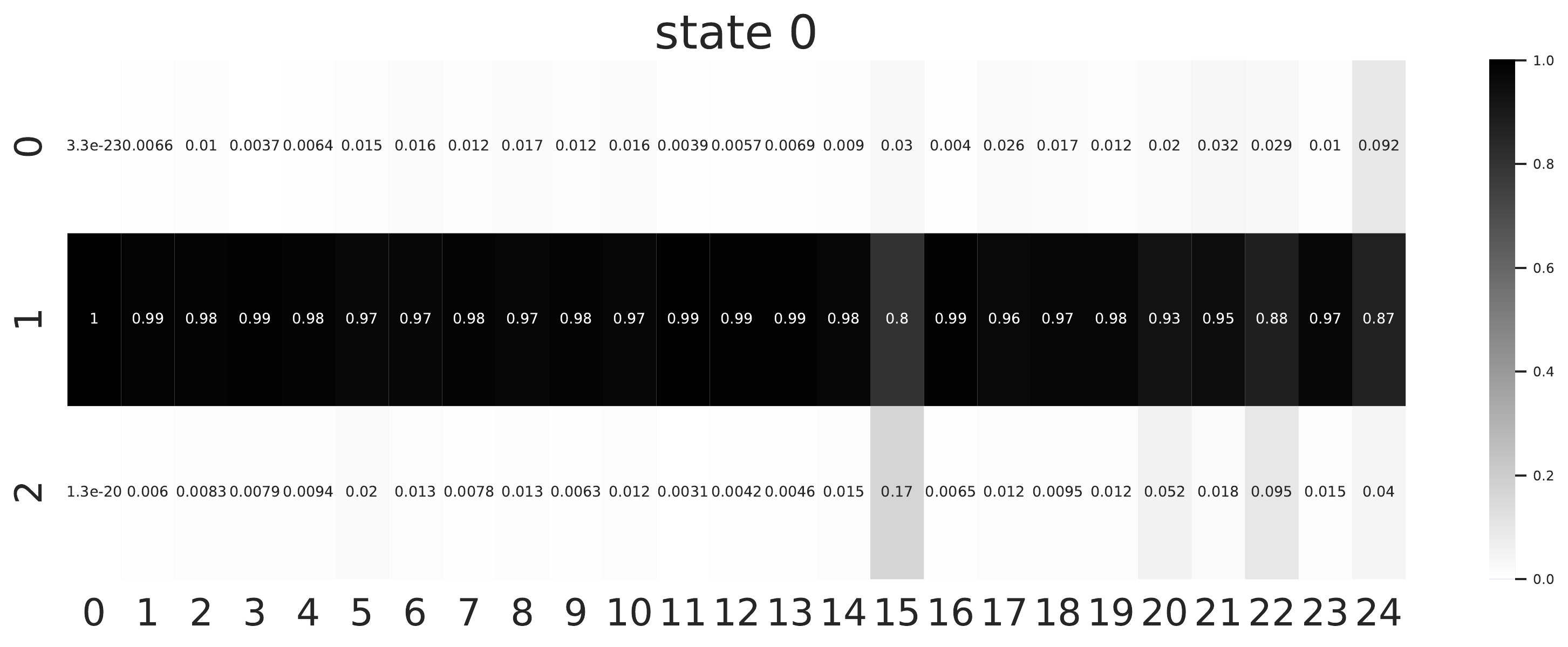}
    \includegraphics[width=0.6\linewidth]{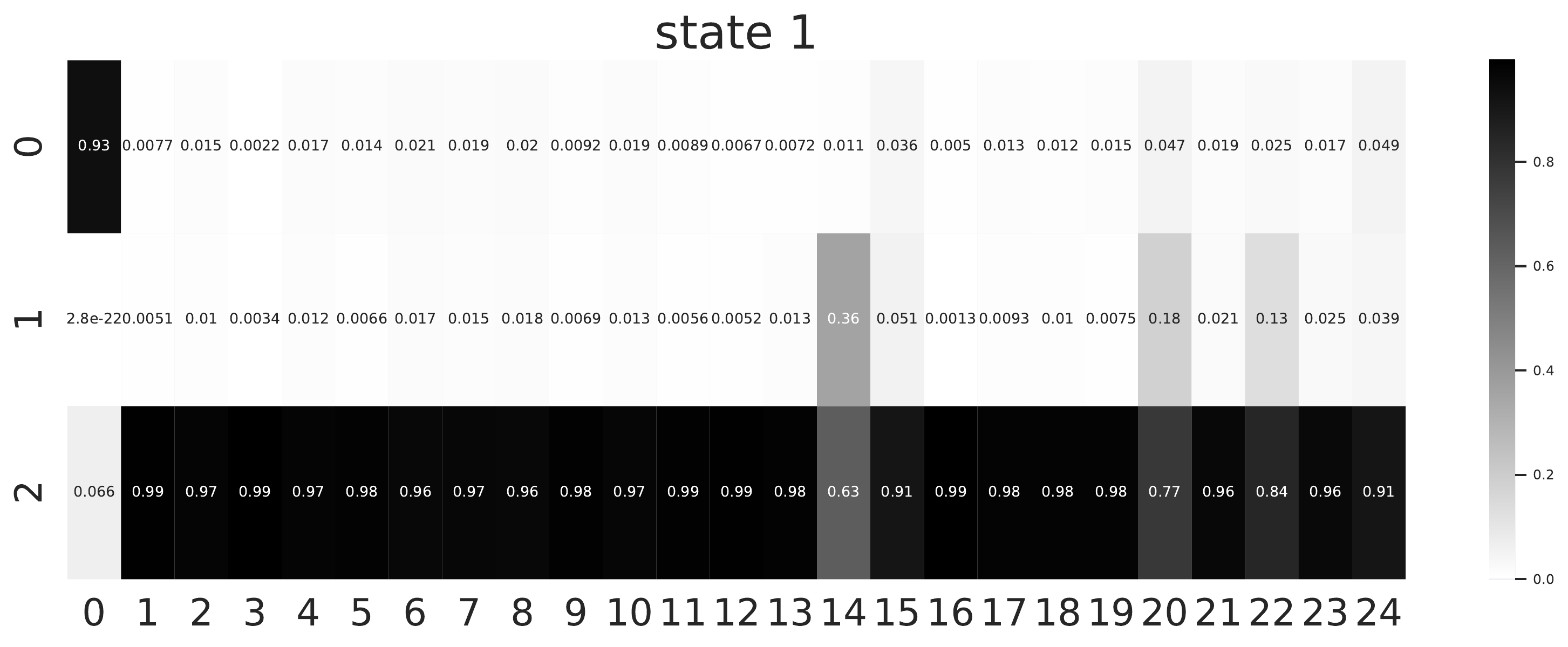}
    \includegraphics[width=0.6\linewidth]{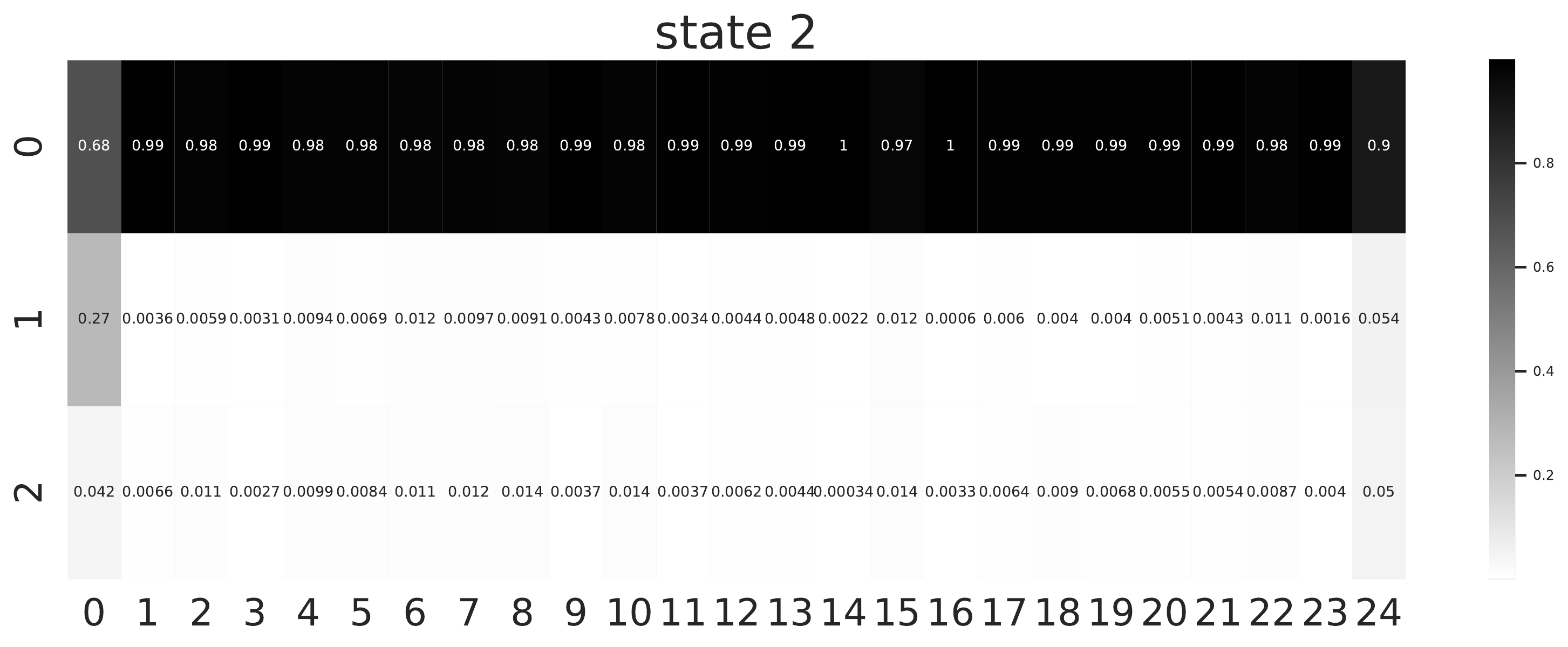}
    \caption{Visualization of decoders from source 3. Note the collapse happens at timestep 14 and 15.}
    \label{fig:comb:source3}
\end{figure}

\begin{figure}[h]
    \centering
    \includegraphics[width=0.6\linewidth]{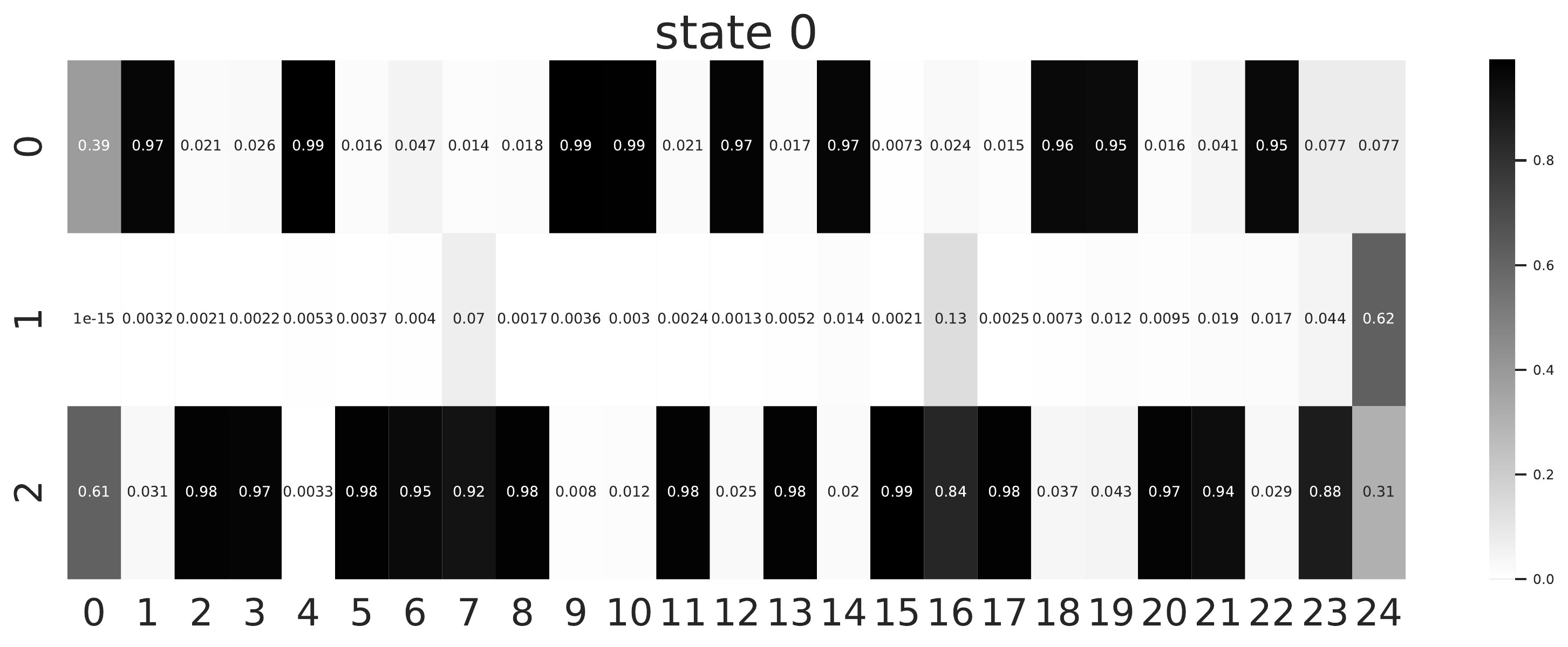}
    \includegraphics[width=0.6\linewidth]{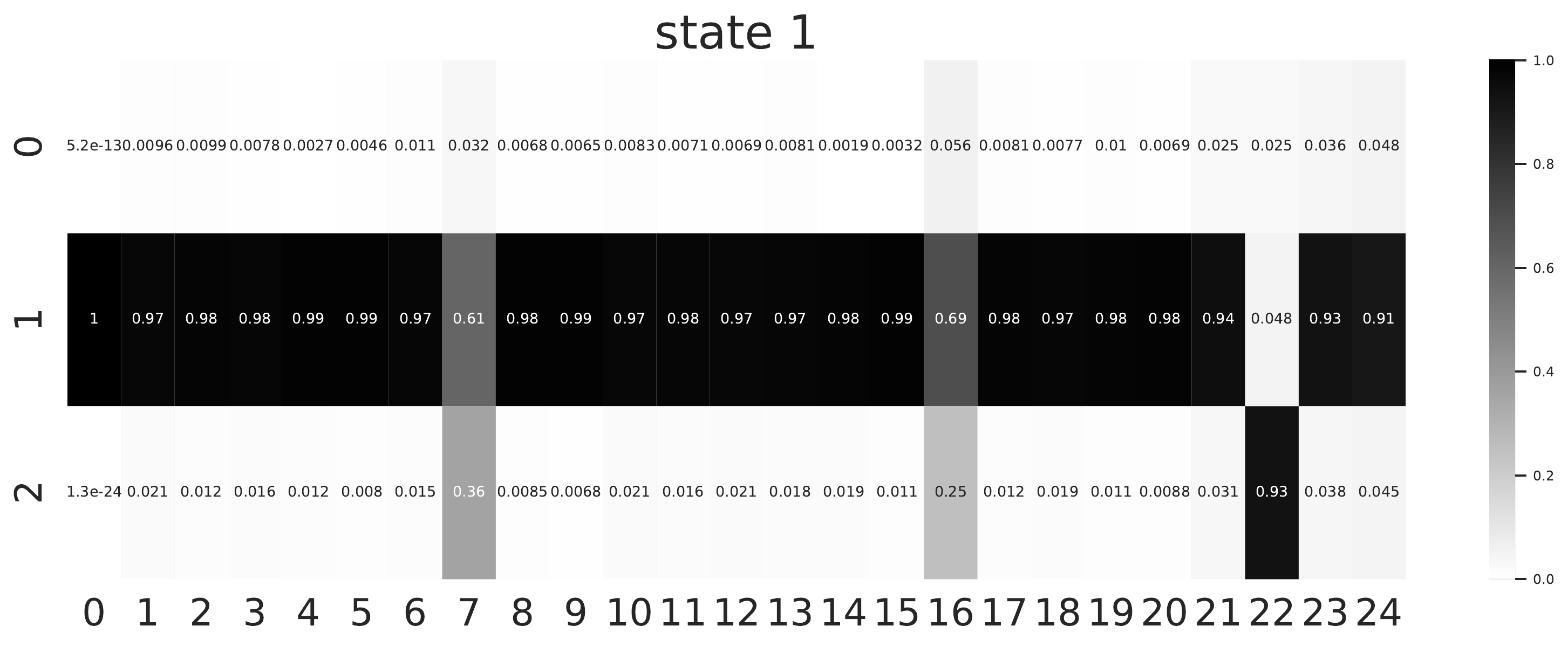}
    \includegraphics[width=0.6\linewidth]{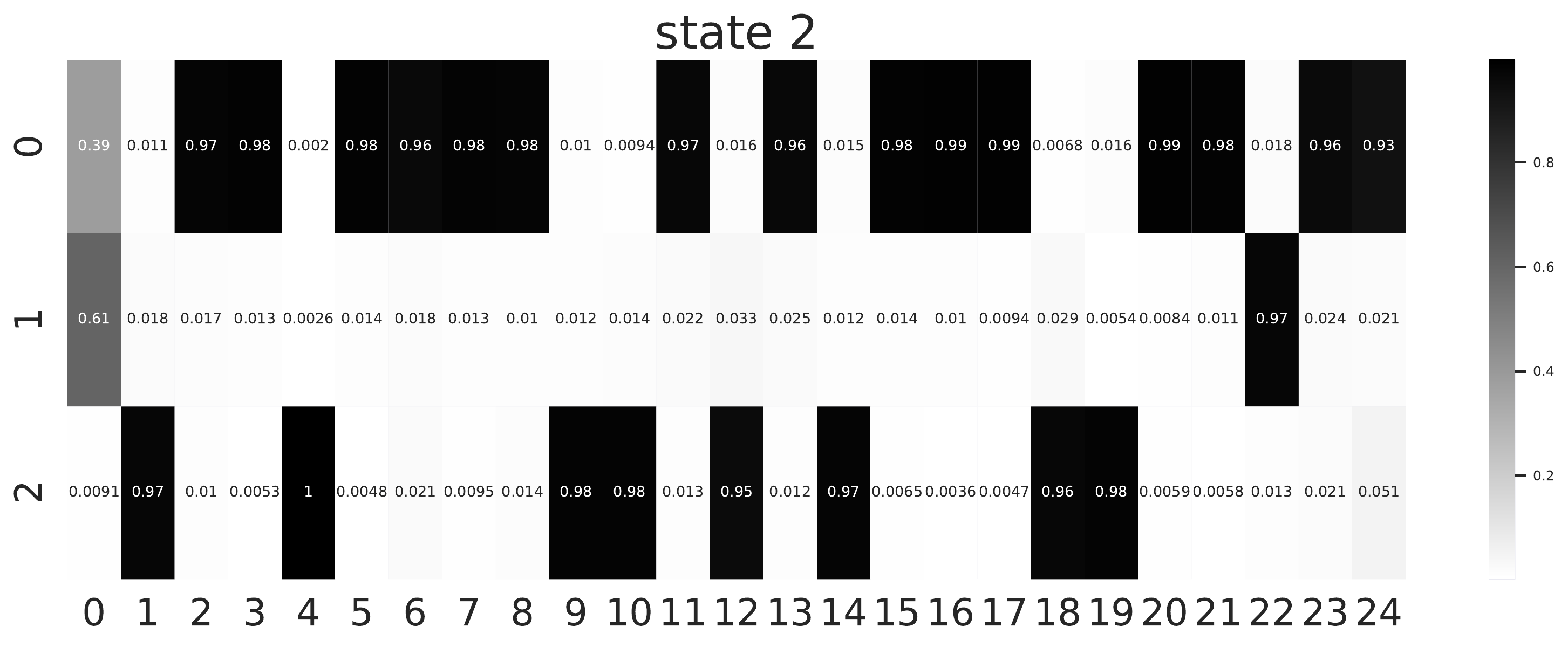}
    \caption{Visualization of decoders from source 4. Note the collapse happens at timestep 7, 16, 24.}
    \label{fig:comb:source4}
\end{figure}

\begin{figure}[h]
    \centering
    \includegraphics[width=0.6\linewidth]{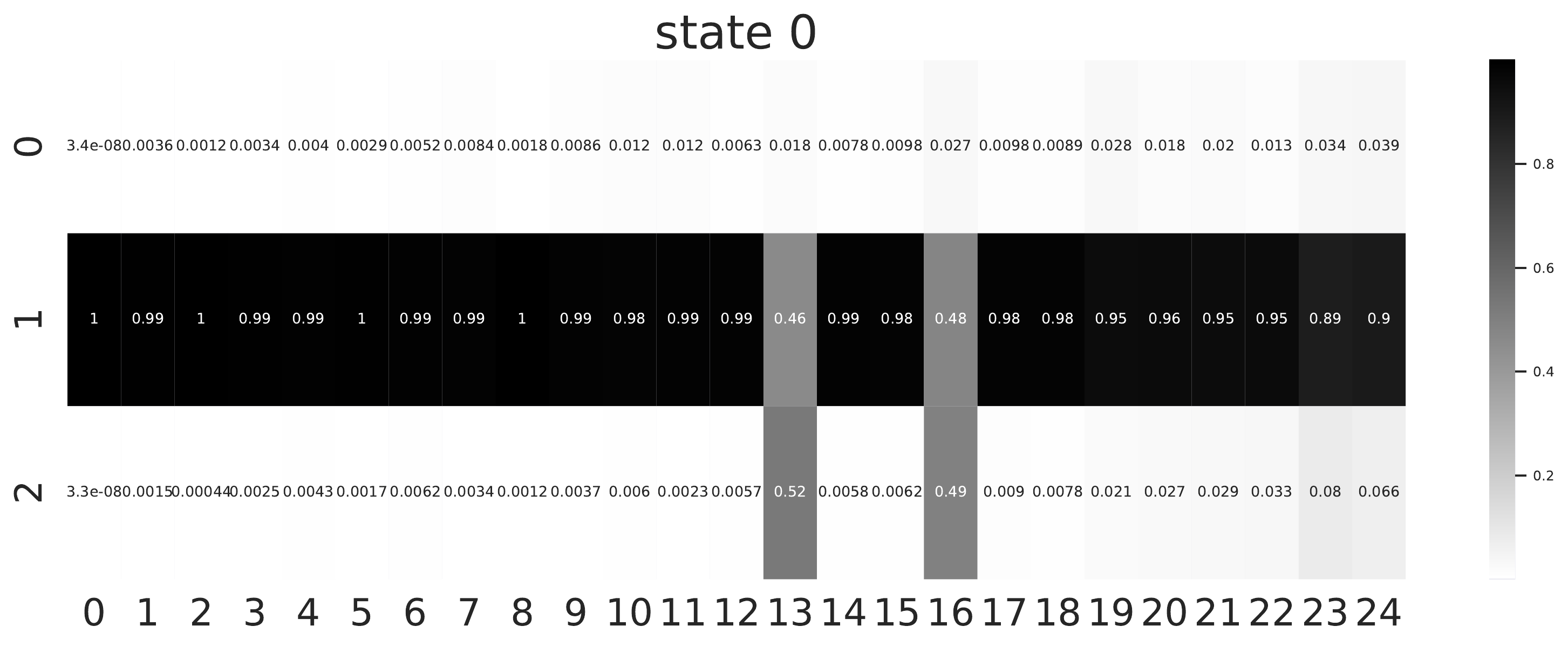}
    \includegraphics[width=0.6\linewidth]{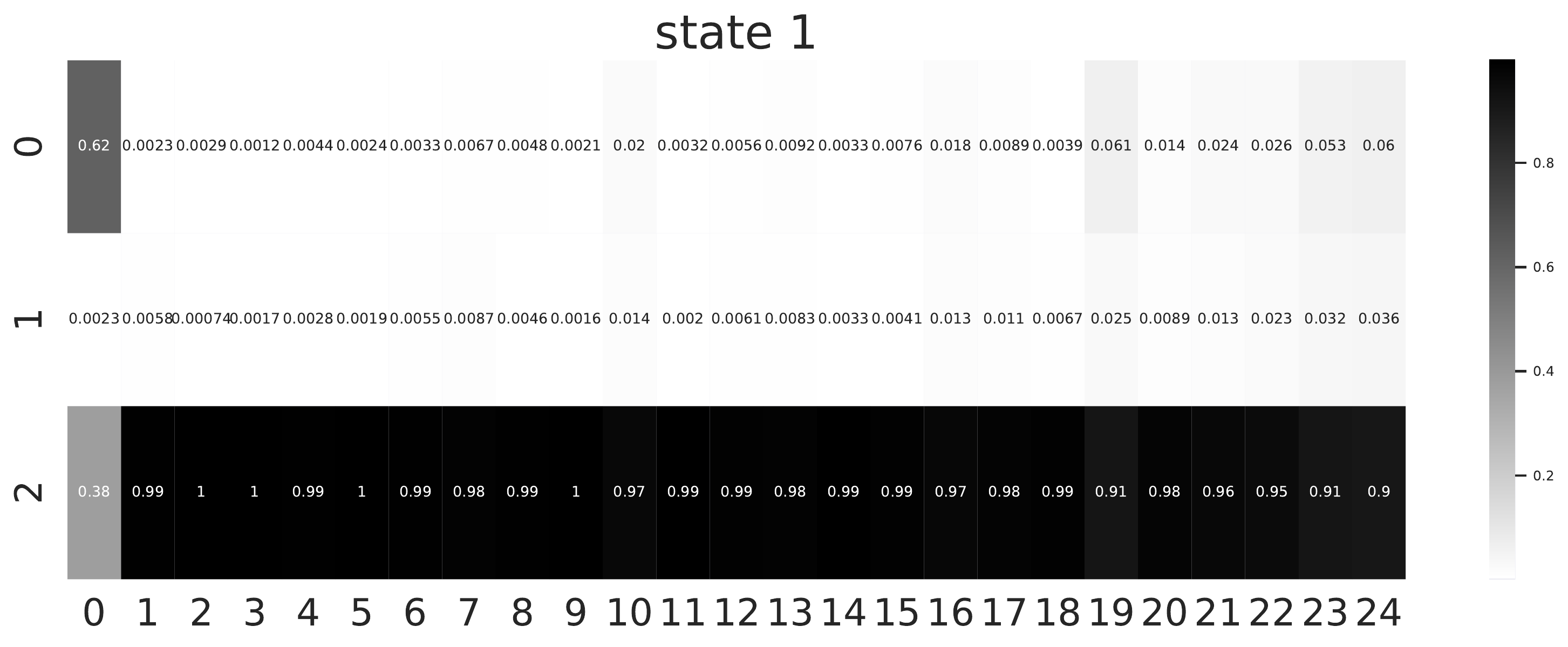}
    \includegraphics[width=0.6\linewidth]{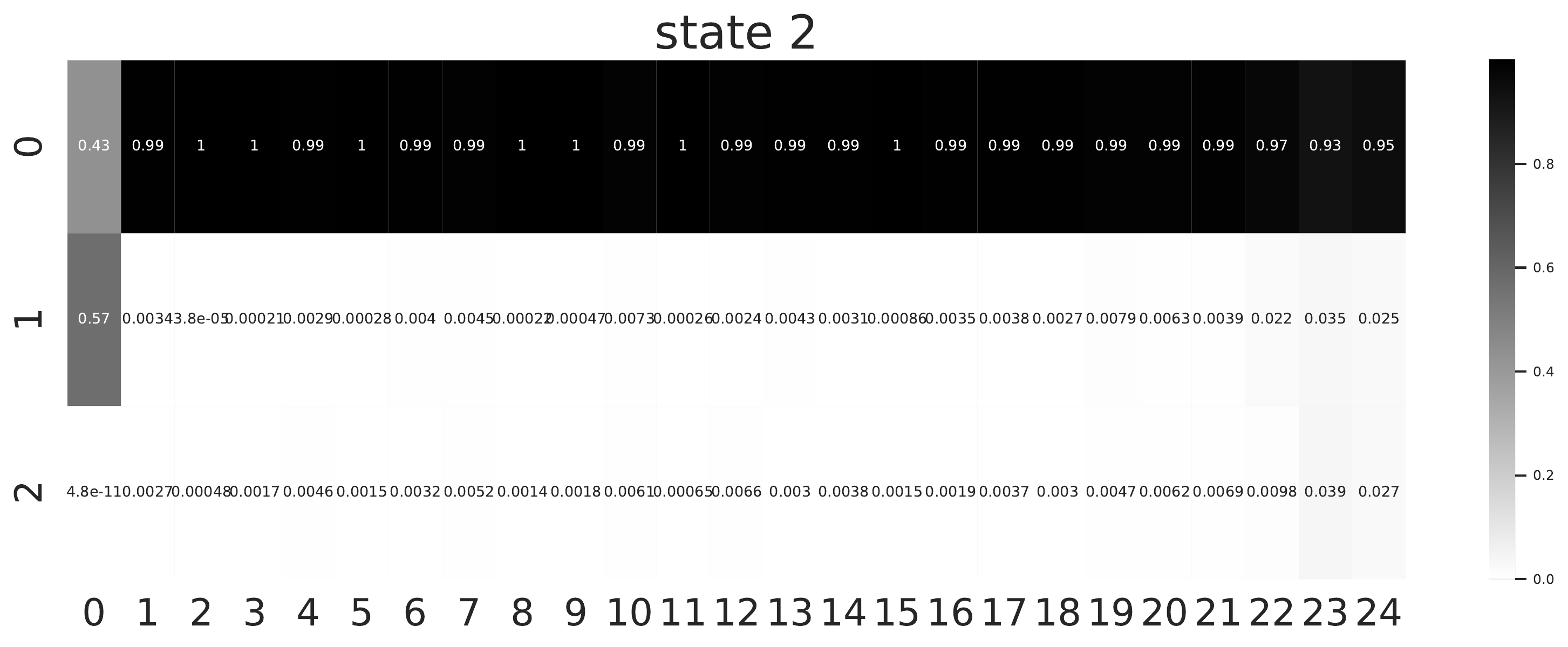}
    \caption{Visualization of decoders from source 3. Note the collapse happens at timestep 13 and 16.}
    \label{fig:comb:source5}
\end{figure}

\begin{figure}[h]
    \centering
    \includegraphics[width=0.6\linewidth]{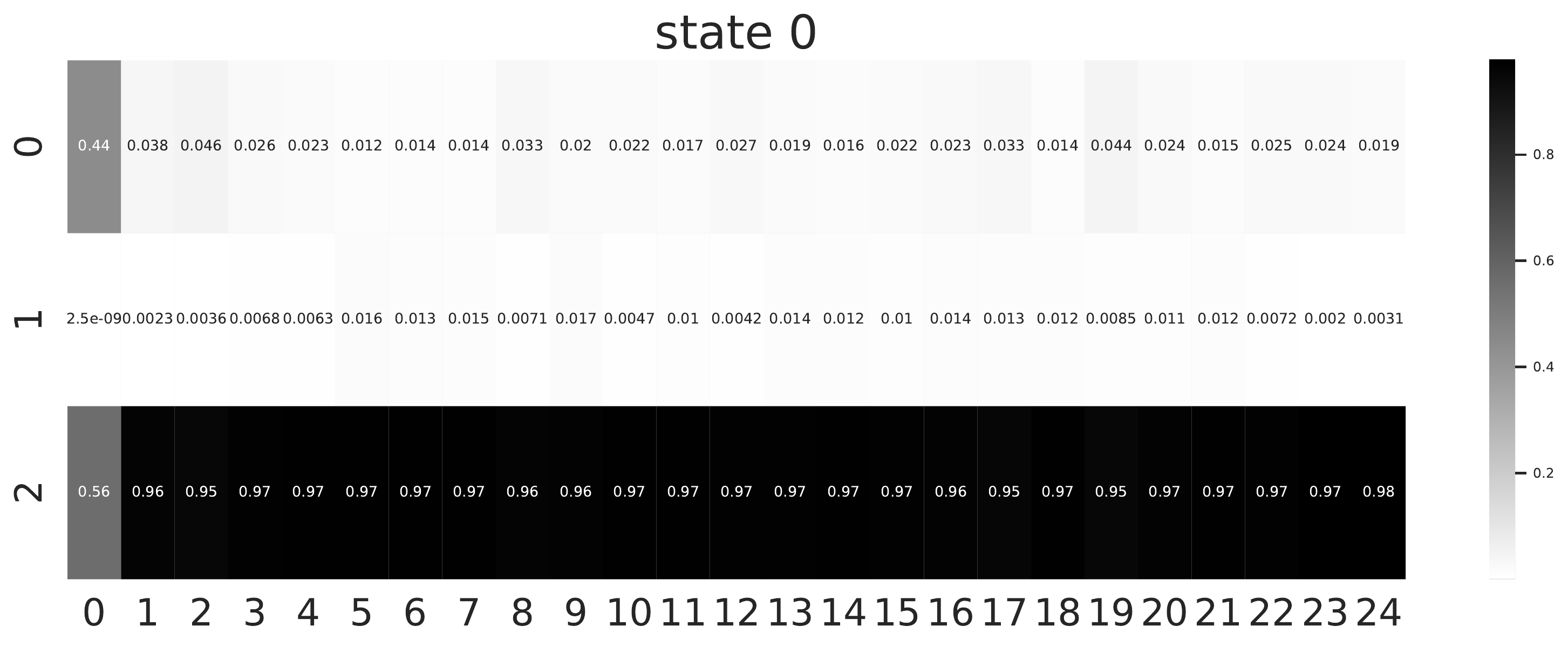}
    \includegraphics[width=0.6\linewidth]{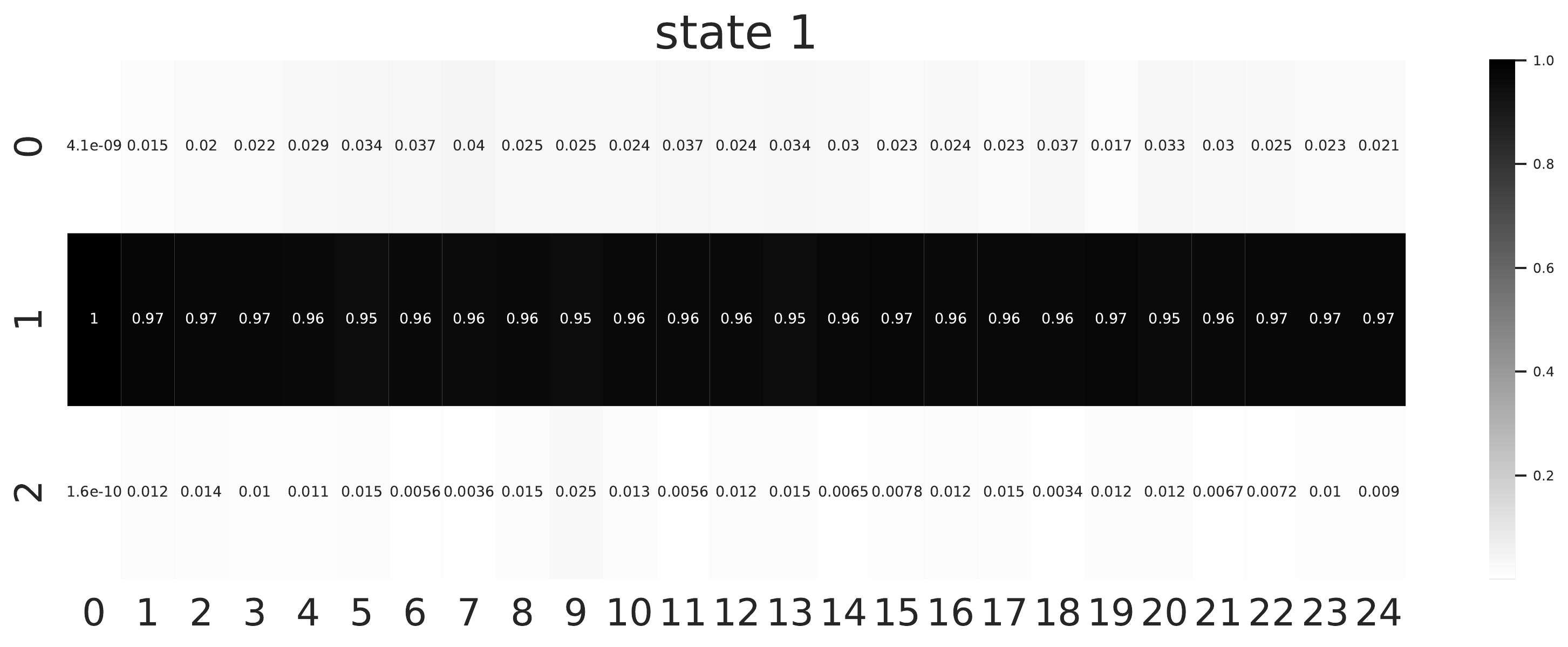}
    \includegraphics[width=0.6\linewidth]{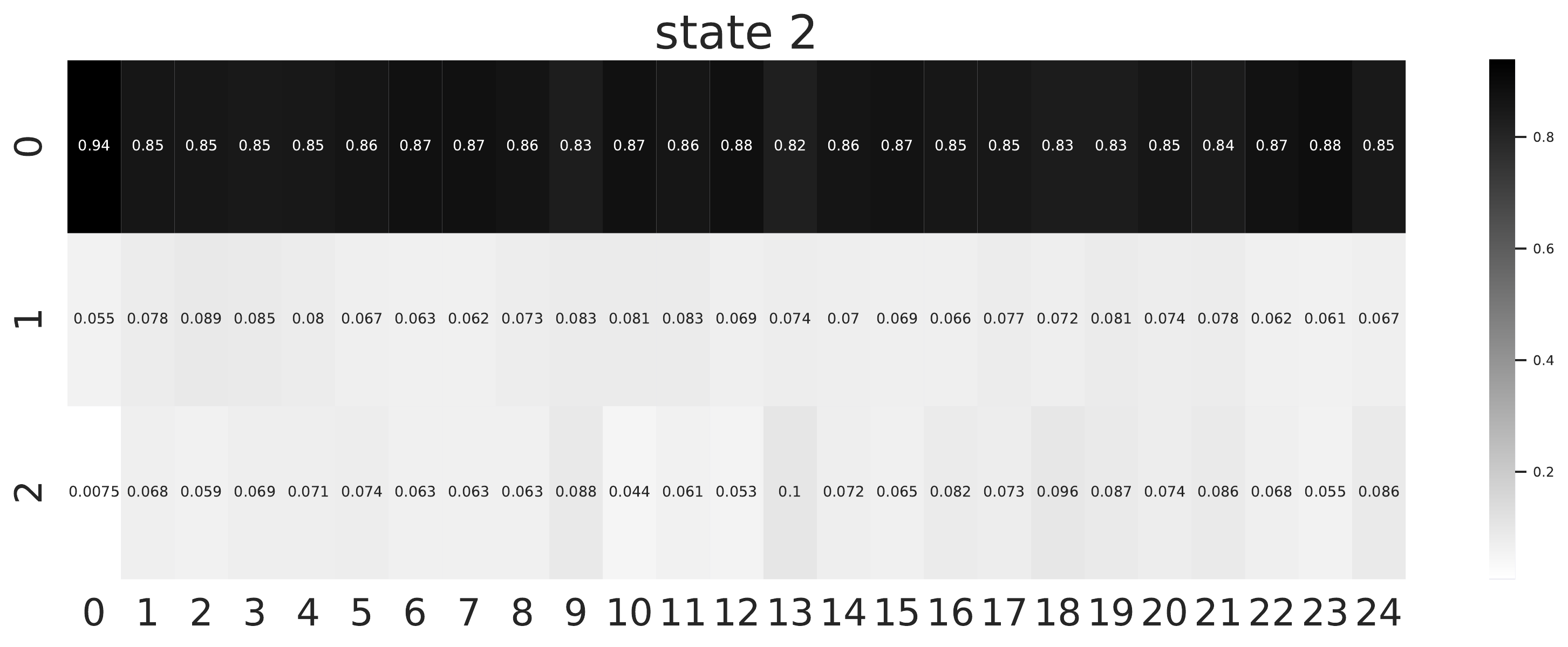}
    \caption{Visualization of decoders from \ouralgo}
    \label{fig:comb:online}
\end{figure}

\begin{figure}[h]
    \centering
    \includegraphics[width=0.6\linewidth]{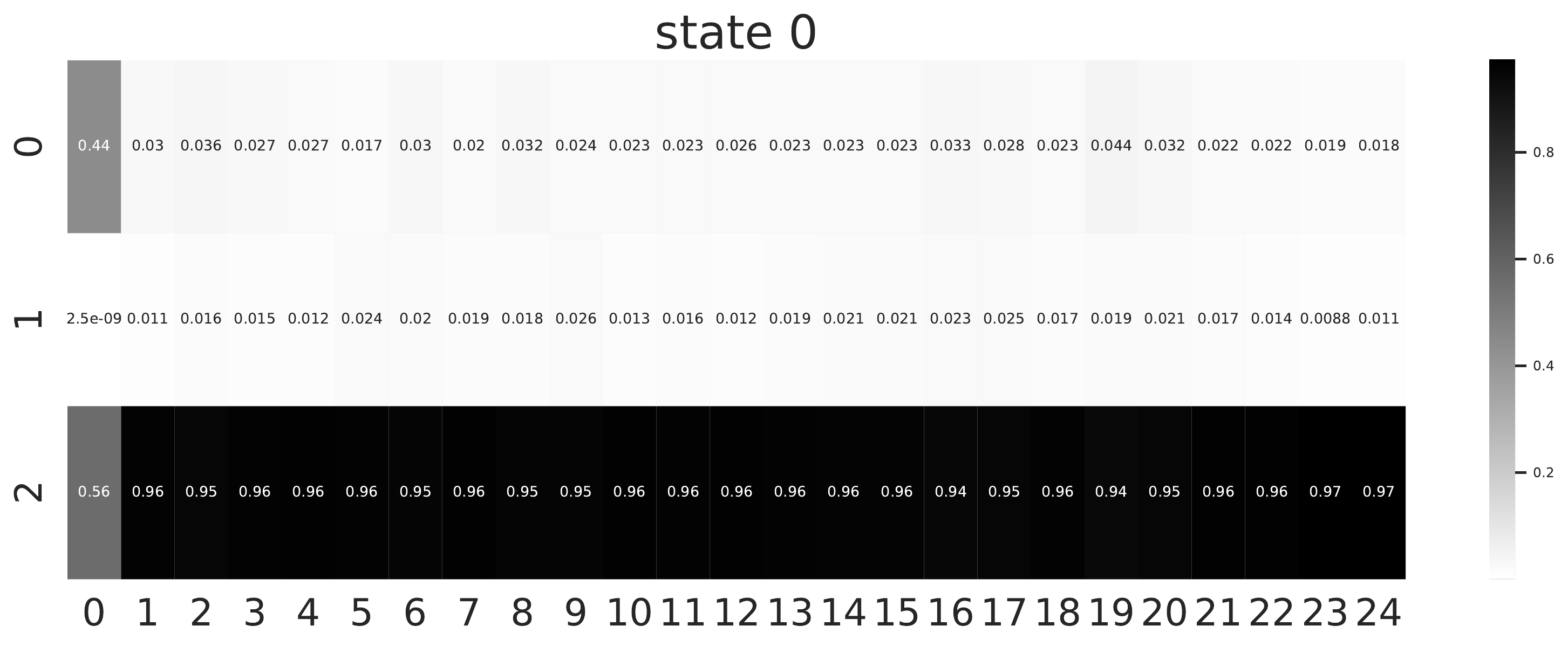}
    \includegraphics[width=0.6\linewidth]{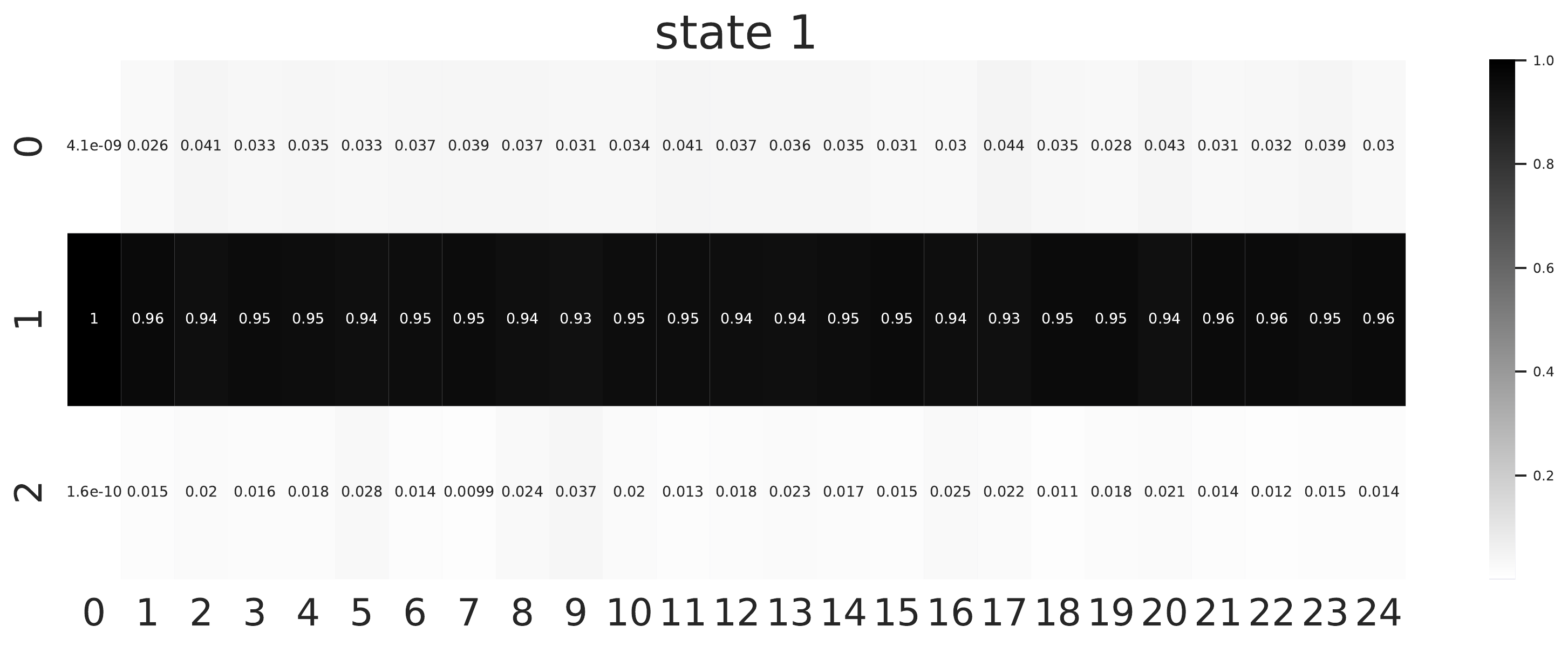}
    \includegraphics[width=0.6\linewidth]{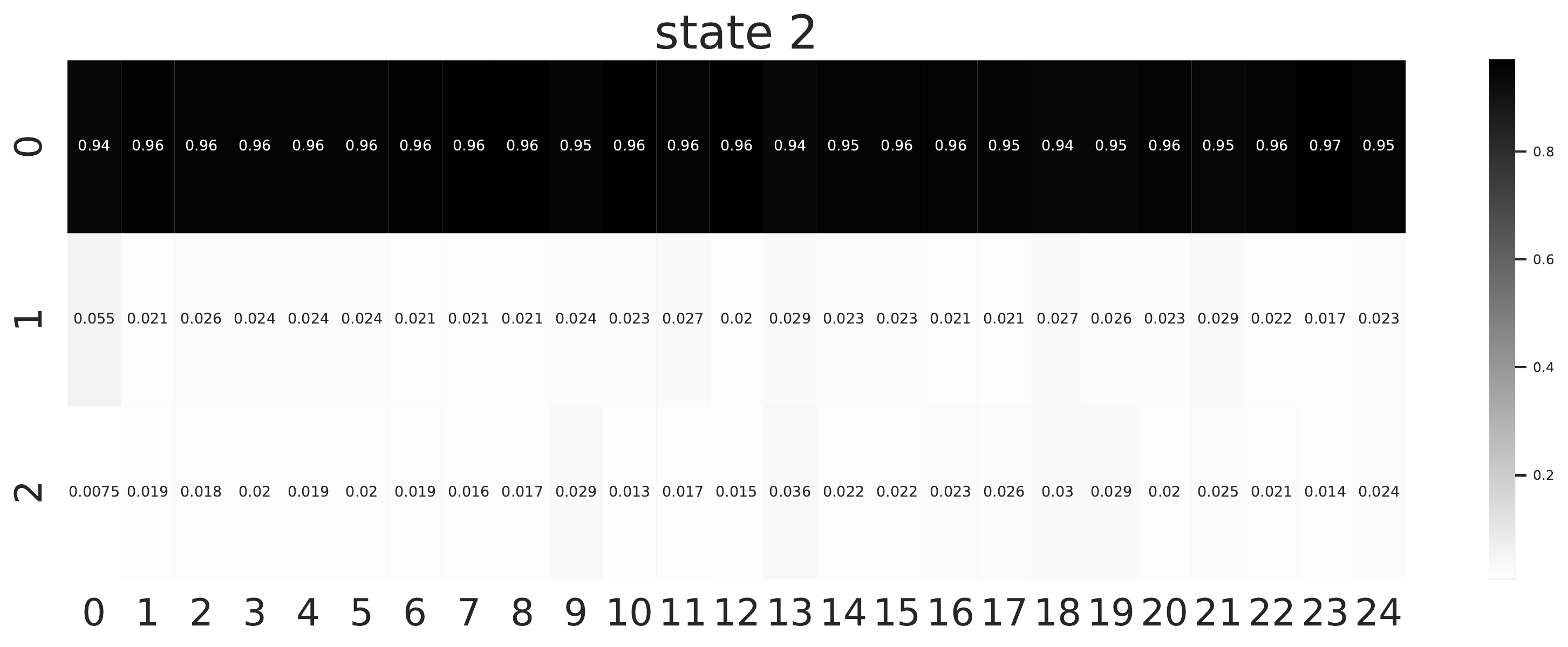}
    \caption{Visualization of decoders from \ouralgg}
    \label{fig:comb:generative}
\end{figure}

\begin{figure}[h]
    \centering
    \includegraphics[width=0.6\linewidth]{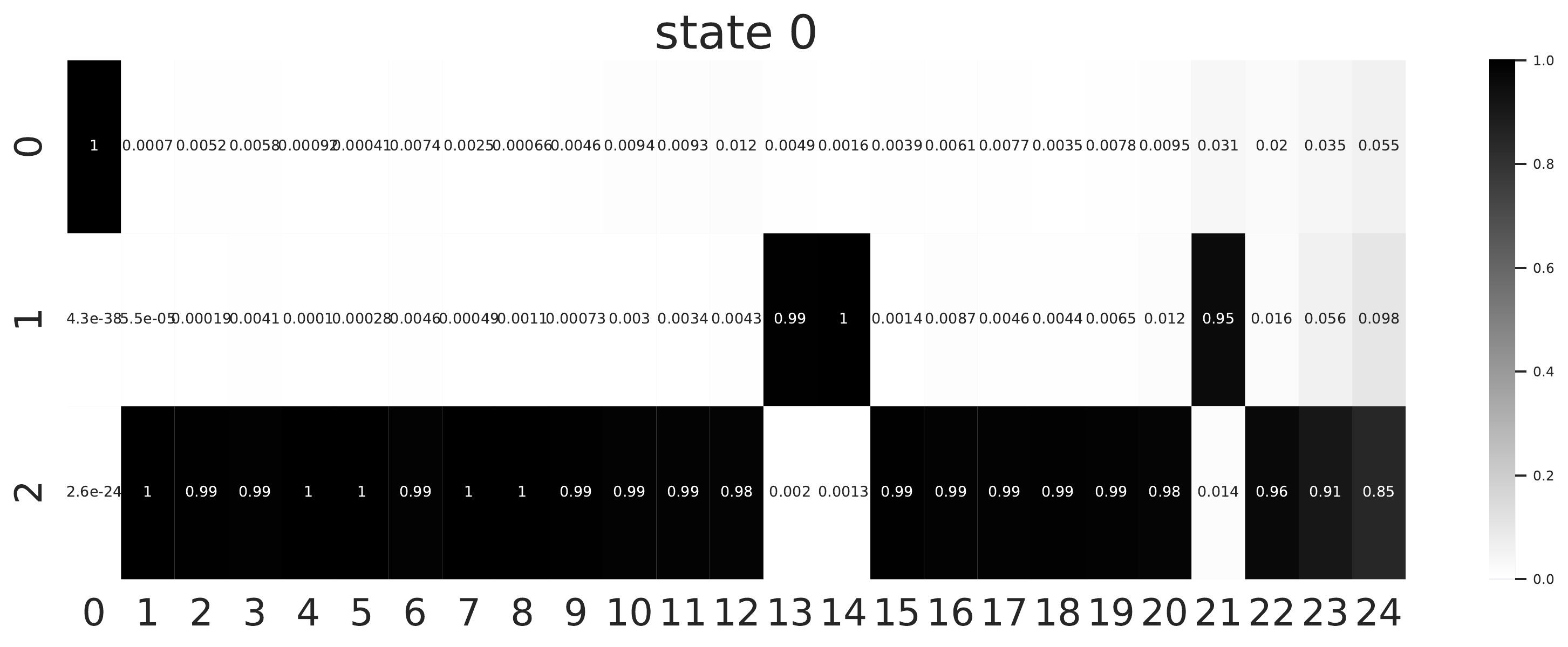}
    \includegraphics[width=0.6\linewidth]{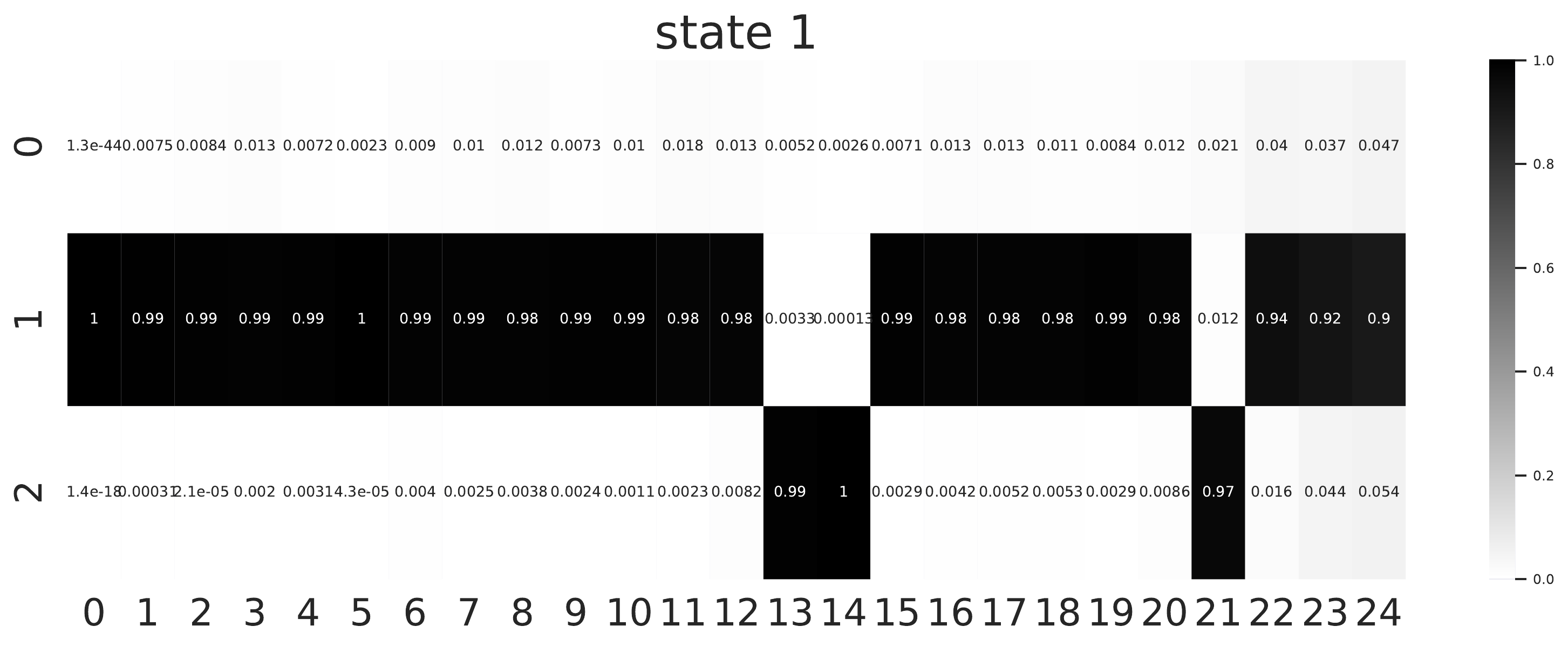}
    \includegraphics[width=0.6\linewidth]{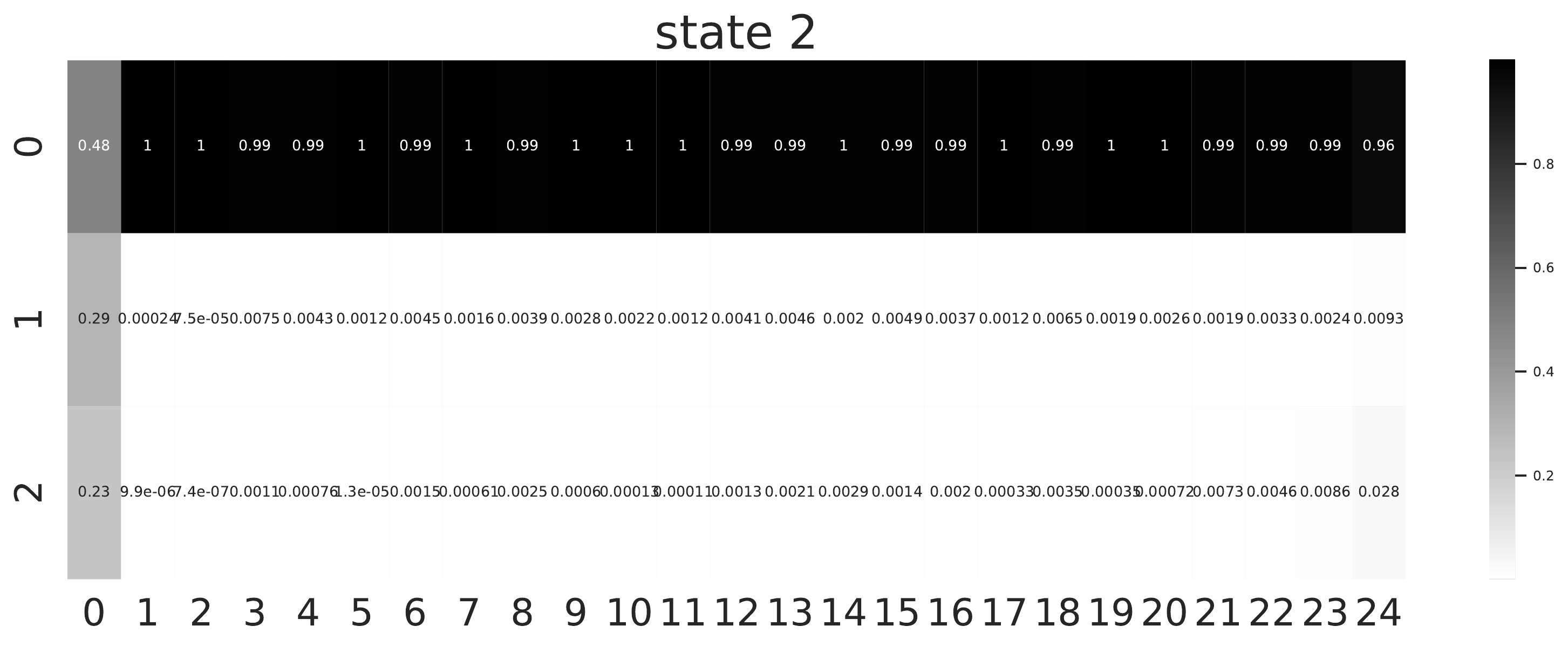}
    \caption{Visualization of decoders from running \algname on target.}
    \label{fig:comb:briee}
\end{figure}

\subsubsection{Visualizations from the feature coverage experiment}
We record the visualization of the 2 sources from Fig.~\ref{fig:part:source1} and Fig.~\ref{fig:part:source2}; \ouralgo in Fig.\ref{fig:part:online}; \ouralgg in Fig.~\ref{fig:part:generative}; running \algname on target in Fig.~\ref{fig:part:briee}. Note that the features collapse at some timesteps in Fig.~\ref{fig:part:generative} and Fig.~\ref{fig:part:briee}, but this is acceptable because the optimal actions at those timesteps are the same for the collapsed states.

\begin{figure}[h]
    \centering
    \includegraphics[width=0.6\linewidth]{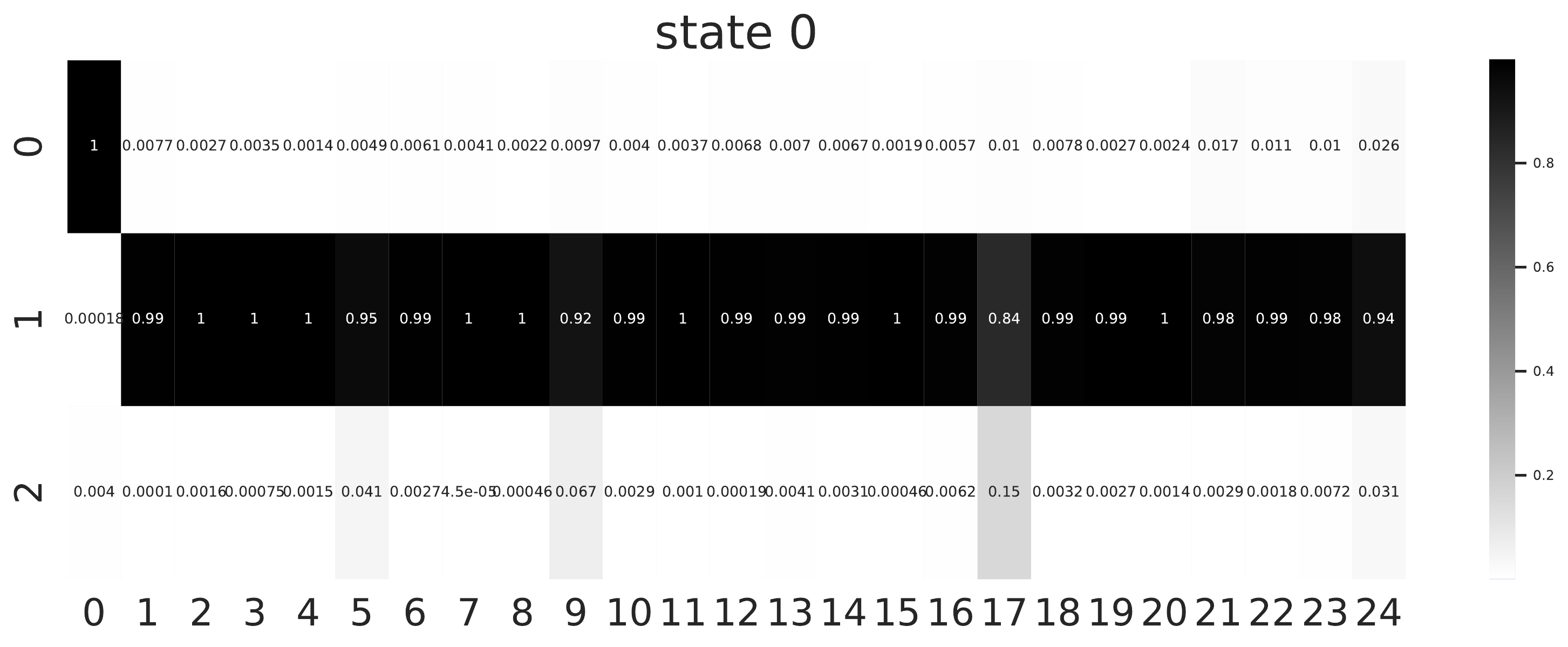}
    \includegraphics[width=0.6\linewidth]{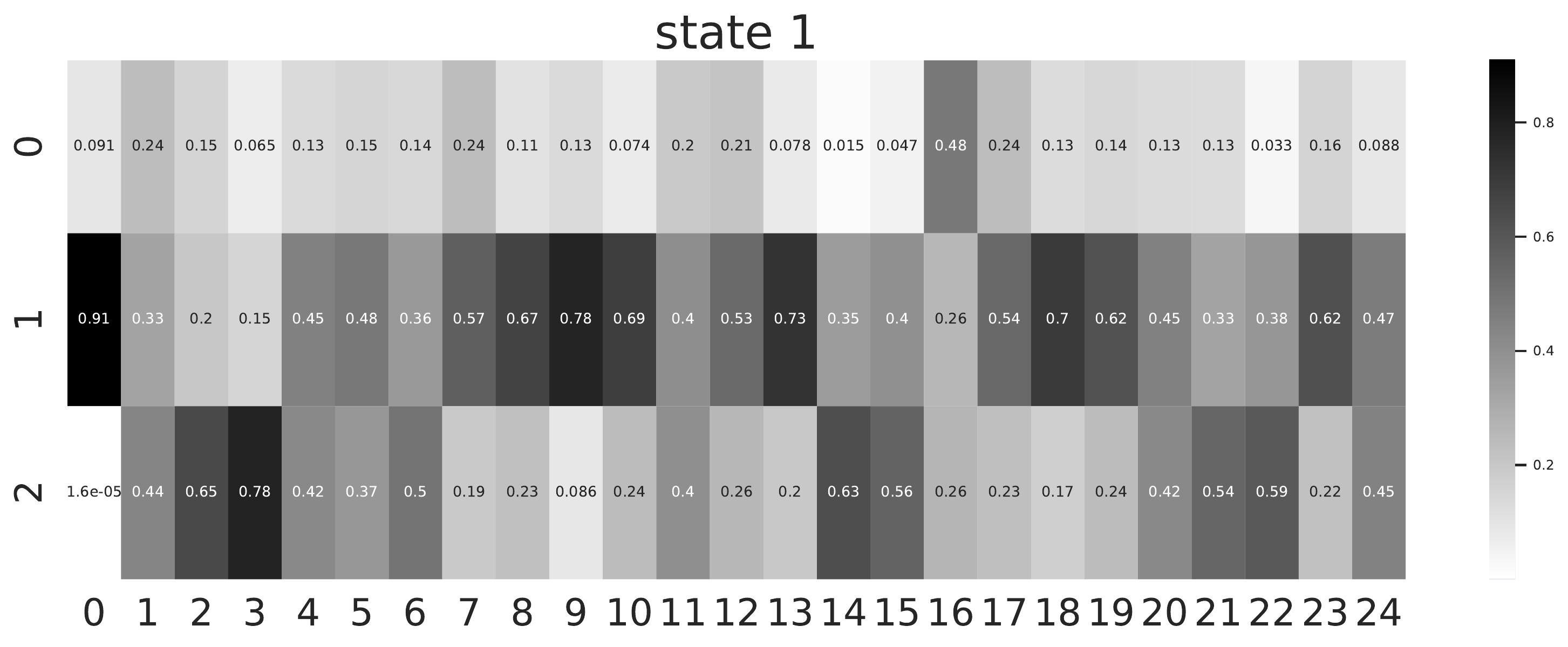}
    \includegraphics[width=0.6\linewidth]{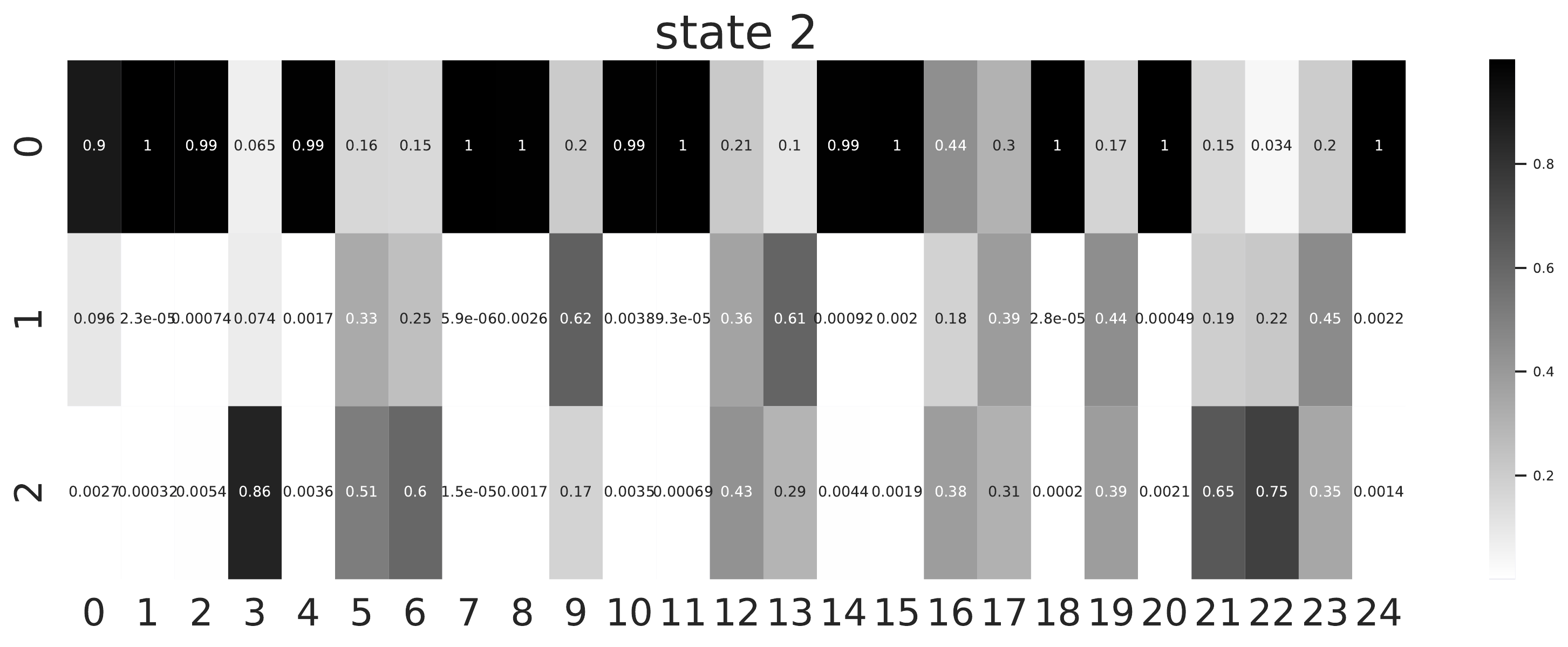}
    \caption{Visualization of decoders from source environment 1.}
    \label{fig:part:source1}
\end{figure}

\begin{figure}[h]
    \centering
    \includegraphics[width=0.6\linewidth]{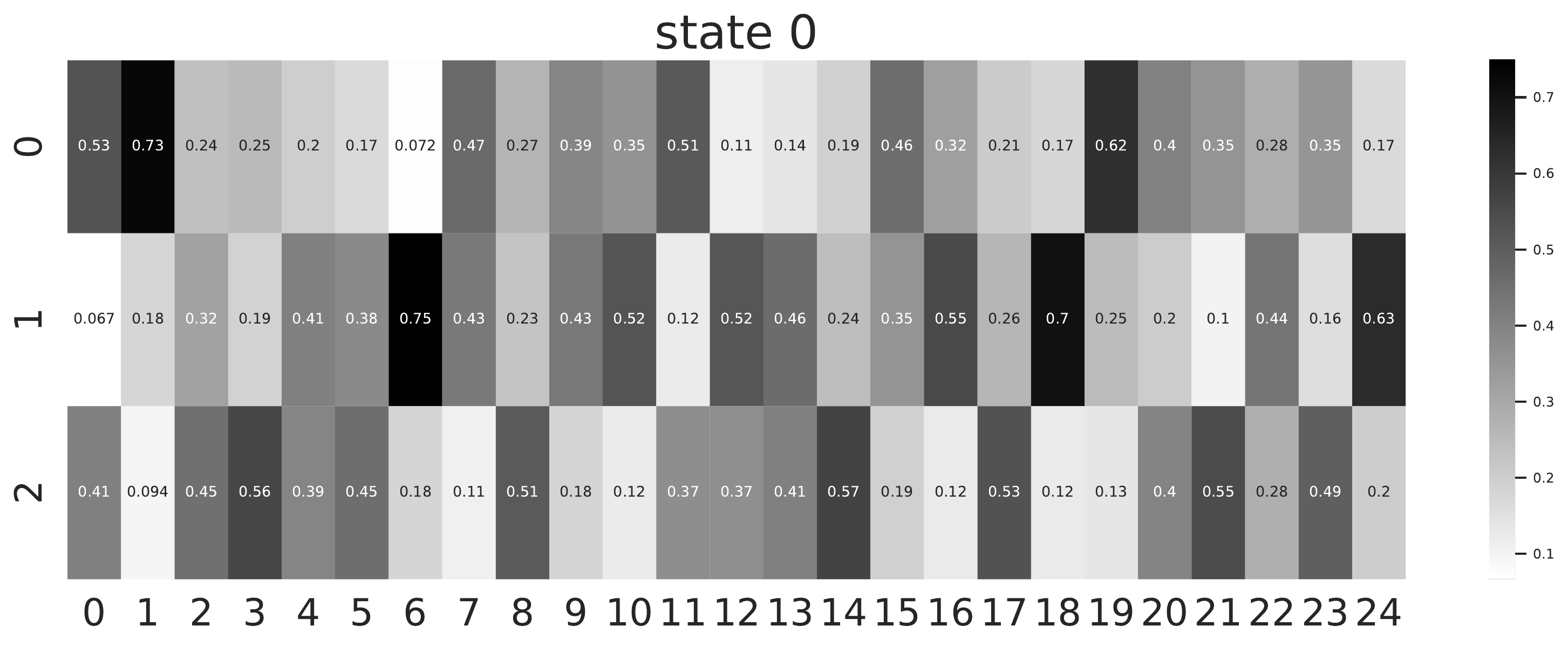}
    \includegraphics[width=0.6\linewidth]{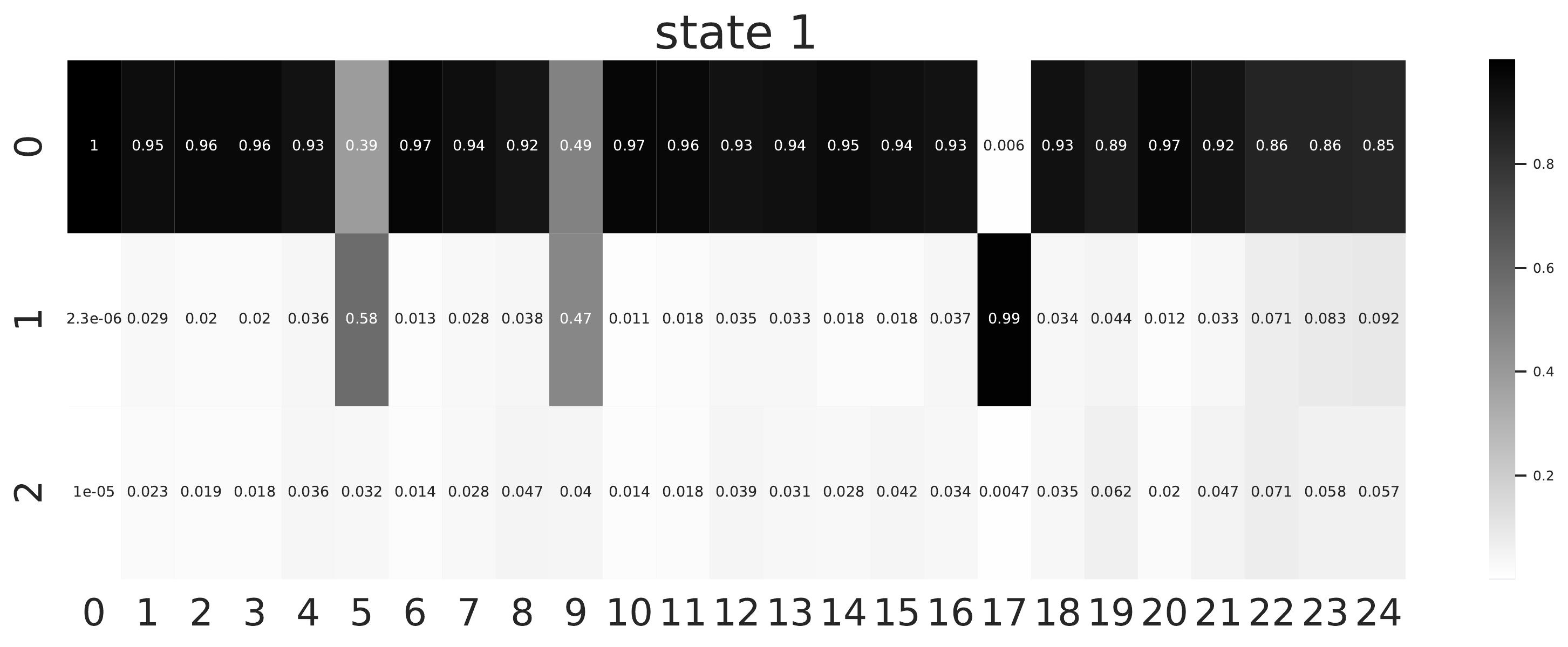}
    \includegraphics[width=0.6\linewidth]{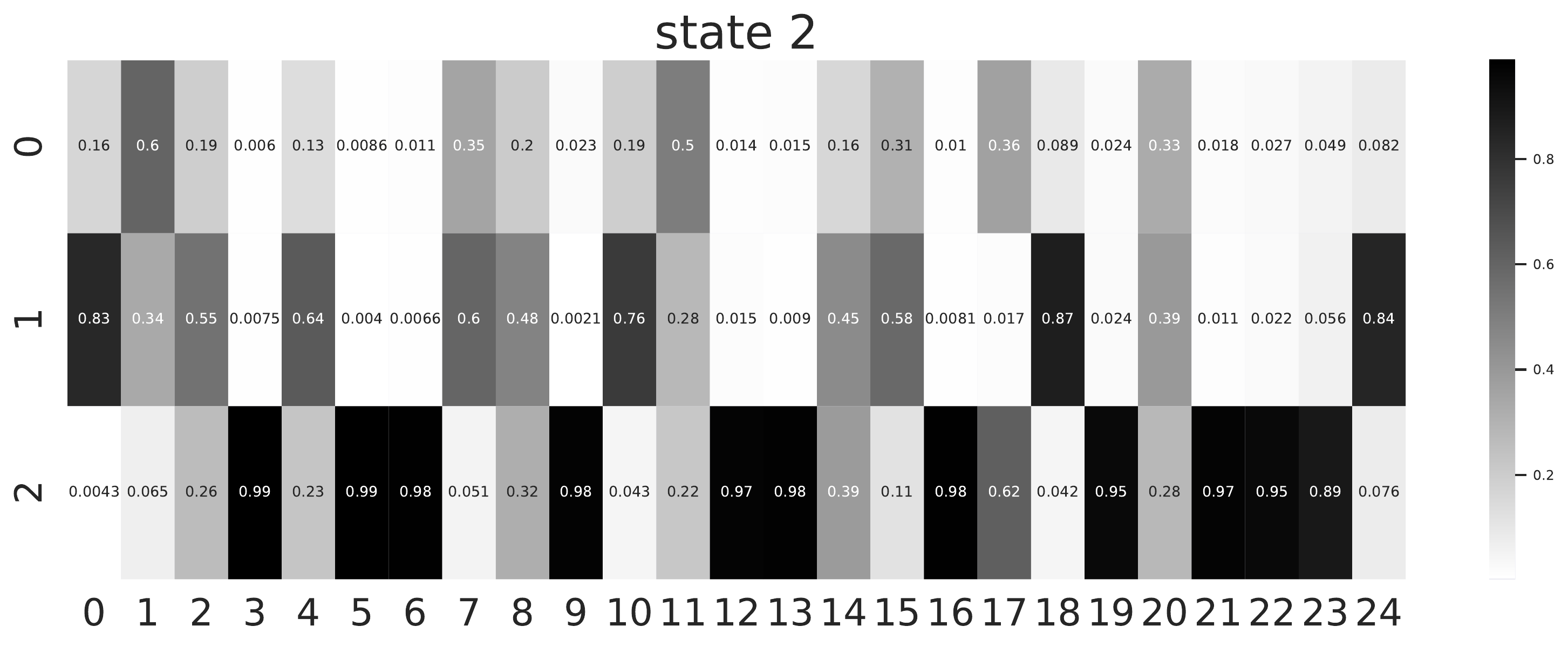}
    \caption{Visualization of decoders from source environment 2.}
    \label{fig:part:source2}
\end{figure}

\begin{figure}[h]
    \centering
    \includegraphics[width=0.6\linewidth]{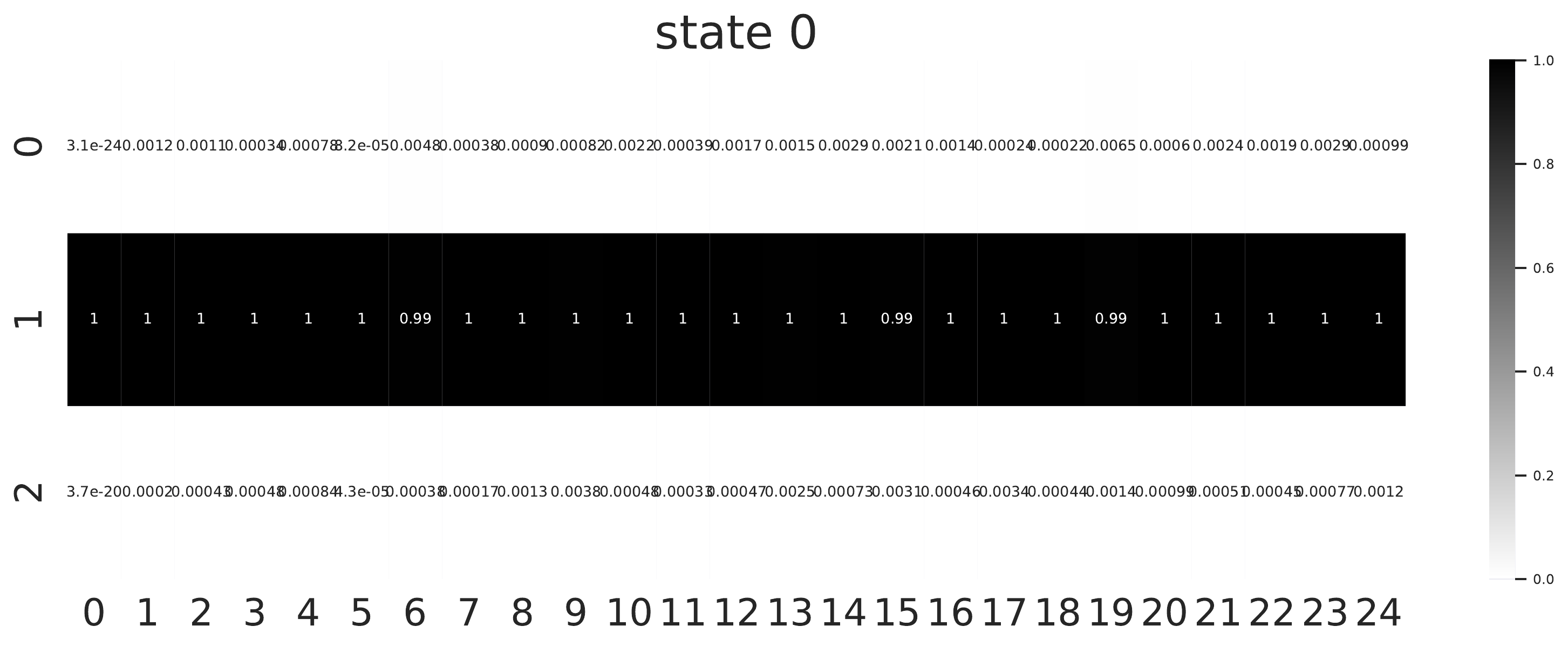}
    \includegraphics[width=0.6\linewidth]{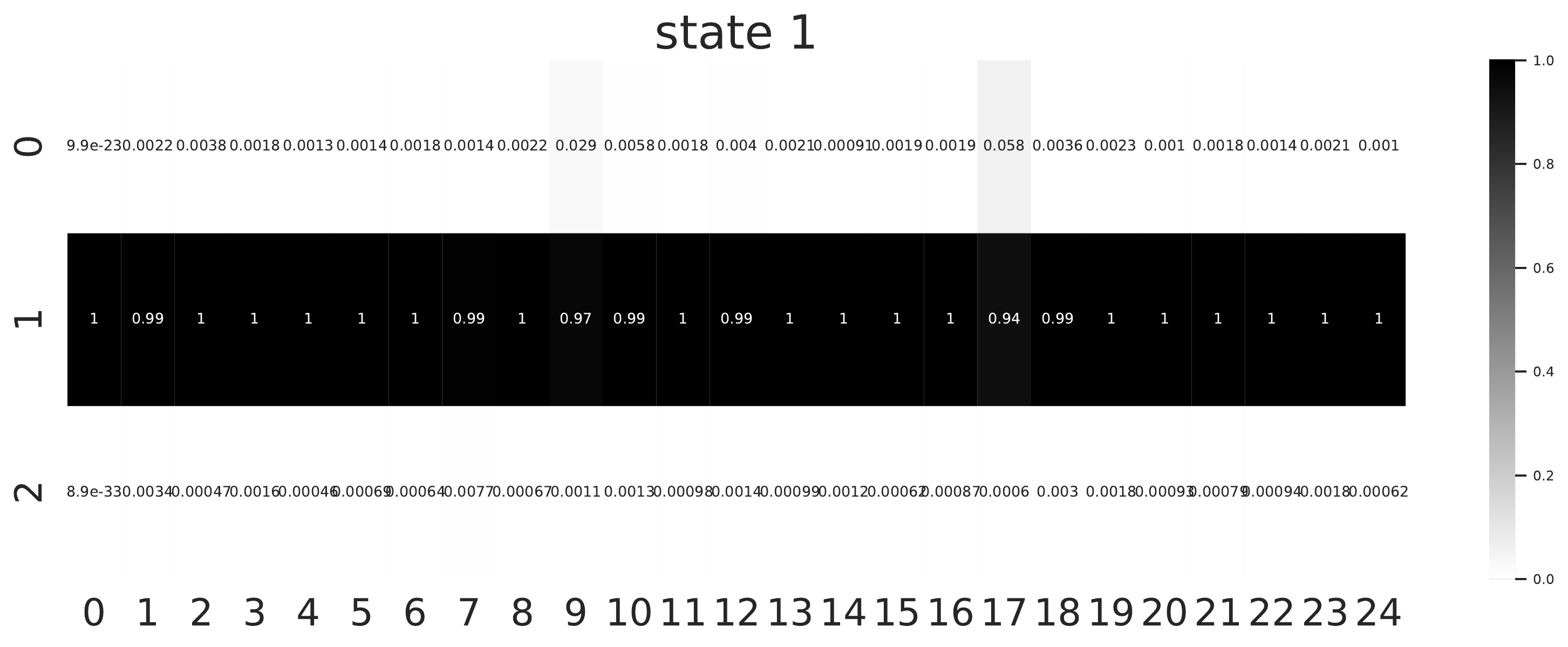}
    \includegraphics[width=0.6\linewidth]{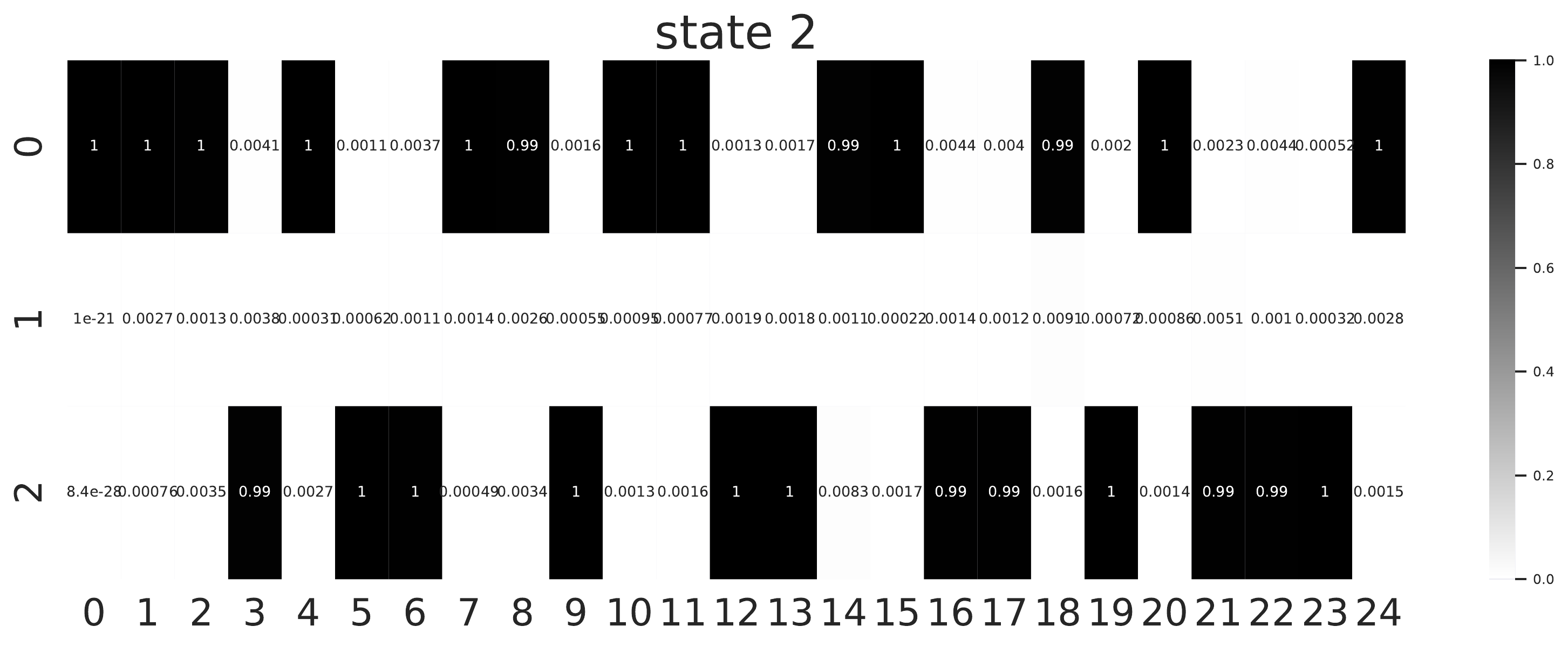}
    \caption{Visualization of decoders from \ouralgo.}
    \label{fig:part:online}
\end{figure}

\begin{figure}[h]
    \centering
    \includegraphics[width=0.6\linewidth]{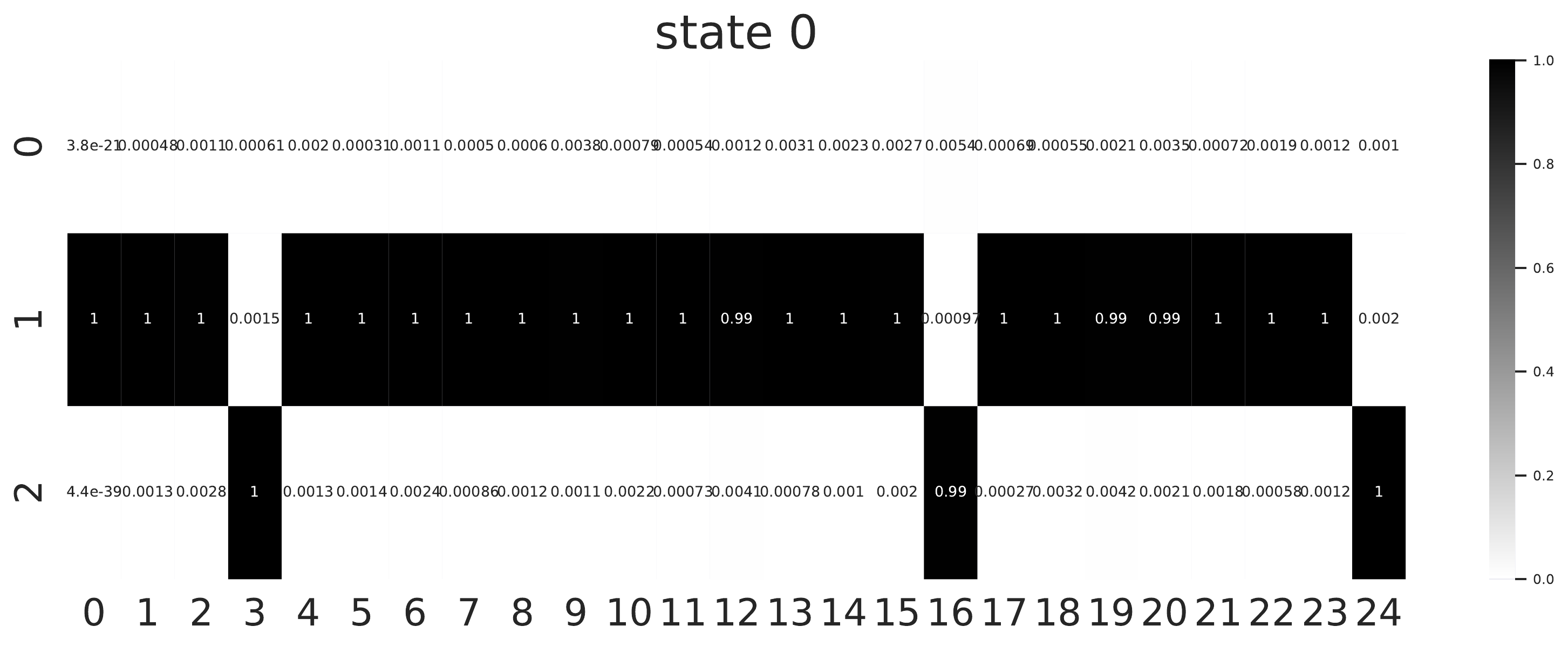}
    \includegraphics[width=0.6\linewidth]{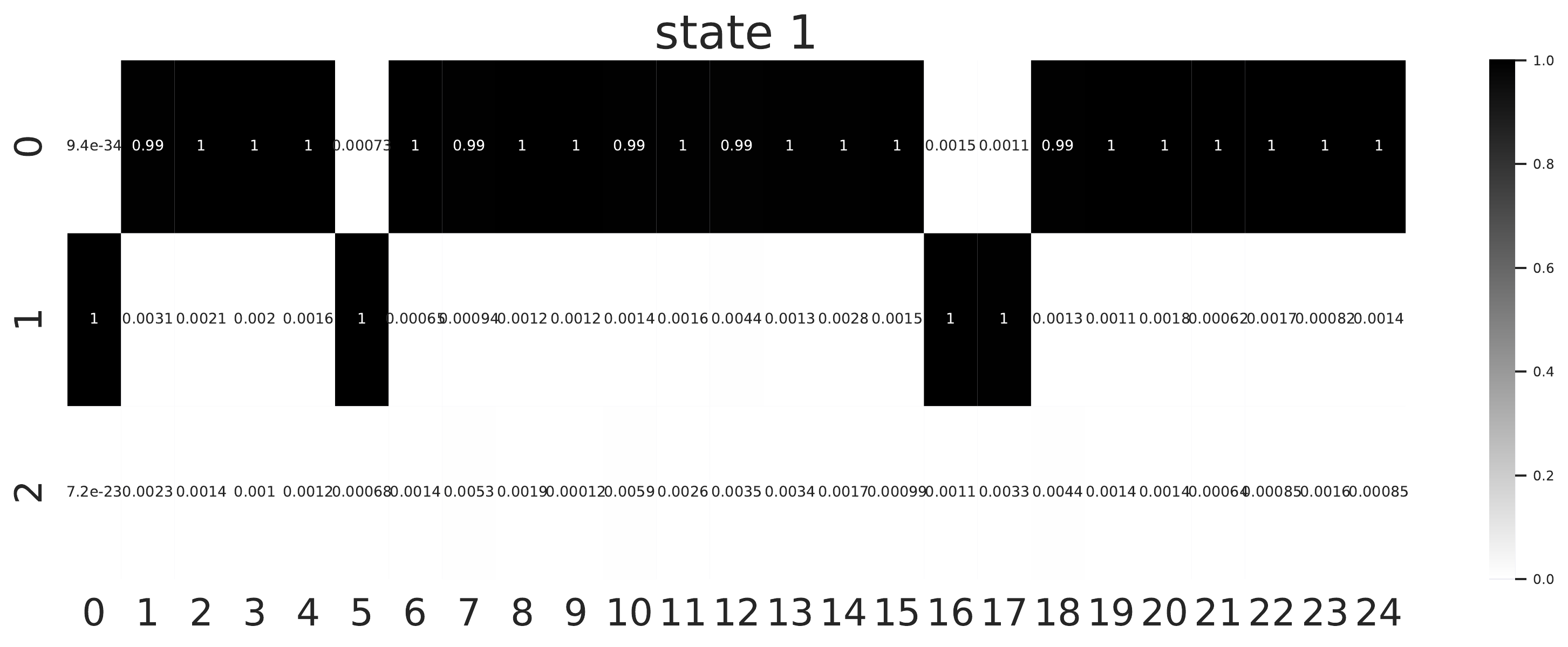}
    \includegraphics[width=0.6\linewidth]{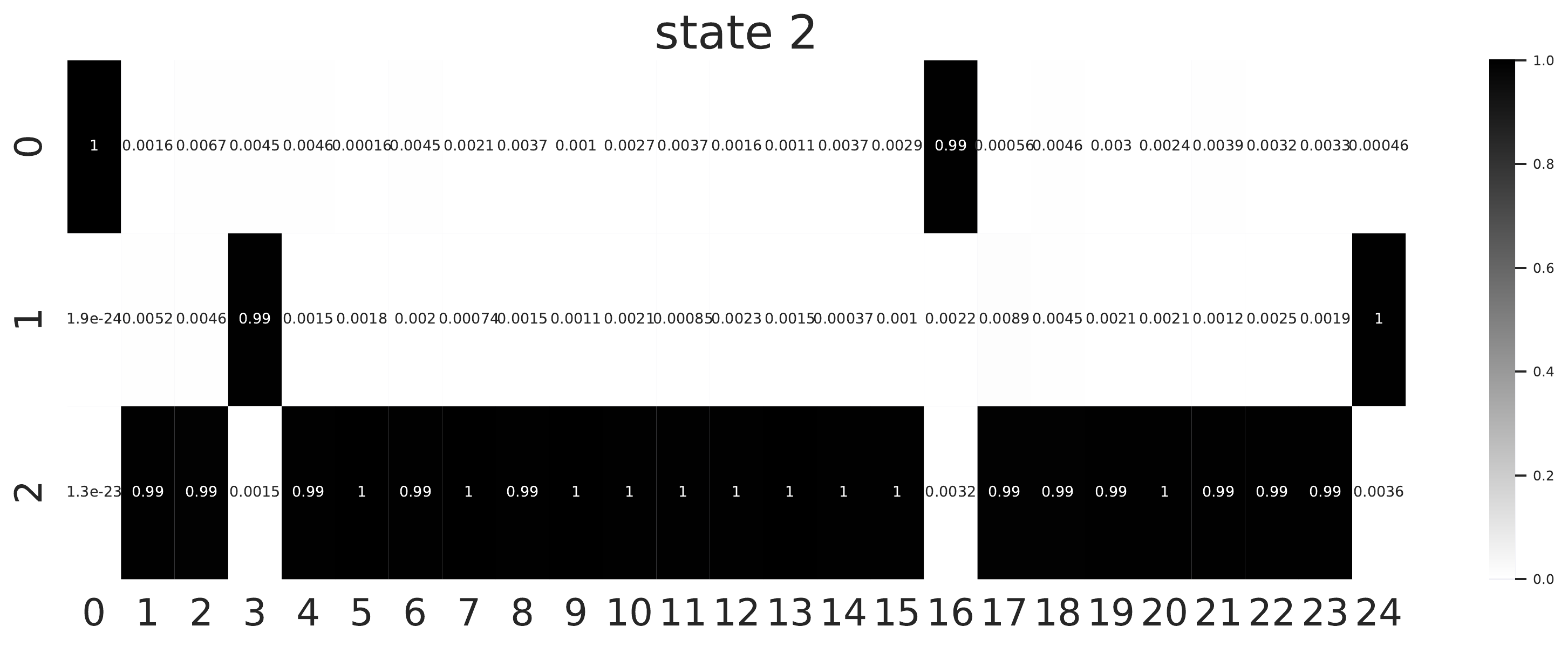}
    \caption{Visualization of decoders from \ouralgg.}
    \label{fig:part:generative}
\end{figure}

\begin{figure}[h]
    \centering
    \includegraphics[width=0.6\linewidth]{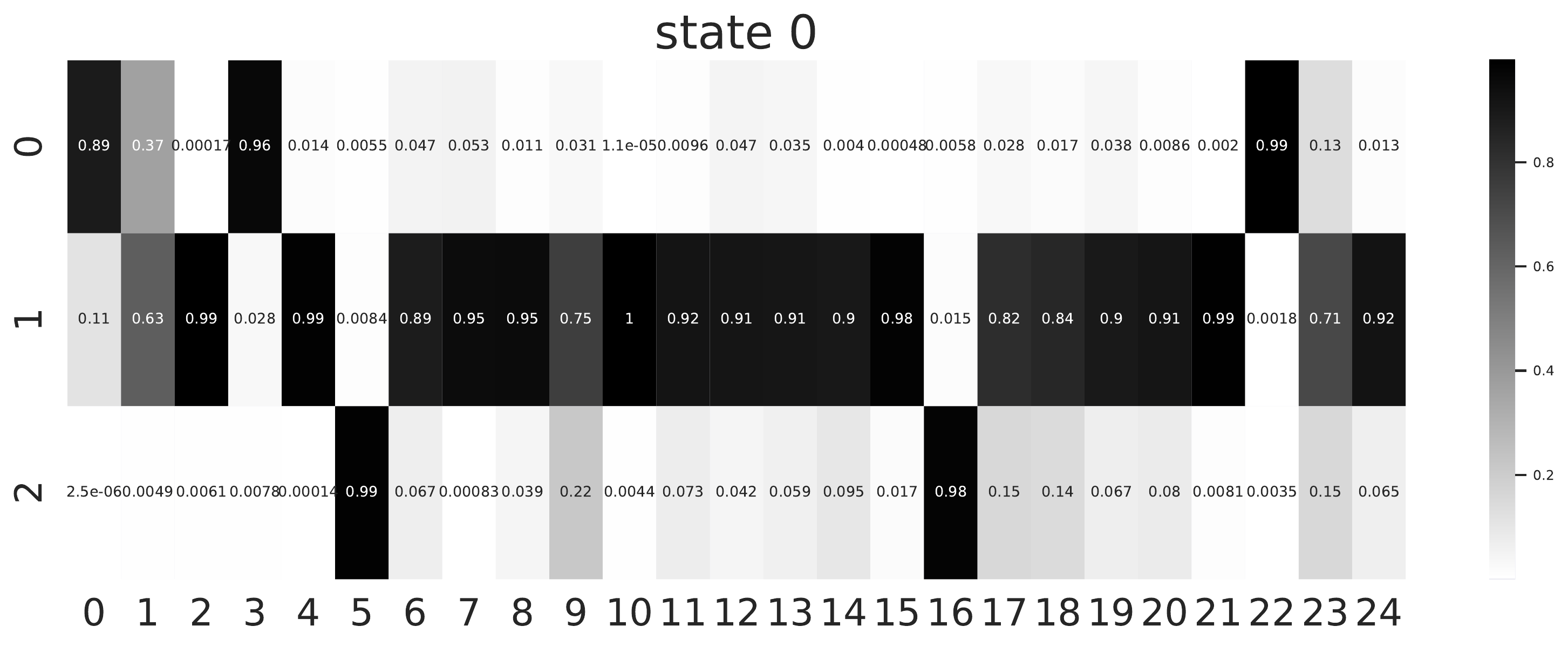}
    \includegraphics[width=0.6\linewidth]{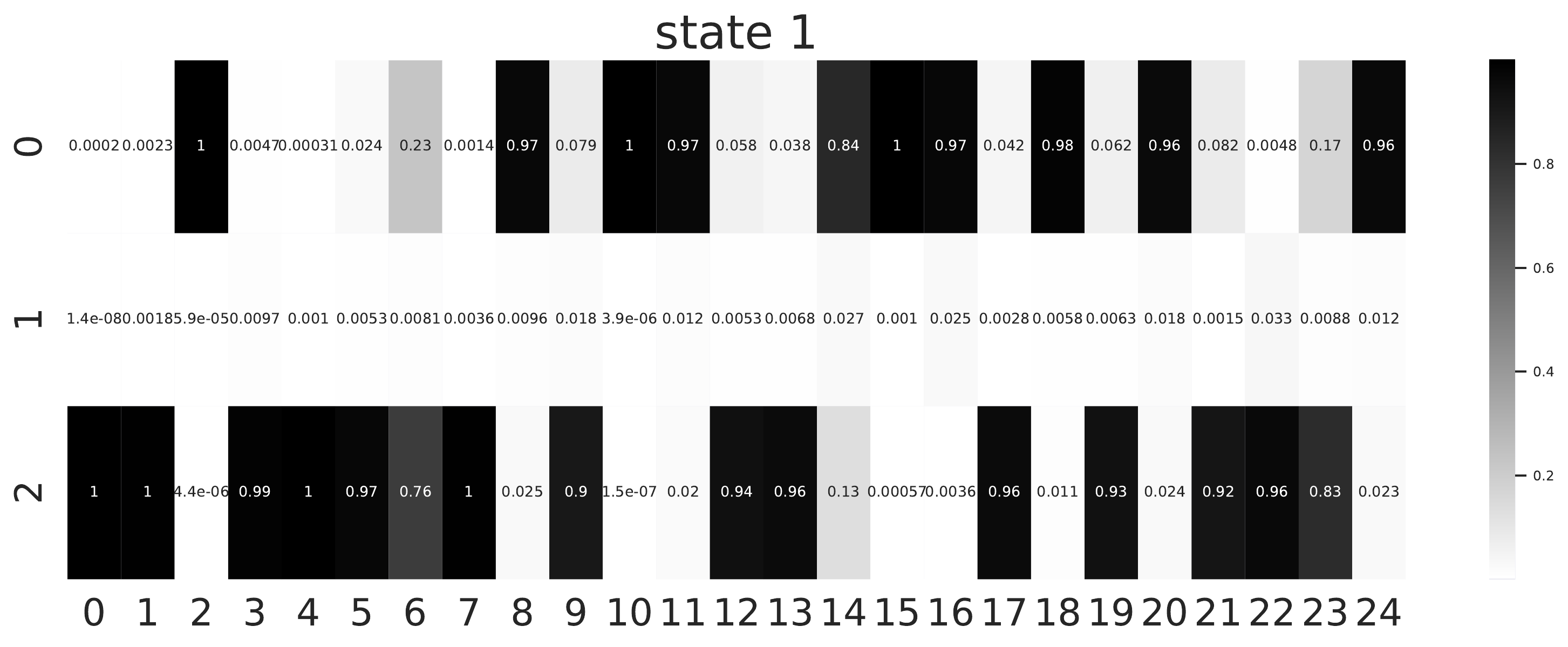}
    \includegraphics[width=0.6\linewidth]{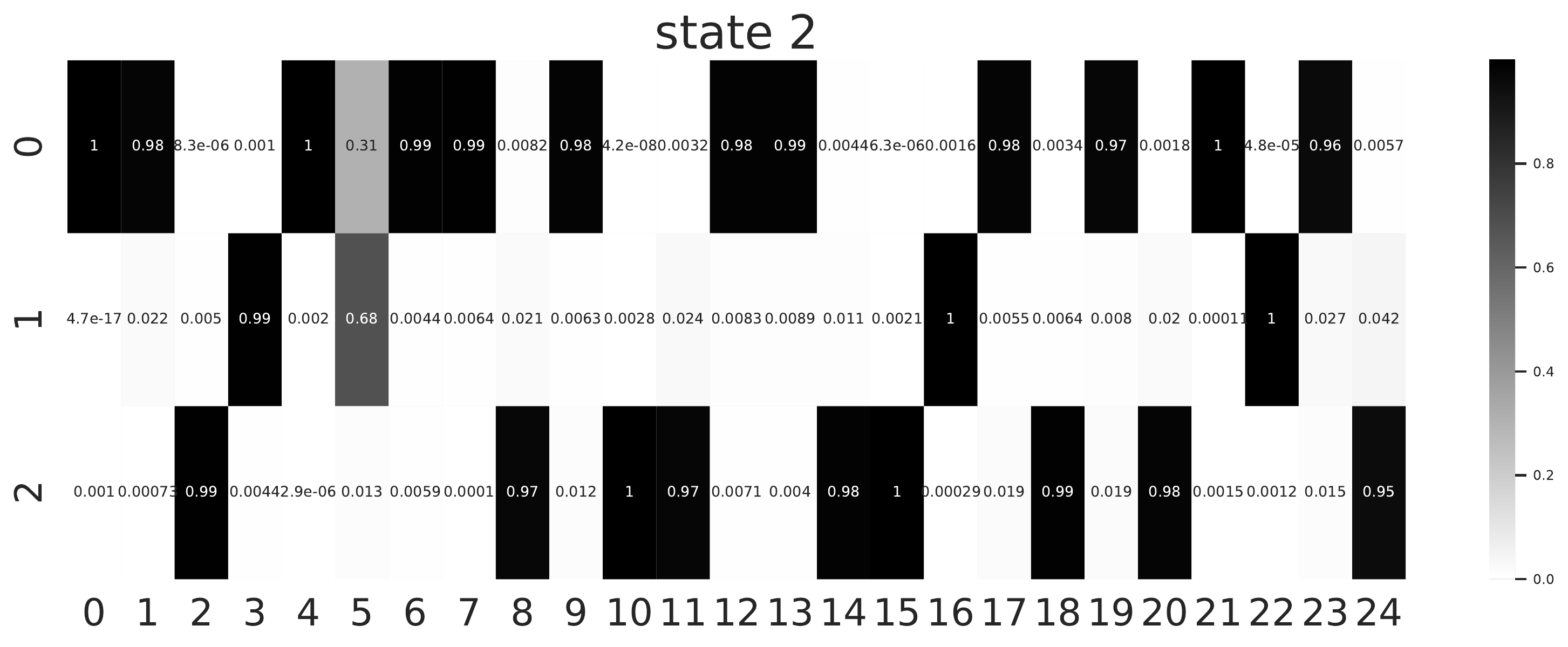}
    \caption{Visualization of decoders running \algname in the target.}
    \label{fig:part:briee}
\end{figure}

\end{document}

%% file: notation.tex
\newcommand{\algname}{\textsc{Briee}\xspace}

\newcommand{\repucb}{\textsc{Rep-UCB}\xspace}
\newcommand{\ouralg}{\textsc{RepTransfer}\xspace}

\newcommand{\ouralgo}{O-\textsc{RepTransfer}\xspace}
\newcommand{\ouralgg}{G-\textsc{RepTransfer}\xspace}
\newcommand{\expAlg}{\textsc{EPS}\xspace}